\documentclass{article}
\usepackage{amsmath,amscd,amsthm,amssymb}
\usepackage{graphicx,subcaption}
\usepackage[dvipsnames]{xcolor}
\usepackage{stfloats}
\usepackage{epsf,epstopdf}
\usepackage[top=0.6in,bottom=0.8in,left=0.5in,right=0.5in]{geometry}
\usepackage{fontspec}
\usepackage{fancyhdr}
\usepackage{hyperref}
\usepackage[all]{xy}
\usepackage{indentfirst} 
\usepackage{multicol}
\usepackage{gensymb}
\usepackage{bm}
\usepackage{ctex}
\usepackage{float}
\usepackage{multirow}
\usepackage{tikz}
\usepackage{pgfplots}
\usepackage{abstract}
\allowdisplaybreaks[4]
\usepackage[numbers,sort]{natbib}
\newtheorem{theorem}{Theorem} 
\newtheorem{lemma}{Lemma}
\newtheorem{proposition}{Proposition}
\newtheorem{assumption}{Assumption}
\newtheorem{remark}{Remark}

\allowdisplaybreaks[4]

\providecommand{\keywords}[1]
{
  \ \ \textbf{Keywords:} #1
}
\definecolor{light-gray}{gray}{0.92}

\title{Generalization Analysis of Distributed Kernel-based Robust Gradient Descent Algorithms$^\dag$\footnotetext{\dag~
The work described in this paper is supported by the National Natural Science Foundation of China (Grants Nos. 12071356, 12271473, 12526595). 
Yuan Mao is also supported by the Fundamental Research Funds for the Central Universities [Project No. 2662025XXQD001].
The corresponding author is Yuan Mao. Email addresses: mengjunyi@zju.edu.cn (J. Y. Meng), ymao@mail.hzau.edu.cn (Y. Mao), guozc@zju.edu.cn (Z. C. Guo).}}   
\author{Jun-Yi Meng$^1$, Yuan Mao$^2$ and Zheng-Chu Guo$^1$\\
\small $^1$ School of Mathematical Sciences, Zhejiang University, Hangzhou 310058, P. R. China \\
\small $^2$  College of Informatics, Huazhong Agricultural University, Wuhan 430070, P. R. China }
\date{}

\begin{document}
\maketitle

\begin{abstract}
    In this paper, we investigate the generalization performance of distributed gradient descent algorithms in a reproducing kernel Hilbert space under a robust loss function $l_{\sigma}$.  By exploiting the spectral characterization of gradient descent together with the intrinsic properties of robust loss functions, we establish optimal learning rates for the distributed kernel-based robust gradient descent (DKRGD) algorithm with an appropriately chosen scale parameter $\sigma$. The proposed parameter choice of $\sigma$ simultaneously alleviates the saturation phenomenon and guarantees statistical robustness. A key technical contribution is a novel error analysis that provides substantially sharper bounds for products of operators, thereby significantly relaxing existing restrictions on the maximum number of local machines while retaining optimal learning rates. Finally, we develop a communication-efficient strategy that further improves the convergence performance of DKRGD.
\end{abstract}


\hspace*{0.6em}\textbf{Mathematics Subject Classification (2020):} 68T05, 68Q32, 68W15, 62G08, 62G35, 62J02

\section{Introduction}\label{section: introdcution}
With the rapid growth of data across a wide range of applications, performing machine learning tasks on a single machine is becoming increasingly impractical. Meanwhile, privacy constraints and data isolation often prevent data collected by different organizations from being directly pooled. For example, patient records may be stored separately across hospitals, while financial institutions typically maintain their risk-management data independently. These considerations have motivated significant interest in distributed learning, where statistical models are constructed from data stored across multiple local machines without requiring direct access to the entire data set. Among various distributed learning paradigms, the ``divide-and-conquer'' approach has attracted particular attention due to its conceptual simplicity and computational efficiency \cite{zhang2015divide,lin2017distributed}.

A substantial body of work has investigated the statistical performance of distributed kernel-based learning algorithms. For distributed regularized least-squares algorithms, optimal learning rates in expectation were established in \cite{zhang2015divide,lin2017distributed}. Subsequently, \cite{guoLearningTheoryDistributed2017} and \cite{mucke2018parallelizing} extended the analysis to distributed spectral algorithms and derived optimal learning rates in expectation. More recently, \cite{zhou2020distributed} refined the theoretical analysis of distributed kernel ridge regression (DKRR), establishing optimal learning rates both in expectation and with high probability.

Despite these advances, the general frameworks developed in \cite{guoLearningTheoryDistributed2017,mucke2018parallelizing} exhibit a saturation phenomenon with respect to the maximal admissible number of local machines. More precisely, once the regularity of the regression function exceeds a certain threshold, further improvements in regularity no longer allow for an increase in the number of local machines while preserving the optimal learning rate. Moreover, the results in \cite{zhou2020distributed} are restricted to the regularity range $\frac{1}{2}\leq r\leq 1$. To alleviate this saturation phenomenon, \cite{lin2018distributed} and \cite{hu2020distributed} exploited the specific spectral structure of gradient descent algorithms to derive improved bounds on the maximum number of local machines.

Another important issue in statistical learning is robustness against outliers and heavy-tailed observations. Least-squares estimators can be sensitive to atypical observations because large residuals are penalized disproportionately heavily. To improve robustness, \cite{guo2018gradient,HongGuo2025RobustKernelBased,GuoChristmannShi2024Optimality} considered a family of robust loss functions of the form
$
l_\sigma(u)=G\left(\frac{u^2}{\sigma^2}\right),
$
where $G$ is a windowing function and $\sigma>0$ is a scale parameter. Appropriate choices of $G$ lead to a broad class of robust loss functions that reduce the influence of large residuals while maintaining desirable statistical properties.

These developments naturally raise the question of whether robust kernel-based gradient methods can be effectively combined with distributed learning. In particular, it is important to understand whether a distributed robust gradient descent algorithm can simultaneously retain statistical robustness, achieve optimal learning rates, and accommodate a sufficiently large number of local machines. The last issue is especially important because the restriction on the number of machines directly determines the extent to which a distributed method can benefit from increasing computational resources.

In this paper, we study the generalization performance of DKRGD in a reproducing kernel Hilbert space (RKHS). Each local machine applies a robust gradient descent algorithm to its own data, and the resulting local estimators are subsequently aggregated to form a global estimator. Our analysis characterizes the statistical performance of the last iterate under standard source and capacity conditions, and reveals the interaction among the regularity of the regression function, the capacity of the hypothesis space, the robust scale parameter $\sigma$, and the number of local machines.

The main contributions of this paper are summarized as follows.

First, we establish a generalization theory for DKRGD both in expectation and with high probability. With appropriate choices of the early-stopping time and the scale parameter $\sigma$, the proposed algorithm achieves the optimal learning rates determined by the regularity of the regression function and the capacity of the underlying RKHS. In particular, the scale parameter can be selected so that
the optimal learning rates do not deteriorate due to the additional error induced by the robust loss, while the estimator retains the robustness associated with the loss function $l_\sigma$.

Second, we develop a refined error analysis that yields sharper estimates for the operator products arising in the study of distributed gradient descent. The resulting bounds substantially relax the restriction on the maximum number of local machines. In particular, the improved analysis alleviates the saturation phenomenon that arises in existing general analyses of distributed spectral algorithms when the regularity of the regression function becomes sufficiently high.

Third, we develop a DKRGD algorithm with communication based on a Newton--Raphson-type correction. The proposed communication strategy further enlarges the admissible number of local machines while preserving the optimal learning rate with high probability. Moreover, the maximum number of local machines increases with the number of communication rounds, revealing a favorable trade-off between communication cost and distributed scalability.

The remainder of this paper is organized as follows. Section \ref{section: problem setting} outlines the regression setting and robust loss functions, followed by a detailed description of DKRGD and its communication strategy. Under mild assumptions, Section \ref{section:section2} establishes optimal learning rates for DKRGD both with high probability and in expectation, and optimal learning rates for DKRGD with communication under high probability. Additionally, it discusses the corresponding restrictions on the number of local machines. Section \ref{section:section3} discusses and compares the theoretical results with related distributed learning methods. Section \ref{section:section4} presents the error decomposition and auxiliary results used in the analysis. Finally, the proofs of the main results are provided in Section \ref{section:section5}.

\section{Problem Setting and Distributed Kernel-based Robust Gradient Descent}\label{section: problem setting}

In this paper, we consider the nonparametric regression problem based on i.i.d. samples $ D=\{x_i, y_i\}_{i=1}^{|D|}$
drawn from an unknown distribution $\rho$ on $Z := \mathcal{X} \times \mathcal{Y}$, where input space $\mathcal{X}$ is a separable metric space and $\mathcal{Y}\subset \mathbb{R}$ is the output space. The training data are assumed to be generated 
according to the model $Y=f_\rho(X)+\epsilon,$
where $ \mathbb{E}[\epsilon|X]=0,$ $X$ is the explanatory variable and $Y$ denotes the corresponding response variable. The regression function $f_\rho$ is defined as 
\begin{equation*}
    f_\rho(x)=\int_{\mathcal{Y}}y \mathrm{d}\rho(y|x), \quad x\in \mathcal{X},
\end{equation*}
where $\rho(y|x)$ denotes the conditional distribution at $x$ induced by $\rho$. The goal of regression is to recover $f_\rho$ from the observed random samples.

Our objective is to investigate the generalization performance of robust gradient descent algorithms in a distributed learning setting. To this end, we adopt the class of robust loss functions introduced in \cite{guo2018gradient}, defined by
\begin{equation*}
l_\sigma(u)=G(\frac{u^{2}}{\sigma^2}),
\end{equation*}  where $G \colon \mathbb{R}_+ \to \mathbb{R}$ is a windowing function and $\sigma>0$ is a scale parameter controlling the degree of robustness.
Hereafter, the windowing function $G$ is assumed to satisfy the following two conditions, which are
\begin{equation}\label{eq:windowing1}
    G'_+(0)>0,\  C_G:=\sup_{s\in (0,\infty)}|G'(s)| \quad \text{with } G'(s)>0 \ \text{for} \ s>0,
\end{equation}
and there exists some $p\geq 0$ and $c_p>0$ such that
\begin{equation}\label{eq:windowing2}
    \left|G'(s)-G'_+(0)\right|\leq c_p|s|^p,\quad \forall s>0.
\end{equation}
By choosing different windowing functions $G$, the general form
$l_\sigma(u)=G(u^2/\sigma^2)$ encompasses a variety of commonly used robust loss functions. Several representative examples for regression are presented below, where $\mathbb{I}_{\mathcal{A}}$ denotes the indicator function of a set $\mathcal{A}$.

\noindent\textbf{Example 1.} Huber's loss combines the advantages of mean squared error and mean absolute error \cite{huber1992robust}.  It takes a quadratic form for small errors and a linear form for larger errors. The balance between two terms is controlled by a threshold parameter $ \sigma $, providing robustness to outliers while maintaining precision.
$$
l_\sigma(u)=\mathbb{I}_{\{|u|\leq \sigma\}}\frac{u^2}{2\sigma^2}+\mathbb{I}_{\{|u|>\sigma\}}\left(\frac{|u|}{\sigma}-\frac{1}{2}\right),\ G(s)=\mathbb{I}_{\{s\leq 1\}}\frac{s}{2}+\mathbb{I}_{\{s> 1\}}\left(\sqrt{s}-\frac{1}{2}\right),\ p=0,\ c_p=\frac{1}{2}.
$$

\noindent\textbf{Example 2.} Fair loss focuses on balancing treatment for larger errors through a piecewise linear modeling approach. It imposes a smaller penalty on larger errors compared to mean squared error, making it suitable for scenarios where equal treatment of errors across different ranges is desired.
$$
l_\sigma(u)=\frac{|u|}{\sigma}-\log\left(1+\frac{|u|}{\sigma}\right),\ G(s)=\sqrt{s}-\log(1+\sqrt{s}),\ p=\frac{1}{2},\ c_p=\frac{1}{2}.
$$
Both loss functions above are convex, and their curves for different values of $ \sigma $ are shown in Figure \ref{fig:fig1}.
\begin{figure}[H] 
    \centering
    \begin{minipage}[t]{0.48\textwidth}
        \includegraphics[scale=0.65]{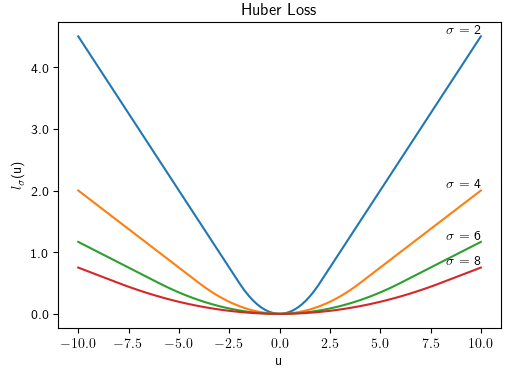}
        \subcaption{}
    \end{minipage}
    \begin{minipage}[t]{0.48\textwidth}
        \includegraphics[scale=0.65]{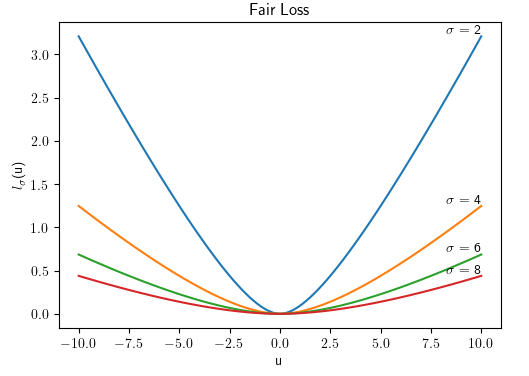}
        \subcaption{}
    \end{minipage}
    \vspace{-4pt}
    \caption{Two convex types of $l_\sigma$ with different $\sigma$. (a) Huber Loss, (b) Fair Loss.}
    \label{fig:fig1}
\end{figure}

Note that the loss function is not necessarily convex; two illustrative nonconvex losses are given below, and their plots for different values of $\sigma$ are shown in Figure \ref{fig: fig2}.

\noindent\textbf{Example 3.} Cauchy loss, motivated by the Cauchy distribution, is designed to aggressively suppress extreme errors.
Compared to Huber's loss, Cauchy loss has a slower decay in the tails, leading to a stronger suppression of outliers.
$$
l_\sigma(u)=\log\left(1+\frac{u^2}{2\sigma^2}\right),\ G(s)=\log\left(1+\frac{s}{2}\right),\ p=1,\ c_p=\frac{1}{4}.
$$

\noindent\textbf{Example 4.} Welsch loss can be induced by the well-known correntropy loss, which resembles mean squared error for small errors and adopts a square root form for larger errors \cite{Holland1977RobustRU}.
The roots of the second derivative of Welsch loss are located at $u=\pm \sigma$, indicating that the loss function is concave for $|u|>\sigma$ and convex for $|u|<\sigma$.
Therefore, adjusting scaling parameter $ \sigma $ allows rejecting outliers while keeping a similar prediction accuracy as that of least squares loss.
$$
l_\sigma(u)=1-\exp\left(-\frac{u^2}{2\sigma^2}\right),\ G(s)=1-\exp\left(-\frac{s}{2}\right),\ p=1,\ c_p=\frac{1}{4}.
$$
\begin{figure}[H] 
    \centering
    \begin{minipage}[t]{0.48\textwidth}
        \includegraphics[scale=0.65]{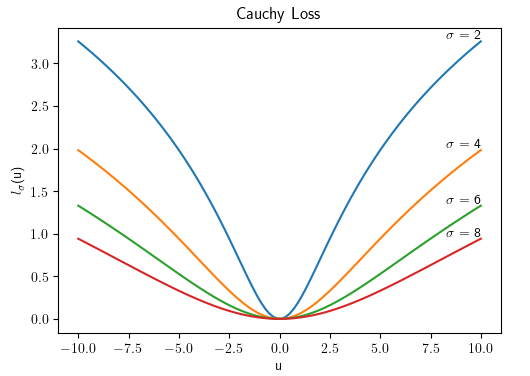}
        \subcaption{}
    \end{minipage}
    \vspace{-2pt}
    \begin{minipage}[t]{0.48\textwidth}
        \includegraphics[scale=0.65]{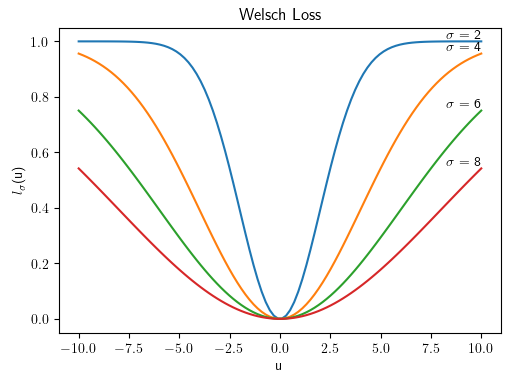}
        \subcaption{}
    \end{minipage}
    \vspace{-2pt}
    \caption{Two nonconvex types of $l_\sigma$ with different $\sigma$. (a) Cauchy Loss, (b) Welsch Loss.}
    \label{fig: fig2}
\end{figure}

\subsection{Distributed Kernel-based Robust Gradient Descent}
In this subsection, we introduce the DKRGD algorithm. We begin with some notation that will be used throughout the subsequent analysis.

Denote by $K \colon \mathcal{X}\times \mathcal{X} \to \mathbb{R}$ a Mercer kernel and $(\mathcal{H}_K,\left\| \cdot  \right\|_{K})$ the corresponding RKHS.
Given i.i.d samples $D=\left\{(x_i,y_i)\right\}_{i=1}^{|D|}$ drawn from $ \rho $, the empirical risk is defined by a robust loss function $l_\sigma$ as
\begin{equation}\label{eq:emprisk}
    \mathcal{E}_{l_\sigma} = \frac{1}{|D|}\sum_{(x_i,y_i)\in D} l_\sigma\left(f(x_i)-y_i\right),
\end{equation}
where $|D|$ is the cardinality of $D$.
Then a robust estimator can be obtained by minimizing the empirical risk $ \mathcal{E}_{l_\sigma} $ over $\mathcal{H}_K$.
Let $K_x \colon \mathcal{X} \to \mathbb{R}$ be the function defined by $K_x(\cdot)=K(x,\cdot)$ for $x\in \mathcal{X}$.
By the reproducing property of kernel $ K $ and proposition 2.1 in \cite{yao2007early}, we can compute the gradient of \eqref{eq:emprisk} and further define the kernel-based robust gradient descent algorithms as follows with $f_{1,D}=0$ and 
\begin{equation}\label{eq:algorithm}
    f_{t+1,D}=f_{t,D}-\frac{\eta}{|D|}\sum_{i=1}^{|D|} G'(\xi_{t,D,\sigma}(z_i))(f_{t,D}(x_i)-y_i)K_{x_i}, \quad \forall t\geqslant 1,
\end{equation}
where $\eta >0$ is the step size, $z_i=(x_i,y_i),$  and $$\xi_{t,D,\sigma}(z_i)=\frac{(f_{t,D}(x_i)-y_i)^2}{\sigma^2}.$$ 

We next extend (\ref{eq:algorithm}) to a distributed learning setting. Specifically, the full data set $D$ is partitioned into $m$ disjoint subsets $\{D_j\}_{j=1}^m$ such that $D=\bigcup_{j=1}^m D_j$ and $D_i \cap D_j = \emptyset$ for $i\neq j$. On the $j$-th local machine, the robust gradient descent iteration \eqref{eq:algorithm} is performed using only the local sample $D_j$, yielding the local estimator $f_{t,D_j}$. Subsequently, the local estimators are aggregated through an average weighted by the local sample sizes to produce a global estimator
\begin{equation}\label{eq:equation4}
    \bar{f}_{t,D}=\sum_{j=1}^m \frac{|D_j|}{|D|} f_{t,D_j}.
\end{equation}
When each local machine performs $T$ gradient descent iterations, our analysis focuses on the last iterate
$\bar{f}_{T+1,D}$, where $T=T(|D|):\mathbb{N}\to \mathbb{N}$ serves as an early-stopping rule to prevent overfitting.

The distributed estimator \eqref{eq:equation4} follows the standard divide-and-conquer paradigm: the local machines perform their computations independently, and communication is required only for the final aggregation of the local estimators. Although such a strategy is simple and communication-efficient, weighted aggregation alone may fail to fully compensate for the statistical loss caused by splitting the full sample across multiple local machines \cite{zhou2020distributed}. Consequently, to retain the optimal learning rate, the number of local machines typically has to satisfy a restrictive upper bound.

A natural approach to alleviating this restriction is to allow additional communication among the local machines. Existing distributed learning methods achieve this by exchanging various types of information, including data \cite{bellet2015distributed}, gradients \cite{zeng2018nonconvex,zhou2020distributed,yin2021distributed}, and local estimators \cite{huang2019distributed,LiuShi2025DistributedFunctional,WangHuLei2025DistributedRobust}. In particular, \cite{zhou2020distributed} developed a Newton--Raphson-based communication strategy for DKRR, which substantially relaxed the restriction on the admissible number of local machines. Motivated by this idea, we develop an analogous communication strategy for DKRGD in the next subsection.

One of the main objectives of this paper is to establish optimal learning rates for the distributed estimator \eqref{eq:equation4}, both in expectation and with high probability. By exploiting the concentration inequality developed in \cite{minsker2017some}, our analysis allows a substantially larger number of local machines than that permitted by existing analyses. Furthermore, we introduce a DKRGD algorithm with communication that further relaxes the restriction on the number of local machines, without requiring additional unlabeled data as in \cite{caponnetto2010cross,lin2018distributed}.

\subsection{DKRGD with Communication}\label{section:section20}
In this subsection, we propose a novel communication strategy for DKRGD with the aim of further relaxing the restriction on the number of local machines. Our construction is motivated by the Newton--Raphson-type communication strategy introduced in \cite{zhou2020distributed}. To facilitate the development of the algorithm, we first derive two operator representations of the robust gradient descent iteration \eqref{eq:algorithm}.

Define the empirical integral operator $L_{K,D}$ as follows,
$$
L_{K,D}(f)=\frac{1}{|D|}\sum_{(x_i,y_i)\in D} f(x_i)K_{x_i},\quad \quad \forall f\in \mathcal{H}_K,
$$
and denote 
\begin{equation}
    \label{eq:EtD}
    E_{t,D,\sigma}=\frac{1}{|D|}\sum_{(x_i,y_i)\in D} (G'_+(0)-G'(\xi_{t,D,\sigma}(z_i)))(f_{t,D}(x_i)-y_i)K_{x_i},
\end{equation}
then the algorithm (\ref{eq:algorithm}) can be rewritten as
\begin{equation*}
    \begin{aligned}
    f_{t+1,D}&=f_{t,D}-\frac{\eta}{|D|}\sum_{i=1}^{|D|} G'_+(0)(f_{t,D}(x_i)-y_i)K_{x_i}+\eta E_{t,D,\sigma}\\
    &=(I-\eta G'_+(0) L_{K,D}) f_{t,D}+\eta G'_+(0) \hat{f}_{K,D}+\eta E_{t,D,\sigma},
    \end{aligned}
\end{equation*}
where $I$ denotes the identity operator, and $\hat{f}_{K,D}=\frac{1}{|D|}\sum_{(x_i,y_i)\in D} y_iK_{x_i}$.

By adding and subtracting the population integral operator $L_K$, defined in \eqref{eq:integral}, we also obtain
\begin{equation*}
    \begin{aligned}
    f_{t+1,D}= (I-\eta G'_+(0) L_{K}) f_{t,D}+\eta G'_+(0) (L_{K}-L_{K,D}) f_{t,D}+\eta G'_+(0) \hat{f}_{K,D}+\eta E_{t,D,\sigma}.
    \end{aligned}
\end{equation*}
Then by induction, we can derive two representations for $f_{t+1,D}$ as 
\begin{equation}\label{eq:algorithm1}
    \begin{aligned}
    f_{t+1,D}&= \sum_{i=1}^t \eta G'_+(0) \left(I-\eta G'_+(0) L_{K,D}\right)^{t-i}\hat{f}_{K,D} + \sum_{i=1}^t \eta \left(I-\eta G'_+(0) L_{K,D}\right)^{t-i}E_{i,D,\sigma},
    \end{aligned}
\end{equation}
and 
\begin{equation}\label{eq:algorithm12}
    \begin{aligned}
    f_{t+1,D}&= \sum_{i=1}^t \eta G'_+(0) \left(I-\eta G'_+(0) L_{K}\right)^{t-i}\hat{f}_{K,D} + \sum_{i=1}^t \eta G'_+(0) \left(I-\eta G'_+(0) L_{K}\right)^{t-i} (L_K-L_{K,D}) f_{i,D}\\
    &+ \sum_{i=1}^t \eta \left(I-\eta G'_+(0) L_{K}\right)^{t-i}E_{i,D,\sigma}.
    \end{aligned}
\end{equation}

Based on the representation (\ref{eq:algorithm1}), we can design a communication strategy for DKRGD by the Newton-Raphson iteration applied in \cite{zhou2020distributed}.
For the sake of convenience, we introduce a polynomial denoted by
\begin{equation}\label{eq:gt1}
    g_{t}(x)=\sum_{i=1}^{t}\eta G'_+(0)\left(1-\eta G'_+(0)x\right)^{t-i},\quad \forall t>0,
\end{equation}
and $g_t(L_{K,D})$ is defined by the spectral calculus.
The invertibility of $g_t(L_{K,D})$ can be guaranteed by the specific choice of $\eta$, which will be given in the main results.
Then for any $ f\in \mathcal{H}_K $, we can rewrite $f_{t+1,D}$ as
\begin{equation}\label{eq:newton12}
    \begin{aligned}
    f_{t+1,D}
    &=g_t(L_{K,D})\hat{f}_{K,D}+ \sum_{i=1}^t \eta  \left(I-\eta G'_+(0) L_{K,D}\right)^{t-i}E_{i,D,\sigma}\\
    &=f-g_t(L_{K,D})\left[g_t^{-1}(L_{K,D})f-\hat{f}_{K,D}- g_t^{-1}(L_{K,D})\sum_{i=1}^t \eta  \left(I-\eta G'_+(0) L_{K,D}\right)^{t-i}E_{i,D,\sigma}\right],
    \end{aligned}
\end{equation}
where the last equation can be seen as the well known Newton-Raphson iteration.
Further, a variant of the global gradient for the empirical risk (\ref{eq:emprisk}) on $\mathcal{H}_K$ over $f$ can be expressed as
\begin{equation}\label{eq:totgradient}
    G_{t+1,D}(f)=g_t^{-1}(L_{K,D})f-\hat{f}_{K,D}-g_t^{-1}(L_{K,D})\sum_{i=1}^t \eta  \left(I-\eta G'_+(0) L_{K,D}\right)^{t-i}E_{i,D,\sigma},
\end{equation}
where $E_{i,D,\sigma}$ remains invariant to the change of $f$.
In distributed setting, denote the global estimator (\ref{eq:equation4}) without communication by
$\bar{f}_{t+1,D}^0:=\bar{f}_{t+1,D}=\sum_{j=1}^m \frac{|D_j|}{|D|}f_{t+1,D_j}$. 
Then by applying \eqref{eq:newton12} to each  $D_j$, we can get
\begin{equation*}
    \begin{aligned}
    \bar{f}_{t+1,D}^0
    &=f-\sum_{j=1}^m \frac{|D_j|}{|D|}g_t(L_{K,D_j})\left[g_t^{-1}(L_{K,D_j})f-\hat{f}_{K,D_j}- g_t^{-1}(L_{K,D_j})\sum_{i=1}^t \eta \left(I-\eta G'_+(0) L_{K,D_j}\right)^{t-i} E_{i,D_j,\sigma}\right].
    \end{aligned}
\end{equation*}
For any $l \in \mathbb{N}_{+}$, let
\begin{equation*}
    \beta_{D_j,l-1}=g_t(L_{K,D_j})G_{t+1,D} (\bar{f}_{t+1,D}^{l-1}).
\end{equation*}
Then we can use the Newton-Raphson iteration to design our communication strategy as
\begin{equation}\label{eq:ftl11}
    \begin{aligned}
        \bar{f}_{t+1,D}^l&=\bar{f}_{t+1,D}^{l-1}-\sum_{j=1}^m \frac{|D_j|}{|D|}\beta_{D_j,l-1}\\
        &=\bar{f}_{t+1,D}^{l-1}-\sum_{j=1}^m \frac{|D_j|}{|D|}g_t(L_{K,D_j})\left[g_t^{-1}(L_{K,D})\bar{f}_{t+1,D}^{l-1}-\hat{f}_{K,D}\right]\\
        &+ \sum_{j=1}^m \frac{|D_j|}{|D|}g_t(L_{K,D_j})g_t^{-1}(L_{K,D})\sum_{i=1}^t \eta \left(I-\eta G'_+(0) L_{K,D}\right)^{t-i} E_{i,D,\sigma}.
    \end{aligned}
\end{equation}
Thus, each communication round applies a Newton--Raphson-type correction to the current global estimator by combining information from the local empirical operators. As will be shown in the subsequent analysis, these additional communication rounds enable DKRGD to tolerate a substantially larger number of local machines while retaining the desired generalization performance.

It is worth noting that the global quantity $G_{t+1,D}$ in \eqref{eq:totgradient} cannot be recovered by simply taking a weighted average of its local counterparts in general. Consequently, an exact implementation of \eqref{eq:ftl11} requires the estimation or communication of certain global operator quantities. Developing more efficient aggregation and operator-estimation procedures for this communication scheme is an interesting direction for future work.

\section{Main Results}\label{section:section2}
In this section, we present generalization bounds for DKRGD both with high probability and in expectation. We then investigate DKRGD with communication and show that, under suitable regularity conditions, communication can further relax the restriction on the number of local machines while preserving the optimal learning rate. To state the main results, we first introduce several standard assumptions concerning the boundedness of the output, the regularity of the regression function, and the capacity of the RKHS $ \mathcal{H}_K $.

Throughout this paper, we assume that the kernel $K$ is a Mercer kernel and $\mathcal{X}$ is compact, then $K$ is bounded and define $ \kappa = \sqrt{\sup_{x\in\mathcal{X}}K(x,x)} <\infty $.
Without loss of generality, we will assume $ \kappa \geq 1 $ in the subsequent analysis.
Let $ \rho_{X} $ be the marginal distribution of $ \rho $ and $L^2_{\rho_X}$ be the Hilbert space of $\rho_X$ square integrable functions on $\mathcal{X}$, with norm denoted by$\|\cdot \|_\rho$.
The kernel $K$ induces the integral operator $L_K$ defined on $L^2_{\rho_X}$ (or the RKHS $\mathcal{H}_K$) by
\begin{equation}
    \label{eq:integral}
    L_K(f)=\int_{\mathcal{X}}f(x)K_x \mathrm{d}\rho_X, \quad f\in L^2_{\rho_X}\ (\text{or}\ f \in \mathcal{H}_K).
\end{equation}
We first impose a standard boundedness condition on the output variable.
\begin{assumption}
    \label{assum:1}
    There exists a constant $ M>0 $ such that $ |y|\leq M $ almost surely with respect to $ \rho $.
\end{assumption}
The boundedness of the outputs is a common assumption in the analysis of regression problems, which is also adopted in \cite{guo2018gradient,guoLearningTheoryDistributed2017}.
Although it is slightly stricter than some moment conditions in more general settings \cite{caponnetto2007optimal}, our analysis in this paper can be easily extended to the case by assuming moment conditions on the outputs.

The next assumption characterizes the regularity of the regression function.
\begin{assumption}
    \label{assum:2}
    There exist $r>0$ and $u_\rho\in L^2_{\rho_X}$ such that
    \begin{equation}\label{eq:regularity}
        f_\rho = L_K^r u_\rho,
    \end{equation}
    where $L_K^r$ denotes the r-th power of $L_K \colon L_{\rho_X}^2 \to L_{\rho_X}^2$ as a compact and positive operator.
\end{assumption}
The regularity condition (\ref{eq:regularity}) of the regression function is a common assumption in the analysis of kernel-based learning algorithms \cite{caponnetto2007optimal,guo2018gradient,guoLearningTheoryDistributed2017,zhou2020distributed}, which is also known as the source condition in the context of inverse problems.
It states that $ f_{\rho} $ lies in the range of $ L_K^r $, and the special case $ r=1 / 2 $ corresponds to the situation where $ f_{\rho} \in \mathcal{H}_K $. 
Intuitively, larger values of $r$ indicate higher regularity of 
$ f_{\rho} $, potentially resulting in improved learning rates.

The complexity of the hypothesis space $ \mathcal{H}_K $ with respect to the measure $ \rho_{X} $ is measured by the effective dimension defined as
\begin{equation*}\label{eq:effective}
    \mathcal{N}(\lambda) = \mathrm{Tr}\left((\lambda I+L_K)^{-1}L_K\right), \quad \lambda > 0.
\end{equation*}
And the capacity assumption is given by the polynomial decay of the effective dimension.
\begin{assumption}
    \label{assume:3}
    There exists some $s\in (0,1]$ and a constant $C_0\geq 1$ which is independent of $\lambda$, such that 
    \begin{equation}
        \label{eq:effdimension}
        \mathcal{N}(\lambda) \leq C_0 \lambda^{-s}, \quad \forall \lambda > 0.
    \end{equation}
\end{assumption}
Condition (\ref{eq:effdimension}) with $ s=1 $ is always satisfied by taking the constant $ C_0 = \mathrm{Tr}(L_K) \leqslant \kappa^{2} $. 
Now we are ready to present the main results, the first of which to be proved in section \ref{section:section5} exhibits optimal learning rates for DKRGD with high probability.
Note that the generalization performance of target function estimation is measured by the $\rho-$ distance between the estimator and the regression function, i.e., $\left\|\bar{f}_{T+1,D}-f_\rho\right\|_\rho$.

\begin{theorem}\label{thm:theorem1}
    Let $0<\delta<1$, $p> 0$ and $0<\eta \leq \frac{1}{\kappa^2 \max\{G'_+(0),C_G\}}$. Under Assumption \ref{assum:1}-\ref{assume:3} with $r > \frac{1}{2}$ and $0<s\leq 1$, if $\lambda=|D|^{-\frac{1}{2r+s}}$, $T=\left\lceil |D|^{\frac{1}{2r+s}} \right\rceil$, $|D_1|=\cdots=|D_m|$ and
    \begin{equation}\label{eq:m1}
        m\leq \frac{|D|^{\frac{2r+s-1}{4r+2s}}}{\left( \left(\log |D|\right)^5+1 \right)^{2}},
    \end{equation}
   then with confidence at least $1-\delta$, there holds
   \begin{equation}\label{eq:prob1}
    \left\|\bar{f}_{T+1,D}-f_\rho\right\|_\rho \leq \tilde{C}_1\max\left\{|D|^{-\frac{r}{2r+s}},\frac{|D|^{\frac{p+1}{2r+s}}}{{\sigma^{2p}}}\right\}\left(\log \frac{48}{\delta}\right)^6,
    \end{equation}
    where $\tilde{C}_1$ is a constant independent of $|D|$ or $m$, and will be given explicitly in the proof.
\end{theorem}
Here and throughout the paper, $\lceil x\rceil$ denotes the smallest integer greater than or equal to $x$.
As shown in Theorem \ref{thm:theorem1}, by choosing $\sigma\geq |D|^{\frac{p+1+r}{(2r+s)2p}}$ and neglecting the logarithmic factor, the high-probability upper bound in (\ref{eq:prob1}) yields the optimal learning rate $\mathcal{O}(|D|^{-\frac{r}{2r+s}})$ in $ L_{\rho_{X}}^{2} $, which achieves the mini-max lower bound proved in \cite{steinwart2009optimal,caponnetto2007optimal}.
\begin{remark}
    Theorem \ref{thm:theorem1} establishes a high-probability learning rate for DKRGD under the regularity condition $r > \frac{1}{2}$ of $ f_{\rho} $.
    In the boundary case $ r=\frac{1}{2}$, where $ f_{\rho} \in \mathcal{H}_K $, the resulting bound incurs an additional logarithmic factor $ \log |D| $ compared with \eqref{eq:prob1}.
    The logarithmic factor arises at the final stage of the error aggregation. Since the proof follows the same line as that of Theorem \ref{thm:theorem1}, the details are omitted.
\end{remark}

We next turn to the generalization performance of DKRGD in expectation. Compared with the high-probability analysis, the expectation bound allows a weaker restriction on the number of local machines.
\begin{theorem}\label{thm:theorem2}
    Let $p> 0$ and $0<\eta \leq \frac{1}{\kappa^2 \max\{G'_+(0),C_G\}}$. Under Assumption \ref{assum:1}-\ref{assume:3} with $r > \frac{1}{2}$ and $0<s\leq 1$, if $\lambda=|D|^{-\frac{1}{2r+s}}$, $T=\left\lceil |D|^{\frac{1}{2r+s}} \right\rceil$, $|D_1|=\cdots=|D_m|$ and 
    \begin{equation}\label{eq:m2}
    m\leq \left\{
    \begin{array}{ll}
        |D|^{\frac{2r+\frac{s}{2}-1}{2r+s}}\left(\left(\log |D|\right)^2+1\right)^{-1},&\frac{1}{2} < r \leq 1,\\
        |D|^{\frac{2r-1}{2r+s}}\left(\left(\log |D|\right)^2+1\right)^{-1},&1 < r \leq \frac{3}{2},\\
        |D|^{\frac{2}{2r+s}}\left(\left(\log |D|\right)^2+1\right)^{-1},&\frac{3}{2}<r\leq \frac{5-s}{2},\\
        |D|^{\frac{2r+s-1}{4r+2s}}(\left(\log |D|\right)^8+1)^{-1},&r>\frac{5-s}{2},
    \end{array}    
\right.
    \end{equation}
    then
    \begin{equation}\label{eq:expec1}
        \mathbb{E}\left[\left\|\bar{f}_{T+1,D}-f_\rho\right\|_\rho^2\right] \leq \tilde{C}_2 \max\left\{|D|^{-\frac{2r}{2r+s}},\frac{|D|^{\frac{2p+2}{2r+s}}}{\sigma^{4p}}\right\},
    \end{equation}
    where $\tilde{C}_2$ is a constant independent of $|D|$ or $m$, and will be given explicitly in the proof.
\end{theorem}
Analogously, Theorem \ref{thm:theorem2} establishes that with the choice of $\sigma\geq |D|^{\frac{p+1+r}{(2r+s)2p}}$, one can obtain the optimal learning rate $\mathcal{O}(|D|^{-\frac{2r}{2r+s}})$ from the expected error bound in (\ref{eq:expec1}).
As a comparison between Theorem \ref{thm:theorem1} and Theorem \ref{thm:theorem2}, we can see that the upper bound of \eqref{eq:m1} is tighter than that of \eqref{eq:m2} for all possible values of $r$, which shows a stricter restriction on $ m $ to guarantee the optimal learning rate with high probability than that in expectation.
Note that when the regularity of $f_\rho$ exceeds a certain level, the restriction on $m$ in Theorem \ref{thm:theorem1} and Theorem \ref{thm:theorem2} reduce to the same order $|D|^{\frac{2r+s-1}{4r+2s}}$, differing only by a logarithmic factor.
\begin{remark}  
    An logarithmic factor also arises in Theorem \ref{thm:theorem2} for the expected error bound.
    At the critical level $ r=\frac{1}{2} $, the resulting bound incurs an additional $ \log^{2} |D| $ factor compared with \eqref{eq:expec1}. This quadratic logarithmic factor results from converting the high-probability estimate into an expectation bound: the logarithmic factor in the high-probability estimate appears squared when moving from tail probability to expectation bounds.
    The derivation closely parallels that of Theorem \ref{thm:theorem2} and is omitted for brevity.
\end{remark}

We finally consider DKRGD equipped with the communication strategy introduced in Subsection \ref{section:section20}. The following theorem quantifies how communication can enlarge the admissible number of local machines while retaining the optimal high-probability learning rate. 
\begin{theorem}\label{thm:theorem3}
    Let $0<\delta<1$, $p> 0$, $0<\eta < \frac{1}{\kappa^2 \max\{G'_+(0),C_G\}}$ and $l \in \mathbb{N}_{+}$. Under Assumption \ref{assum:1}-\ref{assume:3} with $r> \frac{1}{2}$ and $0<s\leq 1$, if $\lambda=|D|^{-\frac{1}{2r+s}}$, $T=\left\lceil |D|^{\frac{1}{2r+s}} \right\rceil$, $|D_1|=\cdots=|D_m|$ and
    \begin{equation}\label{eq:m3}
        m\leq \min\left\{|D|^{\frac{2r-1}{2r+s}},|D|^{\frac{(2r-1)(l+1)+s}{(2r+s)(l+2)}}\right\}\left(\left(\log |D|\right)^{7}+1\right)^{-1},
    \end{equation}
    then with confidence at least $1-\delta$, there holds
   \begin{equation}\label{eq:comm1}
    \left\|\bar{f}_{T+1,D}^l-f_\rho\right\|_\rho\leq \tilde{C}_3\max\left\{|D|^{-\frac{r}{2r+s}},\frac{|D|^{\frac{p+1}{2r+s}}}{\sigma^{2p}}\right\} \left(\log \frac{96}{\delta}\right)^{6+3l},
    \end{equation}
    where $\tilde{C}_3$ is a constant independent of $|D|$ or $m$, and will be given explicitly in the proof.
\end{theorem}
It should be pointed that the optimal learning rate $\mathcal{O}(|D|^{-\frac{r}{2r+s}})$ can be also obtained by setting $\sigma\geq |D|^{\frac{p+1+r}{(2r+s)2p}}$ in \eqref{eq:comm1}.
By comparing (\ref{eq:m1}) with (\ref{eq:m3}), the proposed communication strategy relaxes the restriction on $ m $ from order
$|D|^{\frac{2r+s-1}{4r+2s}}$ to order $\min\left\{|D|^{\frac{2r-1}{2r+s}},|D|^{\frac{(2r-1)(l+1)+s}{(2r+s)(l+2)}}\right\}$ up to a logarithmic factor.
Further, the upper bound of $m$ in (\ref{eq:m3}) is increasing with respect to the number of communications $l$, and it tends to 
order $|D|^{\frac{2r-1}{2r+s}}$ up to a logarithmic factor as $l\to \infty$, which is significantly larger than the upper bound of $m$ in (\ref{eq:m1}) and coincides with that in \cite{guoLearningTheoryDistributed2017,mucke2018parallelizing,lin2017distributed,lin2018distributed}.
Although this limit is not better than that in Theorem \ref{thm:theorem2} for every value of $ r $, it should be highlighted that the optimal learning rate in Theorem \ref{thm:theorem3} is achieved with high probability, which is stronger than that in expectation in Theorem \ref{thm:theorem2}.
To conclude, Theorem \ref{thm:theorem3} conducts the generalization analysis with high probability for DKRGD with communication strategy, which significantly relaxes the restriction on the number of local machines to ensure optimal learning rates, and thus enhances the practical applicability of DKRGD in distributed learning.

\section{Related Work}\label{section:section3}

As a special instance of spectral algorithms, gradient descent (GD) has received extensive attention in the literature. A comprehensive overview of gradient-based optimization methods is provided in \cite{ruder2016overview}, covering widely used variants of gradient descent, such as mini-batch and stochastic gradient descent, adaptive optimization methods including Adam and AdaMax, and distributed stochastic optimization frameworks such as Hogwild! and TensorFlow. In the context of nonparametric learning, the statistical properties of gradient descent have been rigorously investigated in \cite{bauer2007regularization,caponnetto2010cross,dicker2017kernel}, where optimal learning rates were established under suitable regularity and capacity conditions. More recently, \cite{lin2018distributed} studied kernel-based gradient descent in a distributed setting and established optimal learning rates in probability. An important observation therein is that the intrinsic regularization properties of gradient descent can be exploited to alleviate the saturation phenomenon that typically arises in distributed spectral algorithms.

Despite these favorable theoretical properties, standard gradient descent based on the squared loss is generally sensitive to outliers and heavy-tailed noise, and its performance may deteriorate substantially in such settings. Consequently, several approaches have been developed to improve its robustness. Early-stopping strategies, which serve as an implicit regularization mechanism for controlling overfitting to noisy observations, have been extensively studied in \cite{yao2007early,raskutti2014early}. Alternatively, robust loss functions can be employed in place of the squared loss. In particular, \cite{guo2018gradient} investigated gradient descent equipped with a broad class of robust losses, while \cite{hu2020kernel} analyzed the convergence behavior of kernel gradient descent under the maximum correntropy criterion.

Motivated by these developments in robust gradient descent and distributed learning
\cite{guo2018gradient,lin2018distributed,zhou2020distributed,HongGuo2025RobustKernelBased,GuoChristmannShi2024Optimality},
we investigate the generalization performance of DKRGD, a problem that has received comparatively little theoretical attention to date. A particularly relevant starting point is \cite{guo2018gradient}, which demonstrates that robust gradient descent, combined with an appropriate early-stopping rule and a suitable choice of the scale parameter $\sigma$, achieves optimal convergence rates in $L^2_{\rho_X}$. More precisely, under Assumptions \ref{assum:1}--\ref{assum:2} and an eigenvalue decay condition slightly stronger than Assumption \ref{assume:3}, the following estimate holds with confidence at least $1-\delta$:
$$
\left\|f_{T+1}-f_\rho\right\|_\rho\leq \mathcal{O}\! \left( \max\left\{|D|^{-\frac{r}{2r+s}},\frac{|D|^{\frac{p+1}{2r+s}}}{\sigma^{2p}}\right\}\left[\log (6/ \delta) \right]^3 \right).
$$
Consequently, choosing $\sigma\geq |D|^{\frac{p+1+r}{(2r+s)2p}}$ yields the optimal convergence rate $$\mathcal{O}(|D|^{-\frac{r}{2r+s}})$$ in the $L^2_{\rho_X}$-norm.
Under Assumptions \ref{assum:2}--\ref{assume:3} and a condition implied by Assumption \ref{assum:1}, optimal learning rate in probability for robust gradient descent algorithms was established in \cite{lin2018distributed}:  if $r > \frac{1}{2}$ and $ m\leq |D|^\frac{r-\frac{1}{2}}{2r+s}/[(\log |D|)^5+1] $, then with confidence at least $1-\delta$ there holds
$$
\left\|\bar{f}_{T+1,D}-f_\rho\right\|_\rho \leq \mathcal{O} \!\left( |D|^{-\frac{r}{2r+s}}\left[\log (12 / \delta) \right]^4 \right).
$$
This result shows that optimal learning rates can be retained in the distributed setting, provided that the number of local machines is sufficiently controlled.

In contrast, by deriving sharper estimates for certain operator products through the concentration inequality developed in \cite{minsker2017some}, we establish optimal learning rates in probability for DKRGD while substantially relaxing the restriction on the number of local machines. More specifically, the admissible number $m$ of local machines can be enlarged to $ |D|^{\frac{2r+s-1}{4r+2s}}/[(\log |D|)^5+1]^{2}$.
Moreover, the result of \cite{lin2018distributed} permits a nontrivial number of local machines only when $r>\frac12$, and therefore does not cover the boundary case $r=\frac12$. This limitation is largely overcome by Theorem \ref{thm:theorem1}, although an additional logarithmic factor $\log |D|$ appears in the resulting learning rate. 

As for the optimal learning rates in expectation, we further loosen the restriction on $m$.
Compared with \cite{guoLearningTheoryDistributed2017}, the maximum number $m$ is relaxed slightly from $ |D|^{\frac{2r-1}{2r+s}} $ to $ |D|^{\frac{2r+\frac{s}{2}-1}{2r+s}} [ \left(\log |D|\right)^2+1 ]^{-1} $ when $\frac{1}{2} \leq r\leq 1$. 
In addition, the saturation phenomenon inherent to regularized least squares \cite{guoLearningTheoryDistributed2017} is effectively mitigated by leveraging the results in Theorem \ref{thm:theorem1} and the standard tail-expectation formula $\mathbb{E}[\xi]=\int_{0}^{\infty} \mathrm{Prob}(\xi>t) \mathrm{d}t$.  
Specifically, while maintaining optimal learning rates in expectation, the admissible upper bound on the number of local machines can scale up with increasing $r$ when $r>(5-s)/2$.

In distributed learning, numerous studies have focused on relaxing the restriction on the number of local machines by utilizing unlabeled data \cite{caponnetto2010cross,lin2018distributed} or introducing communication strategies \cite{zhou2020distributed,bellet2015distributed,yin2021distributed}. Such approaches can compensate for the statistical loss caused by data partitioning and thereby improve the practical applicability of distributed algorithms.
For nonparametric regression, the most relevant work is \cite{zhou2020distributed}, in which a communication strategy based on the Newton-Raphson iteration is proposed for distributed kernel ridge regression. Inspired by this approach, we propose a communication-enhanced version of DKRGD that further relaxes the restriction on the number of local machines without requiring additional unlabeled data. Although our condition on the number of local machines for attaining the optimal learning rate is slightly more restrictive than that obtained in \cite{zhou2020distributed}, our approach offers two complementary advantages. First, owing to the intrinsic regularization properties of gradient descent, our theoretical results extend beyond the usual range $1/2\leq r\leq1$. More precisely, while the results in \cite{zhou2020distributed} are restricted to $1/2\leq r\leq1$, Theorem \ref{thm:theorem3} applies to the entire regime
$r>\frac12.$ Second, the use of a robust loss makes the resulting distributed algorithm more resistant to outliers and heavy-tailed noise. It would also be interesting to investigate whether the proposed communication strategy can be adapted to other distributed learning frameworks, such as deep distributed convolutional neural networks \cite{Zhou2018DeepDistributedUniversality} and distributed gradient descent functional learning \cite{Yu2024a}.

\section{Error decomposition}\label{section:section4}
In this section, we conduct the error analysis of DKRGD using the integral operator approach \cite{guoLearningTheoryDistributed2017,lin2018distributed,zhou2020distributed,Smale2007,Guo2019a}. We first introduce several preliminary lemmas and propositions that will be repeatedly used in the subsequent analysis. We then derive two distinct error decompositions for DKRGD, which serve as the basis for establishing generalization bounds both with high probability and in expectation, respectively. Finally, we develop a corresponding error decomposition for DKRGD with communication and use it to derive a high-probability generalization bound.

\subsection{Preliminaries}
As the intermediate function in the error decomposition, we first introduce the data-free sequence $\{f_t\}_{t \geq 1}$ which can be regarded as the limit of $\{f_{t,D}\}_{t \geq 1}$ when the sample size $|D|$ goes to infinity, and then we will derive two representations for $f_{t+1}$ by using the integral operator approach.
Specifically, it is defined as follows,
\begin{equation}\label{eq:ft1}
    f_{t+1}=f_t-\eta\int_{\mathcal{Z}}G'(\xi_{t,\sigma}(z))(f_t(x)-y)K_x\mathrm{d}\rho,
\end{equation}
where $f_1 = 0$, $1\leq t\leq T$ and $\xi_{t,\sigma}(z)=\frac{(y-f_t(x))^2}{\sigma^2}$, $(x,y) \in D$.
For the simplicity of subsequent calculations, we can further rewrite $f_{t+1}$ as
\begin{equation*}
    \begin{aligned}
    f_{t+1}&=f_t- \eta G'_+(0) L_K(f_t-f_\rho)+\eta \int_{\mathcal{Z}}(G'_+(0)-G'(\xi_{t,\sigma}(z)))(f_t(x)-y)K_x\mathrm{d}\rho\\
    &=\left(I-\eta G'_+(0) L_{K}\right)f_t +\eta G'_+(0) L_K f_\rho+\eta E_{t,\sigma},
    \end{aligned}
\end{equation*}
and
\begin{equation*}
    \begin{aligned}
    f_{t+1}&=\left(I-\eta G'_+(0) L_{K,D}\right)f_t+\eta G'_+(0) (L_{K,D}-L_K)f_t+\eta G'_+(0) L_K f_\rho+\eta E_{t,\sigma},
    \end{aligned}
\end{equation*}
where  $E_{t,\sigma}=\int_{\mathcal{Z}}(G'_+(0)-G'(\xi_{t,\sigma}(z)))(f_t(x)-y)K_x\mathrm{d}\rho$.

Analogous to Equation (\ref{eq:algorithm1}) and (\ref{eq:algorithm12}), we can employ mathematical induction to derive the following two representations for $f_{t+1}$ as
\begin{equation}\label{eq:datafree1}
    \begin{aligned}
    f_{t+1}&=\sum_{i=1}^t \eta G'_+(0) \left(I-\eta G'_+(0) L_{K}\right)^{t-i}L_K f_\rho + \sum_{i=1}^t \eta \left(I-\eta G'_+(0) L_{K}\right)^{t-i}E_{i,\sigma},
    \end{aligned}
\end{equation}
and
\begin{equation}\label{eq:datafree2}
    \begin{aligned}
    f_{t+1}&=\sum_{i=1}^t \eta G'_+(0) \left(I-\eta G'_+(0) L_{K,D}\right)^{t-i}L_K f_\rho + \sum_{i=1}^t \eta G'_+(0) \left(I-\eta G'_+(0) L_{K,D}\right)^{t-i} (L_{K,D}-L_K) f_i\\
    & \qquad+ \sum_{i=1}^t \eta \left(I-\eta G'_+(0) L_{K,D}\right)^{t-i}E_{i,\sigma}.
    \end{aligned}
\end{equation}
By combining (\ref{eq:algorithm1}) and (\ref{eq:datafree2}), it is easy to obtain the representation for the difference between $f_{T+1,D}$ and $f_{T+1}$ as
\begin{equation}\label{eq:representation}
    f_{T+1,D}-f_{T+1}= \sum_{i=1}^T \eta G'_+(0) \left(I-\eta G'_+(0) L_{K,D}\right)^{T-i}\left\{\hat{f}_{K,D}-L_K f_\rho+\frac{1}{G'_+(0)}(E_{i,D,\sigma}-E_{i,\sigma})+(L_K-L_{K,D})f_i\right\}.
\end{equation}
Moreover, we can combine (\ref{eq:algorithm12}) and (\ref{eq:datafree1}) to yield another representation for $f_{T+1,D} − f_{T+1}$ as
\begin{equation}\label{eq:representation2}
    f_{T+1,D}-f_{T+1}= \sum_{i=1}^T \eta G'_+(0) \left(I-\eta G'_+(0) L_{K}\right)^{T-i}\left\{\hat{f}_{K,D}-L_K f_\rho+\frac{1}{G'_+(0)}(E_{i,D,\sigma}-E_{i,\sigma})+(L_K-L_{K,D})f_{i,D}\right\}.
\end{equation}

Building on the operator representations introduced above, we now establish key estimates for operator products and the data-free iterative sequence $\{f_t\}_{t \geq 1}$, which form the foundation of our subsequent error analysis.
Throughout this paper, we simply denote the operator norm as $\|\cdot\|$.
The following two lemmas are obtained by adapting the proof in \cite[Prop. 4.1, Thm. 1(i)]{guo2018gradient} to accommodate our constant step-size $ \eta $.
For brevity, detailed proofs are deferred to the Appendix.
\begin{lemma}\label{lemma:lemma2}
    Define $\{f_t\}_{t\geq 1}$ as (\ref{eq:ft1}) and $ f_1=0 $. If $0< \eta \leq \frac{1}{\kappa^2 \max\{G'_+(0),C_G\}}$, then
    \begin{equation}\label{eq:lemma2}
        \|f_{t}\|_K\leq M\sqrt{\eta C_G (t-1)}\leq \frac{M}{\kappa}(t-1)^\frac{1}{2},\quad \forall\ t \geq 1.
    \end{equation}
\end{lemma}
\begin{lemma}\label{lemma:lemma1}
    For $\lambda>0$, $0< \eta \leq \frac{1}{\kappa^2 \max\{G'_+(0),C_G\}}$, and $T\in \mathbb{N}$, we have
    \begin{equation*}
        \begin{aligned}
            &\max\left\{\left\|\sum_{i=1}^T \eta G'_+(0) (\lambda I+L_K) \left(I-\eta G'_+(0) L_{K}\right)^{T-i}\right\|,\left\| \sum_{i=1}^T \eta G'_+(0) (\lambda I+L_{K,D}) \left(I-\eta G'_+(0) L_{K,D}\right)^{T-i}\right\|\right\}\\
            \leq &\eta G'_+(0)\lambda T+1.
        \end{aligned}
    \end{equation*}
\end{lemma}

In the end of this subsection, we derive an error decomposition based on \eqref{eq:representation}, which is crucial for the subsequent error analysis of DKRGD in both the high-probability and expectation settings.
\begin{proposition}\label{prop:prop1}
    Let $\lambda>0$, $0< \eta \leq \frac{1}{\kappa^2 \max\{G'_+(0),C_G\}}$, and $T\in \mathbb{N}$. If Assumption \ref{assum:2} holds for $r\geq \frac{1}{2}$, we have
    \begin{equation}\label{eq:max}
        \begin{aligned}
            \left\|(\lambda I+L_{K})^{\frac{1}{2}}(f_{T+1,D}-f_{T+1})\right\|_K&\leq (1+\eta G'_+(0) \lambda T)\mathcal{A}_{D,\lambda}^2(\mathcal{P}_{D,\lambda}+\mathcal{Q}_{D,\lambda}\left\|f_\rho\right\|_K)\\
        &\qquad+\sum_{i=1}^T \eta \left\|(\lambda I+L_{K,D})^{\frac{1}{2}}    \left(I-\eta G'_+(0) L_{K,D}\right)^{T-i}\right\| \mathcal{A}_{D,\lambda}\left\|E_{i,D,\sigma}-E_{i,\sigma}\right\|_K\\
        &\qquad+\sum_{i=1}^T \eta G'_+(0) \left\|(\lambda I+L_{K,D})\left(I-\eta G'_+(0) L_{K,D}\right)^{T-i}\right\|\mathcal{A}_{D,\lambda}^2\mathcal{Q}_{D,\lambda}\left\|f_i-f_\rho\right\|_K,
        \end{aligned}
    \end{equation}
    where 
    \begin{equation*}
        \begin{aligned}
            \mathcal{A}_{D,\lambda}= &\left\|(\lambda I+L_{K})^{\frac{1}{2}}(\lambda I+L_{K,D})^{-\frac{1}{2}}\right\|,\\
            \mathcal{P}_{D,\lambda}= &\left\|(\lambda I+L_{K})^{-\frac{1}{2}}(\hat{f}_{K,D}-L_K f_\rho)\right\|_K, \\
            \mathcal{Q}_{D,\lambda}=&\left\|(\lambda I+L_{K})^{-\frac{1}{2}}(L_K-L_{K,D})\right\|.
        \end{aligned}
    \end{equation*}
\end{proposition}

\begin{proof}
    With the representation \eqref{eq:representation} for $f_{T+1,D}-f_{T+1}$, we have
    \begin{equation}\label{eq:proposition1}
        \begin{aligned}
            &\quad\left\|(\lambda I+L_{K})^{\frac{1}{2}}(f_{T+1,D}-f_{T+1})\right\|_K \leq  \mathcal{A}_{D,\lambda} \Biggl\|\sum_{i=1}^T \eta G'_+(0)(\lambda I+L_{K,D}) \left(I-\eta G'_+(0) L_{K,D}\right)^{T-i} \cdot \\
            & \qquad \qquad \qquad \qquad \qquad(\lambda I+L_{K,D})^{-\frac{1}{2}}\left\{\hat{f}_{K,D}-L_K f_\rho+\frac{E_{i,D,\sigma}-E_{i,\sigma}}{G'_+(0)}+(L_K-L_{K,D})f_i\right\}\Biggr\|_K \\
            & \leq \mathcal{A}_{D,\lambda}\Biggl\{ \left\|\sum_{i=1}^T \eta (\lambda I+L_{K,D})\left(I-\eta G'_+(0) L_{K,D}\right)^{T-i}(\lambda I+L_{K,D})^{-\frac{1}{2}}(E_{i,D,\sigma}-E_{i,\sigma})\right\|_K\\
            & \qquad \qquad + \left\|\sum_{i=1}^T \eta G'_+(0)(\lambda I+L_{K,D})\left(I-\eta G'_+(0) L_{K,D}\right)^{T-i} (\lambda I+L_{K,D})^{-\frac{1}{2}}(\hat{f}_{K,D}-L_K f_\rho)\right\|_K\\
            & \qquad \qquad \qquad+ \left\|\sum_{i=1}^T \eta G'_+(0)(\lambda I+L_{K,D})\left(I-\eta G'_+(0) L_{K,D}\right)^{T-i} (\lambda I+L_{K,D})^{-\frac{1}{2}}(L_K-L_{K,D})f_i\right\|_K \Biggr\}\\
            & =:\mathcal{A}_{D,\lambda}(A_1+A_2+A_3). 
        \end{aligned}
    \end{equation}
    For $ A_1 $, applying the triangle inequality of $ K $-norm, we have
    \begin{equation*}
        \begin{aligned}
        A_1&\leq \sum_{i=1}^T \eta \left\|(\lambda I+L_{K,D})^{\frac{1}{2}}\left(I-\eta G'_+(0) L_{K,D}\right)^{T-i}\right\| \|E_{i,D,\sigma}-E_{i,\sigma}\|_K.
        \end{aligned}
    \end{equation*}
    To estimate $ A_2 $, Lemma \ref{lemma:lemma1} yields
    \begin{equation*}
        \begin{aligned}
        A_2&\leq \left\|\sum_{i=1}^T \eta G'_+(0)(\lambda I+L_{K,D})\left(I-\eta G'_+(0) L_{K,D}\right)^{T-i}\right\| \left\|(\lambda I+L_{K,D})^{-\frac{1}{2}}(\hat{f}_{K,D}-L_K f_\rho)\right\|_K\\
        &\leq (1+\eta G'_+(0)\lambda T)\left\|(\lambda I+L_{K,D})^{-\frac{1}{2}}(\lambda I+L_{K})^{\frac{1}{2}}\right\| \left\|(\lambda I+L_{K})^{-\frac{1}{2}}(\hat{f}_{K,D}-L_K f_\rho)\right\|_K\\
        &= (1+\eta G'_+(0)\lambda T)\mathcal{A}_{D,\lambda} \mathcal{P}_{D,\lambda},
        \end{aligned}
    \end{equation*}
    where the last equality holds since $ \left\| L_1L_2 \right\| =\left\| (L_1L_2)^* \right\|=\left\| L_2L_1 \right\| $ for any self-adjoint operators $L_1$ and $L_2$ on Hilbert spaces.
    Concerning $ A_3 $, by the decomposition $f_i=f_i-f_\rho+f_\rho$, we have
    \begin{equation*}
        \begin{aligned}
            & A_3\leq \sum_{i=1}^T \eta G'_+(0)\left\|(\lambda I+L_{K,D})\left(I-\eta G'_+(0) L_{K,D}\right)^{T-i}\right\|\mathcal{A}_{D,\lambda}\mathcal{Q}_{D,\lambda}\|f_i-f_\rho\|_K \\
            & \qquad \qquad  + \left\|\sum_{i=1}^T \eta G'_+(0)(\lambda I+L_{K,D})\left(I-\eta G'_+(0) L_{K,D}\right)^{T-i}\right\|\mathcal{A}_{D,\lambda}\mathcal{Q}_{D,\lambda}\|f_\rho\|_K \\
            \leq &\sum_{i=1}^T \eta G'_+(0)\left\|(\lambda I+L_{K,D})\left(I-\eta G'_+(0) L_{K,D}\right)^{T-i}\right\| \mathcal{A}_{D,\lambda}\mathcal{Q}_{D,\lambda}\|f_i-f_\rho\|_K +(1+\eta G'_+(0)\lambda T)\mathcal{A}_{D,\lambda} \mathcal{Q}_{D,\lambda} \|f_\rho\|_K.
        \end{aligned}
    \end{equation*}
    Finally, the proof of Proposition \ref{prop:prop1} is completed by substituting the bounds on $A_1$, $A_2$, and $A_3$ into (\ref{eq:proposition1}).
\end{proof}

\subsection{Error decomposition I for DKRGD}
To derive the generalization bounds in probability, we establish an error decomposition for DKRGD that will be used to prove Theorem~\ref{thm:theorem1}. Specifically, by introducing the data-free sequence $\{f_t\}_{t \geq 1}$, we decompose the error $\left\|\bar{f}_{T+1,D}-f_\rho\right\|_\rho$ into two parts: $\left\|\bar{f}_{T+1,D}-f_{T+1}\right\|_\rho$ and $\left\|f_{T+1}-f_\rho\right\|_\rho$. These two components are then estimated in the following two propositions, respectively.

\begin{proposition}\label{prop:prop2}
    For $\lambda>0$, $0< \eta \leq \frac{1}{\kappa^2 \max\{G'_+(0),C_G\}}$, and $T\in \mathbb{N}$. If Assumption \ref{assum:2} holds for $r\geq \frac{1}{2}$, we have
    \begin{equation*}
        \begin{aligned}
        \left\|\bar{f}_{T+1,D}-f_{T+1}\right\|_\rho &\leq (1+\eta G'_+(0)\lambda T) (\mathcal{P}_{D,\lambda}+\mathcal{Q}_{D,\lambda}\|f_\rho\|_K)\\
        &\qquad+\sum_{i=1}^T \eta\left\|(\lambda I+L_K)^{\frac{1}{2}} \left(I-\eta G'_+(0) L_{K}\right)^{T-i}\right\| \sum_{j=1}^{m}\frac{|D_j|}{|D|}\left\|E_{i,D_j,\sigma}-E_{i,\sigma}\right\|_K\\
        &\qquad+\sum_{i=1}^T\eta G'_+(0) \left\|(\lambda I+L_K) \left(I-\eta G'_+(0) L_{K}\right)^{T-i}\right\|\mathcal{Q}_{D,\lambda}\|f_i-f_\rho\|_K+\max_{1\leq j\leq m} \mathcal{K}_{D_j,\lambda},
        \end{aligned}
    \end{equation*}
    where
    \begin{equation*}
        \begin{aligned}
        \mathcal{K}_{D_j,\lambda}=&\sum_{i=2}^T \eta G'_+(0)\left\|(\lambda I+L_K) \left(I-\eta G'_+(0) L_{K}\right)^{T-i}\right\|\mathcal{R}_{D_j,\lambda}\mathcal{A}_{D_j,\lambda} \cdot\\
        & \qquad\biggl\{(1+\eta G'_+(0) \lambda (i-1))\mathcal{A}_{D_j,\lambda}(\mathcal{P}_{D_j,\lambda}+\mathcal{Q}_{D_j,\lambda}\|f_\rho\|_K)  \\
        &\qquad +\sum_{l=1}^{i-1} \eta \left\|(\lambda I+L_{K,D_j})^{\frac{1}{2}}\left(I-\eta G'_+(0) L_{K,D_j}\right)^{i-l-1}\right\| \left\|E_{l,D_j,\sigma}-E_{l,\sigma}\right\|_K\\
        &\qquad +\sum_{l=1}^{i-1} \eta G'_+(0) \left\|(\lambda I+L_{K,D_j})\left(I-\eta G'_+(0) L_{K,D_j}\right)^{i-l-1}\right\|\mathcal{A}_{D_j,\lambda}\mathcal{Q}_{D_j,\lambda}\|f_l-f_\rho\|_K\biggr\},
        \end{aligned}
    \end{equation*}
    and
    \begin{equation*}
        \mathcal{R}_{D,\lambda}:=\left\|(\lambda I+L_{K})^{-\frac{1}{2}}(L_K-L_{K,D})(\lambda I+L_{K})^{-\frac{1}{2}}\right\|.
    \end{equation*} 
\end{proposition}

\begin{proof}
    Based on the definition \eqref{eq:equation4} of $\bar{f}_{t,D}$, we can apply (\ref{eq:representation2}) to $D_j$ for each fixed $j\in {1,...,m}$ to have that
    \begin{equation*}
        \begin{aligned}
        &\bar{f}_{T+1,D}-f_{T+1} =\sum_{j=1}^m \frac{|D_j|}{|D|}(f_{T+1,D_j}-f_{T+1})\\
        =&\sum_{j=1}^m \frac{|D_j|}{|D|}\left\{ \sum_{i=1}^T \eta G'_+(0) \left(I-\eta G'_+(0) L_{K}\right)^{T-i}\left[\hat{f}_{K,D_j}-L_K f_\rho+\frac{1}{G'_+(0)}(E_{i,D_j,\sigma}-E_{i,\sigma})+(L_K-L_{K,D_j})f_{i,D_j}\right]\right\}.
        \end{aligned}
    \end{equation*}
    Since $\sum_{j=1}^m \frac{|D_j|}{|D|}=1$ and 
    $$
    \sum_{j=1}^{m}\frac{|D_j|}{|D|}\hat{f}_{K,D_j}=\sum_{j=1}^{m}\frac{|D_j|}{|D|}\frac{1}{|D_j|}\sum_{(x_i,y_i)\in D_j} y_iK_{x_i}=\frac{1}{|D|}\sum_{(x_i,y_i)\in D} y_iK_{x_i}=\hat{f}_{K,D},
    $$
    we have
    \begin{equation}\label{eq:ft}
        \begin{aligned}
        \left\|\bar{f}_{T+1,D}-f_{T+1}\right\|_\rho &\leq \left\|\sum_{i=1}^T \eta G'_+(0) \left(I-\eta G'_+(0) L_{K}\right)^{T-i}(\hat{f}_{K,D}-L_K f_\rho)\right\|_\rho \\
        &\qquad+\left\|\sum_{i=1}^T \eta \left(I-\eta G'_+(0) L_{K}\right)^{T-i}\sum_{j=1}^{m}\frac{|D_j|}{|D|}( E_{i,D_j,\sigma}-E_{i,\sigma})\right\|_\rho\\
        &\qquad+\left\|\sum_{i=1}^T \eta G'_+(0) \left(I-\eta G'_+(0) L_{K}\right)^{T-i}\sum_{j=1}^{m}\frac{|D_j|}{|D|}(L_K-L_{K,D_j})f_{i,D_j}\right\|_\rho\\
        &=:I_1+I_2+I_3.
        \end{aligned}
    \end{equation}

    Since $ \left\| f \right\|_{\rho}=\left\| L_{K}^{\frac{1}{2}} f \right\|_{K} $ for any $f\in \mathcal{H}_K$ and $ \left\| L_{K}^{\frac{1}{2}} (\lambda I +L_{K}) ^{-\frac{1}{2}} \right\| \leq 1$ for any $\lambda>0$, we have
    \begin{equation}\label{eq:I1}
        \begin{aligned}
        I_1&= \left\|\sum_{i=1}^T \eta G'_+(0) (\lambda I+L_K) \left(I-\eta G'_+(0) L_{K}\right)^{T-i}\right\| \left\|(\lambda I+L_K)^{-1}(\hat{f}_{K,D}-L_K f_\rho)\right\|_{\rho}\\
        &\leq \left\|\sum_{i=1}^T \eta G'_+(0) (\lambda I+L_K) \left(I-\eta G'_+(0) L_{K}\right)^{T-i}\right\| \left\|(\lambda I+L_K)^{-\frac{1}{2}}(\hat{f}_{K,D}-L_K f_\rho)\right\|_K\\
        &\leq (1+\eta G'_+(0)\lambda T) \mathcal{P}_{D,\lambda},
        \end{aligned}
    \end{equation}
    where the last inequality holds by Lemma \ref{lemma:lemma1}.
    For $I_2$, it's easy to derive that
    \begin{equation}\label{eq:I2}
        \begin{aligned}
        I_2&\leq \sum_{i=1}^T \eta\left\|(\lambda I+L_K)^{\frac{1}{2}} \left(I-\eta G'_+(0) L_{K}\right)^{T-i}\right\| \sum_{j=1}^{m}\frac{|D_j|}{|D|}\left\|E_{i,D_j,\sigma}-E_{i,\sigma}\right\|_K.
        \end{aligned}
    \end{equation}
    For $I_3$, it can be bounded by decomposing $f_{i,D_j}$ as $f_{i,D_j}=f_{i,D_j}-f_i+f_i-f_\rho+f_\rho$ as follows
    \begin{equation}\label{eq:I3}
        \begin{aligned}
        I_3&\leq \left\|\sum_{i=1}^T \eta G'_+(0)(\lambda I+L_K) \left(I-\eta G'_+(0) L_{K}\right)^{T-i}\sum_{j=1}^{m}\frac{|D_j|}{|D|}(\lambda I+L_K)^{-\frac{1}{2}}(L_K-L_{K,D_j})(f_{i,D_j}-f_i)\right\|_K\\
        &\qquad+\left\|\sum_{i=1}^T \eta G'_+(0)(\lambda I+L_K) \left(I-\eta G'_+(0) L_{K}\right)^{T-i}(\lambda I+L_K)^{-\frac{1}{2}}(L_K-L_{K,D})(f_i-f_\rho)\right\|_K\\
        &\qquad+\left\|\sum_{i=1}^T \eta G'_+(0)(\lambda I+L_K) \left(I-\eta G'_+(0) L_{K}\right)^{T-i}(\lambda I+L_K)^{-\frac{1}{2}}(L_K-L_{K,D})f_\rho\right\|_K\\
        &=:I_{3,1}+I_{3,2}+I_{3,3},
        \end{aligned}
    \end{equation}
    where $ I_{3,2} $ and $ I_{3,3} $ is derived from the fact that for any $f\in \mathcal{H}_K$, there holds
    $$
    \sum_{j=1}^{m}\frac{|D_j|}{|D|}L_{K,D_j}f=\sum_{j=1}^{m}\frac{|D_j|}{|D|}\frac{1}{|D_j|}\sum_{x_i\in D_j} f(x_i)K_{x_i}=\frac{1}{|D|}\sum_{x_i\in D} f(x_i)K_{x_i}=L_{K,D} f.
    $$
    For $I_{3,2}$, since $f_i \in \mathcal{H}_K$ and Assumption \ref{assum:2} with $r\geq \frac{1}{2}$ implies $f_\rho \in \mathcal{H}_K$, then
    \begin{equation}\label{eq:I32}
        \begin{aligned}
            I_{3,2}&\leq \sum_{i=1}^T\eta G'_+(0) \left\|(\lambda I+L_K) \left(I-\eta G'_+(0) L_{K}\right)^{T-i}\right\|\mathcal{Q}_{D,\lambda}\left\|f_i-f_\rho\right\|_K.
        \end{aligned}
    \end{equation}
    Analogously, $I_{3,3}$ can be bounded by Lemma \ref{lemma:lemma1} as
    \begin{equation}\label{eq:I33}
        \begin{aligned}
            I_{3,3}&\leq \left\|\sum_{i=1}^T\eta G'_+(0) (\lambda I+L_K) \left(I-\eta G'_+(0) L_{K}\right)^{T-i}\right\| \left\|(\lambda I+L_K)^{-\frac{1}{2}}(L_K-L_{K,D})\right\| \|f_\rho\|_K\\
            & = (1+\eta G'_+(0)\lambda T)\mathcal{Q}_{D,\lambda}\|f_\rho\|_K.
        \end{aligned}
    \end{equation}
    To bound $I_{3,1}$, we use $f_1=f_{1,D}=0$ and Jensen's inequality to obtain    
    \begin{equation*}
        \begin{aligned}
        I_{3,1}&\leq \sum_{i=2}^T \eta G'_+(0)\left\|(\lambda I+L_K) \left(I-\eta G'_+(0) L_{K}\right)^{T-i}\right\|  \left\|\sum_{j=1}^{m}\frac{|D_j|}{|D|}(\lambda I+L_K)^{-\frac{1}{2}}(L_K-L_{K,D_j})(f_{i,D_j}-f_i)\right\|_K\\
        &\leq \sum_{j=1}^{m}\frac{|D_j|}{|D|}\sum_{i=2}^T \eta G'_+(0)\left\|(\lambda I+L_K) \left(I-\eta G'_+(0) L_{K}\right)^{T-i}\right\| \cdot\\
        & \qquad \left\|(\lambda I+L_K)^{-\frac{1}{2}}(L_K-L_{K,D_j})(\lambda I+L_K)^{-\frac{1}{2}}\right\| \left\|(\lambda I+L_K)^{\frac{1}{2}}(f_{i,D_j}-f_i)\right\|_K\\
        &\leq \max_{1\leq j\leq m}\sum_{i=2}^T \eta G'_+(0)\left\|(\lambda I+L_K) \left(I-\eta G'_+(0) L_{K}\right)^{T-i}\right\|\mathcal{R}_{D_j,\lambda} \left\|(\lambda I+L_K)^{\frac{1}{2}}(f_{i,D_j}-f_i)\right\|_K.
        \end{aligned}
    \end{equation*}
    Recall the proof of Proposition \ref{prop:prop1} with $D$ and $T+1$  replaced by $D_j$ and $i$, we have
    \begin{equation*}
        \begin{aligned}
        \left\|(\lambda I+L_K)^{\frac{1}{2}}(f_{i,D_j}-f_i)\right\|_K&\leq (1+\eta G'_+(0) \lambda (i-1))\mathcal{A}_{D_j,\lambda}^2(\mathcal{P}_{D_j,\lambda}+\mathcal{Q}_{D_j,\lambda}\|f_\rho\|_K)\\
        &\qquad+\sum_{l=1}^{i-1} \eta \left\|(\lambda I+L_{K,D_j})^{\frac{1}{2}}\left(I-\eta G'_+(0) L_{K,D_j}\right)^{i-l-1}\right\| \mathcal{A}_{D_j,\lambda}\left\|E_{l,D_j,\sigma}-E_{l,\sigma}\right\|_K\\
        &\qquad+\sum_{l=1}^{i-1} \eta G'_+(0) \left\|(\lambda I+L_{K,D_j})\left(I-\eta G'_+(0) L_{K,D_j}\right)^{i-l-1}\right\|\mathcal{A}_{D_j,\lambda}^2\mathcal{Q}_{D_j,\lambda}\|f_l-f_\rho\|_K.
        \end{aligned}
    \end{equation*}
    It follows that 
    \begin{equation}\label{eq:I31}
        \begin{aligned}
        I_{3,1} &\leq \max_{1\leq j\leq m}\sum_{i=2}^T \eta G'_+(0)\left\|(\lambda I+L_K) \left(I-\eta G'_+(0) L_{K}\right)^{T-i}\right\|\mathcal{R}_{D_j,\lambda}\mathcal{A}_{D_j,\lambda}\cdot  \\
        & \qquad \biggl\{(1+\eta G'_+(0) \lambda (i-1))\mathcal{A}_{D_j,\lambda}(\mathcal{P}_{D_j,\lambda}+\mathcal{Q}_{D_j,\lambda}\|f_\rho\|_K)\\
        & \qquad+\sum_{l=1}^{i-1} \eta \left\|(\lambda I+L_{K,D_j})^{\frac{1}{2}}\left(I-\eta G'_+(0) L_{K,D_j}\right)^{i-l-1}\right\| \left\|E_{l,D_j,\sigma}-E_{l,\sigma}\right\|_K\\
        &\qquad+\sum_{l=1}^{i-1} \eta G'_+(0) \left\|(\lambda I+L_{K,D_j})\left(I-\eta G'_+(0) L_{K,D_j}\right)^{i-l-1}\right\|\mathcal{A}_{D_j,\lambda}\mathcal{Q}_{D_j,\lambda}\|f_l-f_\rho\|_K\biggr\}\\
        &=\max_{1\leq j\leq m}\mathcal{K}_{D_j,\lambda}.
        \end{aligned}
    \end{equation}
    By substituting (\ref{eq:I32}), (\ref{eq:I33}) and (\ref{eq:I31}) into (\ref{eq:I3}), we have
    \begin{equation}\label{eq:I3_1}
        \begin{aligned}
        I_{3} &\leq \max_{1\leq j\leq m}\mathcal{K}_{D_j,\lambda}+(1+\eta G'_+(0)\lambda T)\mathcal{Q}_{D,\lambda}\|f_\rho\|_K\\
        & \qquad+\sum_{i=1}^T\eta G'_+(0) \left\|(\lambda I+L_K) \left(I-\eta G'_+(0) L_{K}\right)^{T-i}\right\|\mathcal{Q}_{D,\lambda}\left\|f_i-f_\rho\right\|_K.
        \end{aligned}
    \end{equation}
    Finally, the proof of Proposition \ref{prop:prop2} is completed by substituting (\ref{eq:I1}), (\ref{eq:I2}) and (\ref{eq:I3_1}) into (\ref{eq:ft}).
\end{proof}

We are now ready to bound the generalization error $\|f_{T+1}-f_\rho\|_\rho$, as formalized in the following proposition. The proof adapts the analytical framework established in \cite{engl2000regularization,yao2007early} to our setting.
\begin{proposition}\label{prop:prop3}
    Define $\{f_t\}_{t\geq 1}$ as (\ref{eq:ft1}) and $ f_1=0 $.
    If Assumption \ref{assum:2} holds for $r\geq \frac{1}{2}$, then for any $t\geq 2$, we have
    \begin{equation}\label{eq:prop3}
        \|f_t-f_\rho\|_K \leq \left(\frac{r-\frac{1}{2}}{e\eta G'_+(0)}\right)^{r-\frac{1}{2}} \|u_\rho\|_\rho (t-1)^{-(r-\frac{1}{2})}+\sum_{i=1}^{t-1} \eta \|E_{i,\sigma}\|_K,
    \end{equation}
    and
    \begin{equation}\label{eq:prop31}
        \|f_{t}-f_\rho\|_\rho \leq \left(\frac{r}{e\eta G'_+(0)}\right)^{r} \|u_\rho\|_\rho (t-1)^{-r}+\sum_{i=1}^{t-1} \eta \left\|(\lambda I+L_K)^{\frac{1}{2}}\left(I-\eta G'_+(0)L_{K}\right)^{t-i-1}\right\| \|E_{i,\sigma}\|_K.
    \end{equation}
\end{proposition}

\begin{proof}
    Firstly, for the sake of convenience, we adopt the notation $g_{t}(\cdot)$ defined in~\eqref{eq:gt1}.
    Since for any $ t\geqslant 2 $, 
    \begin{equation}\label{eq:eqrg}
        \begin{aligned}
        g_{t-1}(x)\cdot x&=\sum_{k=1}^{t-1}\eta G'_+(0) x\left(1-\eta G'_+(0)x\right)^{t-k-1} = \sum_{k=1}^{t-1}\left(1-(1-\eta G'_+(0) x)\right)\left(1-\eta G'_+(0)x\right)^{t-k-1}\\
        &=\sum_{k=1}^{t-1}\left[\left(1-\eta G'_+(0)x\right)^{t-k-1}-\left(1-\eta G'_+(0)x\right)^{t-k}\right] = 1-\left(1-\eta G'_+(0)x\right)^{t-1},
        \end{aligned}
    \end{equation}
    we have that for any adjoint operator $A$,
    \[
        I-g_{t-1}(A)A=\left(I-\eta G'_+(0)A\right)^{t-1}.
    \]
    With the representation (\ref{eq:datafree1}) of $ f_t  $, we can rewrite $f_t-f_\rho$ as
    \begin{equation*}
        \begin{aligned}
        f_t-f_\rho&=g_{t-1}(L_K)L_Kf_\rho-f_\rho+\sum_{i=1}^{t-1} \eta \left(I-\eta G'_+(0) L_{K}\right)^{t-i-1} E_{i,\sigma}\\
        &=-\left(I-\eta G'_+(0)L_K\right)^{t-1}f_\rho+\sum_{i=1}^{t-1} \eta \left(I-\eta G'_+(0) L_{K}\right)^{t-i-1} E_{i,\sigma}.
        \end{aligned}
    \end{equation*}
    Further, it can be derived that   
    \begin{equation*}
        \begin{aligned}
        \|f_t-f_\rho\|_K&\leq \|\left(I-\eta G'_+(0)L_K\right)^{t-1}f_\rho\|_K+\sum_{i=1}^{t-1} \eta \left\|\left(I-\eta G'_+(0) L_{K}\right)^{t-i-1} E_{i,\sigma}\right\|_K\\
        &\leq \left\|L_K^{r-\frac{1}{2}} \left(I-\eta G'_+(0)L_K\right)^{t-1}\right\| \|u_\rho\|_\rho+\sum_{i=1}^{t-1} \eta \left\|\left(I-\eta G'_+(0) L_{K}\right)^{t-i-1} \right\| \|E_{i,\sigma}\|_K\\
        &\leq  \left\|L_K^{r-\frac{1}{2}} \left(I-\eta G'_+(0)L_K\right)^{t-1}\right\| \|u_\rho\|_\rho+\sum_{i=1}^{t-1} \eta \|E_{i,\sigma}\|_K,
        \end{aligned}
    \end{equation*}
    where the second inequality holds by Assumption \ref{assum:2} with $ r \geq \frac{1}{2} $ and $\|u_\rho\|_\rho<\infty$, and the last inequality holds by the fact that $\left\|\left(I-\eta G'_+(0) L_{K}\right)^{t-i-1} \right\|\leq 1$ for any $i\in \{1,...,t-1\}$.

    Since $L_K$ is a compact and positive operator, it admits a spectral decomposition with non-negative eigenvalues $\{\gamma_j\}_{j=1}^\infty$.
    Then
    \begin{equation*}
        \begin{aligned}
        \left\|L_K^{r-\frac{1}{2}} \left(I-\eta G'_+(0)L_K\right)^{t-1}\right\|&\leq \sup_{j\geq 1} \gamma_j^{r-\frac{1}{2}} \left(1-\eta G'_+(0)\gamma_j\right)^{t-1}\\
        &=  \sup_{j\geq 1} \exp \left\{(t-1) \log (1-\eta G'_+(0)\gamma_j)+(r-\frac{1}{2}) \log \gamma_j\right\}\\
        &\leq  \sup_{j\geq 1} \exp \left\{-(t-1) \eta G'_+(0)\gamma_j+(r-\frac{1}{2}) \log \gamma_j\right\}.
        \end{aligned}
    \end{equation*}
    For the function 
    $$
    f(\gamma)=-(t-1)\eta G'_+(0)\gamma+(r-\frac{1}{2}) \log \gamma,\quad \gamma>0,
    $$
    it is easy to know that the maximum is attained at $\gamma^*=\frac{r-\frac{1}{2}}{(t-1)\eta G'_+(0)}$ and
    \[
        f(\gamma^*)=-(r-\frac{1}{2})+(r-\frac{1}{2})\log (r-\frac{1}{2})-(r-\frac{1}{2})\log [(t-1)\eta G'_+(0)].
    \]
    Hence, it yields that
    \begin{equation*}
        \begin{aligned}
        \left\|L_K^{r-\frac{1}{2}} \left(I-\eta G'_+(0)L_K\right)^{t-1}\right\|&\leq \exp \{ f(\gamma^*) \} =\left(\frac{r-\frac{1}{2}}{e\eta G'_+(0)}\right)^{r-\frac{1}{2}} (t-1)^{-(r-\frac{1}{2})}.
        \end{aligned}
    \end{equation*}
    Finally, we have
    \begin{equation*}
        \begin{aligned}
        \|f_t-f_\rho\|_K&\leq \left(\frac{r-\frac{1}{2}}{e\eta G'_+(0)}\right)^{r-\frac{1}{2}} \|u_\rho\|_\rho (t-1)^{-(r-\frac{1}{2})}+\sum_{i=1}^{t-1} \eta \|E_{i,\sigma}\|_K.
        \end{aligned}
    \end{equation*}
    
    On the other hand, by following a similar proof strategy, we can likewise obtain
    \begin{equation*}
        \begin{aligned}
        \|f_{t}-f_\rho\|_\rho&\leq \|L_K^{r} \left(I-\eta G'_+(0)L_K\right)^{t-1}\| \|u_\rho\|_\rho+\sum_{i=1}^{t-1} \eta \left\|(\lambda I+L_K)^{\frac{1}{2}}\left(I-\eta G'_+(0) L_{K}\right)^{t-i-1}\right\| \|E_{i,\sigma}\|_K\\
        &\leq \left(\frac{r}{e\eta G'_+(0)}\right)^{r} \|u_\rho\|_\rho (t-1)^{-r}+\sum_{i=1}^{t-1} \eta \left\|(\lambda I+L_K)^{\frac{1}{2}}\left(I-\eta G'_+(0)L_{K}\right)^{t-i-1}\right\| \|E_{i,\sigma}\|_K.
        \end{aligned}
    \end{equation*}
    This completes the proof of Proposition \ref{prop:prop3}.
\end{proof}

Finally, by combining these two propositions with the triangle inequality and letting $ t=T+1 $, we can easily derive the following error decomposition for DKRGD.
\begin{proposition}\label{prop:prop4}
    For $\lambda>0$, $0< \eta \leq \frac{1}{\kappa^2 \max\{G'_+(0),C_G\}}$, and $T\in \mathbb{N}$. If Assumption \ref{assum:2} holds for $r\geq \frac{1}{2}$, we have
    \begin{equation}\label{eq:prop4}
        \begin{aligned}
        \left\|\bar{f}_{T+1,D}-f_\rho\right\|_\rho&\leq (1+\eta G'_+(0)\lambda T) (\mathcal{P}_{D,\lambda}+\mathcal{Q}_{D,\lambda}\|f_\rho\|_K)\\
        &\qquad+\sum_{i=1}^T \eta\left\|(\lambda I+L_K)^{\frac{1}{2}} \left(I-\eta G'_+(0) L_{K}\right)^{T-i}\right\| \sum_{j=1}^{m}\frac{|D_j|}{|D|}\left\|E_{i,D_j,\sigma}-E_{i,\sigma}\right\|_K\\
        &\qquad+\sum_{i=1}^T\eta G'_+(0) \left\|(\lambda I+L_K) \left(I-\eta G'_+(0) L_{K}\right)^{T-i}\right\|\mathcal{Q}_{D,\lambda}\|f_i-f_\rho\|_K+\max_{1\leq j\leq m} \mathcal{K}_{D_j,\lambda}\\
        &\qquad+\left(\frac{r}{e\eta G'_+(0)}\right)^{r} \|u_\rho\|_\rho T^{-r}+\sum_{i=1}^{T} \eta \left\|(\lambda I+L_K)^{\frac{1}{2}}\left(I-\eta G'_+(0) L_{K}\right)^{T-i}\right\| \|E_{i,\sigma}\|_K.
        \end{aligned}
    \end{equation}
\end{proposition}

\subsection{Error decomposition II for DKRGD}\label{sec:5.3}

As established in~\cite{lin2018distributed}, gradient descent can be interpreted as a special instance of spectral algorithms. To derive the generalization error bounds in expectation for DKRGD, we adopt the general error decomposition framework provided in~\cite{guoLearningTheoryDistributed2017}.
\begin{lemma}
    Let $\bar{f}_{T+1,D}=\sum_{j=1}^m \frac{|D_j|}{|D|}f_{T+1,D_j}$, we have
    \begin{equation}\label{eq:expectationeq}
        \begin{aligned}
        \mathbb{E}\left[\left\|\bar{f}_{T+1,D}-f_\rho\right\|_\rho^2\right] &\leq \sum_{j=1}^{m} \frac{|D_j|^2}{|D|^2} \mathbb{E}\left[\left\|f_{T+1,D_j}-f_\rho\right\|_\rho^2\right]+\sum_{j=1}^{m} \frac{|D_j|}{|D|} \left\|\mathbb{E}\left[f_{T+1,D_j}\right]-f_\rho\right\|_\rho^2.
        \end{aligned}
    \end{equation}
\end{lemma}

To estimate the first term on the right-hand side of (\ref{eq:expectationeq}), we introduce the following proposition.

\begin{proposition}\label{prop:prop6}
    For $\lambda>0$, $0< \eta \leq \frac{1}{\kappa^2 \max\{G'_+(0),C_G\}}$, and $T\in \mathbb{N}$. 
    If Assumption \ref{assum:2} holds for $r\geq \frac{1}{2}$, we have
    \begin{equation}\label{eq:prop6}
        \begin{aligned}
        \|f_{T+1,D_j}-f_\rho\|_\rho &\leq (1+\eta G'_+(0) \lambda T)\mathcal{A}_{D_j,\lambda}^2(\mathcal{P}_{D_j,\lambda}+\mathcal{Q}_{D_j,\lambda}\|f_\rho\|_K)\\
        &\qquad+\sum_{i=1}^T \eta \left\|(\lambda I+L_{K,D_j})^{\frac{1}{2}}    \left(I-\eta G'_+(0) L_{K,D_j}\right)^{T-i}\right\| \mathcal{A}_{D_j,\lambda}\left\|E_{i,D_j,\sigma}-E_{i,\sigma}\right\|_K\\
        &\qquad+\sum_{i=1}^T \eta G'_+(0) \left\|(\lambda I+L_{K,D_j})\left(I-\eta G'_+(0) L_{K,D_j}\right)^{T-i}\right\|\mathcal{A}_{D_j,\lambda}^2\mathcal{Q}_{D_j,\lambda}\|f_i-f_\rho\|_K\\
        &\qquad+\left(\frac{r}{e\eta G'_+(0)}\right)^{r} \|u_\rho\|_\rho T^{-r}+\sum_{i=1}^{T} \eta \left\|(\lambda I+L_K)^{\frac{1}{2}}\left(I-\eta G'_+(0) L_{K}\right)^{T-i}\right\| \|E_{i,\sigma}\|_K.
        \end{aligned}
    \end{equation}
\end{proposition}

\begin{proof}
    We first decompose the error into the following two parts, 
    \[
        f_{T+1,D_j}-f_\rho=f_{T+1,D_j}-f_{T+1}+f_{T+1}-f_\rho.
    \]
    The estimates of $\|f_{T+1}-f_\rho\|_\rho$ is guaranteed by(\ref{eq:prop31}) in Proposition \ref{prop:prop3},
    \begin{equation*}
        \|f_{T+1}-f_\rho\|_\rho \leq \left(\frac{r}{e\eta G'_+(0)}\right)^{r} \|u_\rho\|_\rho T^{-r}+\sum_{i=1}^{T} \eta \left\|(\lambda I+L_K)^{\frac{1}{2}}\left(I-\eta G'_+(0) L_{K}\right)^{T-i}\right\| \|E_{i,\sigma}\|_K.
    \end{equation*}
    For the estimates of $\|f_{T+1,D_j}-f_{T+1}\|_\rho$, applying (\ref{eq:max}) with $D$ replaced by $D_j$ yields that  
    \begin{equation*}
        \begin{aligned}
        \|f_{T+1,D_j}-f_{T+1}\|_\rho & \leq\left\|(\lambda I+L_{K})^{\frac{1}{2}}(f_{T+1,D_j}-f_{T+1})\right\|_K\\
        &\leq (1+\eta G'_+(0) \lambda T)\mathcal{A}_{D_j,\lambda}^2(\mathcal{P}_{D_j,\lambda}+\mathcal{Q}_{D_j,\lambda}\|f_\rho\|_K)\\
        &\qquad+\sum_{i=1}^T \eta \left\|(\lambda I+L_{K,D_j})^{\frac{1}{2}}    \left(I-\eta G'_+(0) L_{K,D_j}\right)^{T-i}\right\| \mathcal{A}_{D_j,\lambda}\left\|E_{i,D_j,\sigma}-E_{i,\sigma}\right\|_K\\
        &\qquad+\sum_{i=1}^T \eta G'_+(0) \left\|(\lambda I+L_{K,D_j})\left(I-\eta G'_+(0) L_{K,D_j}\right)^{T-i}\right\|\mathcal{A}_{D_j,\lambda}^2\mathcal{Q}_{D_j,\lambda}\|f_i-f_\rho\|_K.\\
        \end{aligned}
    \end{equation*}
    Finally, the proof of Proposition \ref{prop:prop6} is completed by combining the bounds on $ \left\|f_{T+1,D_j}-f_{T+1}\right\|_\rho $ and $ \|f_{T+1}-f_\rho\|_\rho $  with the triangle inequality $\|f_{T+1,D_j}-f_\rho\|_\rho \leq \left\|f_{T+1,D_j}-f_{T+1}\right\|_\rho+\|f_{T+1}-f_\rho\|_\rho$.
\end{proof}

The second term on the right-hand side of (\ref{eq:expectationeq}) requires careful treatment. 
To establish a rigorous bound, we introduce a semi-supervised learning version of the local estimator $f_{T+1,D_j}$.
Specifically, define
$$
f_{T+1,D}^*:=\mathbb{E}^*\left[f_{T+1,D}\right]=\mathbb{E}\left[f_{T+1,D}|D(x)\right],
$$
where $\mathbb{E}^*$ denotes the conditional expectation with respect to the output $y$ given the fixed input data $D(x)=\{x_i\}_{i=1}^{|D|}$.
Since $\mathbb{E}^*[y_i]=f_\rho(x_i)$, it follows from \eqref{eq:algorithm1} that
\begin{equation}\label{eq:f*}
    f_{T+1,D}^* = \sum_{i=1}^{T} \eta G'_+(0) \left(I-\eta G'_+(0) L_{K,D}\right)^{T-i}L_{K,D} f_\rho+\sum_{i=1}^{T} \eta \left(I-\eta G'_+(0) L_{K,D}\right)^{T-i} \mathbb{E}^*[E_{i,D,\sigma}].
\end{equation}

For any fixed $ j \in \{ 1,2, \ldots m \} $, from Jensen's inequality, we have
\[
    \left\|\mathbb{E}[f_{T+1,D_j}]-f_\rho\right\|_\rho = \left\| \mathbb{E} [\mathbb{E}^*[f_{T+1,D_j}]-f_\rho] \right\|_\rho \leq \mathbb{E}\left[\left\|f_{T+1,D_j}^*-f_\rho\right\|_\rho\right].
\]
Hence, to estimate $\left\|\mathbb{E}[f_{T+1,D_j}]-f_\rho\right\|_\rho$, it suffices to derive the bound of $\left\|f_{T+1,D_j}^*-f_\rho\right\|_\rho$.
With $ D $ replaced by $ D_j $ in (\ref{eq:f*}), it follows from (\ref{eq:eqrg}) that 
\begin{equation*}
    \begin{aligned}
        f_{T+1,D_j}^*-f_\rho&=\left(\sum_{i=1}^{T} \eta G'_+(0) \left(I-\eta G'_+(0) L_{K,D_j}\right)^{T-i}L_{K,D_j}-I\right) f_\rho+\sum_{i=1}^{T} \eta \left(I-\eta G'_+(0) L_{K,D_j}\right)^{T-i}\mathbb{E}^*[E_{i,D_j,\sigma}]\\
        &=-\left(I-\eta G'_+(0) L_{K,D_j}\right)^{T}f_\rho+\sum_{i=1}^{T} \eta \left(I-\eta G'_+(0) L_{K,D_j}\right)^{T-i}\mathbb{E}^*[E_{i,D_j,\sigma}].
    \end{aligned}
\end{equation*}
Further, since (\ref{eq:regularity}) holds for $r\geq \frac{1}{2}$, we can derive the following error decomposition
\begin{align}\label{eq:expecf*}
    \left\|f_{T+1,D_j}^*-f_\rho\right\|_\rho & \leq  \left\|(\lambda I+L_K)^{\frac{1}{2}}\left(I-\eta G'_+(0) L_{K,D_j}\right)^{T} f_\rho\right\|_K 
    + \left\|\sum_{i=1}^{T} \eta (\lambda I+L_K)^{\frac{1}{2}} \left(I-\eta G'_+(0) L_{K,D_j}\right)^{T-i}\mathbb{E}^*[E_{i,D_j,\sigma}]\right\|_K\notag\\ 
    &\leq \mathcal{A}_{D_j,\lambda}\left\|(\lambda I+L_{K,D_j})^{\frac{1}{2}}\left(I-\eta G'_+(0) L_{K,D_j}\right)^T L_K^{r-\frac{1}{2}}\right\| \|u_\rho\|_\rho\notag\\ 
    & \qquad+\sum_{i=1}^{T} \eta \left\|(\lambda I+L_K)^{\frac{1}{2}}\left(I-\eta G'_+(0) L_{K,D_j}\right)^{T-i}\right\| \left\|\mathbb{E}^*[E_{i,D_j,\sigma}]\right\|_K.
\end{align}
Now we are in a position to estimate the above two terms separately, where the corresponding bounds are derived for different ranges of $r$.

\textbf{Case 1:}  $\frac{1}{2}\leq r\leq \frac{3}{2}$.
When $ \frac{1}{2} \leq r \leq 1 $, by rewriting $L_K^{r-\frac{1}{2}}$ as $L_K^{r-\frac{1}{2}}=(\lambda I+L_{K,D_j})^{r-\frac{1}{2}} (\lambda I+L_{K,D_j})^{-(r-\frac{1}{2})}(\lambda I+L_{K})^{r-\frac{1}{2}}(\lambda I+L_{K})^{-(r-\frac{1}{2})}L_K^{r-\frac{1}{2}}$, we can derive 
\begin{equation}\label{eq:expecr1}
    \begin{aligned}
        \left\|f_{T+1,D_j}^*-f_\rho\right\|_\rho &\leq \mathcal{A}_{D_j,\lambda}\biggl\|(\lambda I+L_{K,D_j})^{\frac{1}{2}} \left(I-\eta G'_+(0) L_{K,D_j}\right)^T (\lambda I+L_{K,D_j})^{r-\frac{1}{2}} \cdot \\
        & \qquad \qquad (\lambda I+L_{K,D_j})^{-(r-\frac{1}{2})}(\lambda I+L_{K})^{r-\frac{1}{2}}(\lambda I+L_{K})^{-(r-\frac{1}{2})}L_K^{r-\frac{1}{2}}\biggr\| \|u_\rho\|_\rho\\
        &\qquad+\sum_{i=1}^{T} \eta \left\|(\lambda I+L_K)^{\frac{1}{2}}\left(I-\eta G'_+(0) L_{K,D_j}\right)^{T-i}\right\| \left\|\mathbb{E}^*[E_{i,D_j,\sigma}]\right\|_K\\
        &\leq \left\|(\lambda I+L_{K,D_j})^r\left(I-\eta G'_+(0) L_{K,D_j}\right)^T \right\|\mathcal{A}_{D_j,\lambda}^{2r}\|u_\rho\|_\rho\\
        &\qquad+\sum_{i=1}^{T} \eta \left\|(\lambda I+L_K)^{\frac{1}{2}}\left(I-\eta G'_+(0) L_{K,D_j}\right)^{T-i}\right\| \left\|\mathbb{E}^*[E_{i,D_j,\sigma}]\right\|_K,
    \end{aligned}
\end{equation}
where the last inequality follows from the fact that for two positive operators $A,B$ on a Hilbert space and $s\in [0,1]$, there holds \cite{blanchard2010optimal}
\begin{equation}\label{eq:property1}
    \|A^sB^s\|\leq \|AB\|^s.
\end{equation}
When $1 < r\leq \frac{3}{2}$, considering that $ 2r-1>1 $, we can employ an approach analogous to the case of $\frac{1}{2}\leq r\leq 1$ to obtain
\begin{equation}\label{eq:expecr2}
    \begin{aligned}
        \left\|f_{T+1,D_j}^*-f_\rho\right\|_\rho &\leq \mathcal{A}_{D_j,\lambda}\left\|(\lambda I+L_{K,D_j})^{\frac{1}{2}} \left(I-\eta G'_+(0) L_{K,D_j}\right)^T (\lambda I+L_{K,D_j})^{r-\frac{1}{2}}\right\|\tilde{\mathcal{A}}_{D_j,\lambda}^{r-\frac{1}{2}}\|u_\rho\|_\rho \\
        &\qquad+\sum_{i=1}^{T} \eta \left\|(\lambda I+L_K)^{\frac{1}{2}}\left(I-\eta G'_+(0) L_{K,D_j}\right)^{T-i}\right\| \left\|\mathbb{E}^*[E_{i,D_j,\sigma}]\right\|_K\\
        &\leq \left\|(\lambda I+L_{K,D_j})^r\left(I-\eta G'_+(0) L_{K,D_j}\right)^T\right\|\tilde{\mathcal{A}}_{D_j,\lambda}^{r}\|u_\rho\|_\rho\\
        &\qquad+\sum_{i=1}^{T} \eta \left\|(\lambda I+L_K)^{\frac{1}{2}}\left(I-\eta G'_+(0) L_{K,D_j}\right)^{T-i}\right\| \left\|\mathbb{E}^*[E_{i,D_j,\sigma}]\right\|_K.
    \end{aligned}
\end{equation}
where 
\[
    \tilde{\mathcal{A}}_{D,\lambda}=\|(\lambda I+L_{K,D})^{-1}(\lambda I+L_{K})\|,
\]
and the last inequality holds by the result that $ \mathcal{A}_{D_j,\lambda} \leq \tilde{\mathcal{A}}^{\frac{1}{2}}_{D_j,\lambda} $.

\textbf{Case 2.} $ \frac{3}{2}<r\leq \frac{5}{2}$. 
Note that $ r-\frac{1}{2} >1 $, hence the property in (\ref{eq:property1}) cannot be applied directly. 
To address this issue, we adopt an operator decomposition technique as follows
\[
    (\lambda I+L_K)^{r-\frac{1}{2}}=(\lambda I+L_K)^{r-\frac{1}{2}}-(\lambda I+L_{K,D_j})^{r-\frac{1}{2}}+(\lambda I+L_{K,D_j})^{r-\frac{1}{2}}.
\]
By (\ref{eq:expecf*}), we have
\begin{equation}\label{eq:totalj01}
    \begin{aligned}
        \bigl\|f_{T+1,D_j}^*-f_\rho\bigr\|_\rho&\leq \mathcal{A}_{D_j,\lambda}\Bigl\|(\lambda I+L_{K,D_j})^{\frac{1}{2}} \left(I-\eta G'_+(0) L_{K,D_j}\right)^T \left((\lambda I+L_{K,D_j})^{r-\frac{1}{2}}-(\lambda I+L_K)^{r-\frac{1}{2}}\right)\Bigr\| \|u_\rho\|_\rho\\
        & \qquad +\mathcal{A}_{D_j,\lambda}\left\|(\lambda I+L_{K,D_j})^r\left(I-\eta G'_+(0) L_{K,D_j}\right)^T \right\| \|u_\rho\|_\rho \\
        &\qquad +\sum_{i=1}^{T} \eta \left\|(\lambda I+L_K)^{\frac{1}{2}}\left(I-\eta G'_+(0) L_{K,D_j}\right)^{T-i}\right\| \left\|\mathbb{E}^*[E_{i,D_j,\sigma}]\right\|_K.
    \end{aligned}
\end{equation}
To estimate the first term of the right-hand side of (\ref{eq:totalj01}), we introduce an operator decomposition in \cite{guoLearningTheoryDistributed2017}
\begin{equation}\label{eq:opdecom1}
    \begin{aligned}
        &(\lambda I+L_{K,D_j})^{r-\frac{1}{2}}-(\lambda I+L_K)^{r-\frac{1}{2}}\\
        =&(\lambda I+L_{K,D_j})\left( (\lambda I+L_{K,D_j})^{r-\frac{3}{2}}-(\lambda I+L_{K,D_j})^{-1}(\lambda I+L_K)^{r-\frac{1}{2}} \right) \\
        =&(\lambda I+L_{K,D_j})\left(((\lambda I+L_{K,D_j})^{r-\frac{3}{2}}-(\lambda I+L_K)^{r-\frac{3}{2}})+((\lambda I+L_{K})^{-1}-(\lambda I+L_{K,D_j})^{-1})(\lambda I+L_K)^{r-\frac{1}{2}}\right).
    \end{aligned}
\end{equation}
Combining (\ref{eq:opdecom1}) with (\ref{eq:totalj01}), we have
\begin{equation}\label{eq:totalj}
    \begin{aligned}
        \left\|f_{T+1,D_j}^*-f_\rho\right\|_\rho&\leq \mathcal{J}_{D_j,\lambda} + \mathcal{A}_{D_j,\lambda} \left\|(\lambda I+L_{K,D_j})^r \left(I-\eta G'_+(0) L_{K,D_j}\right)^T \right\| \|u_\rho\|_\rho\\
        &\qquad+\sum_{i=1}^{T} \eta \left\|(\lambda I+L_K)^{\frac{1}{2}}\left(I-\eta G'_+(0) L_{K,D_j}\right)^{T-i}\right\| \left\|\mathbb{E}^*[E_{i,D_j,\sigma}]\right\|_K,
    \end{aligned}
\end{equation}
where
\begin{equation*}
    \begin{aligned}
    \mathcal{J}_{D_j,\lambda}&:=\mathcal{A}_{D_j,\lambda}\biggl\|(\lambda I+L_{K,D_j})^{\frac{1}{2}}\left(I-\eta G'_+(0) L_{K,D_j}\right)^T(\lambda I+L_{K,D_j}) \cdot  \\
    & \qquad\qquad\left((\lambda I+L_{K,D_j})^{r-\frac{3}{2}}-(\lambda I+L_K)^{r-\frac{3}{2}}\right)\biggr\| \|u_\rho\|_\rho\\ 
    & \qquad+\mathcal{A}_{D_j,\lambda}\biggl\|(\lambda I+L_{K,D_j})^{\frac{1}{2}}\left(I-\eta G'_+(0) L_{K,D_j}\right)^T (\lambda I+L_{K,D_j}) \cdot \\
    & \qquad\qquad\left((\lambda I+L_{K})^{-1}-(\lambda I+L_{K,D_j})^{-1}\right)(\lambda I+L_K)^{r-\frac{1}{2}}\biggr\| \|u_\rho\|_\rho.
\end{aligned}
\end{equation*}

In the following, we will estimate $\mathcal{J}_{D_j,\lambda}$ for different ranges of $r$ to derive the error decomposition for $\|f_{T+1,D_j}^*-f_\rho\|_\rho$.
When $\frac{3}{2}<r\leq \frac{5}{2}$, we have $0 < r-\frac{3}{2}\leq 1$.
With the fact that
\begin{equation*}
    (\lambda I+L_{K,D_j})^{r-\frac{3}{2}}-(\lambda I+L_K)^{r-\frac{3}{2}}=(\lambda I+L_{K,D_j})^{r-\frac{3}{2}}\left(I-(\lambda I+L_{K,D_j})^{-(r-\frac{3}{2})}(\lambda I+L_K)^{r-\frac{3}{2}}\right),
\end{equation*}
and (\ref{eq:property1}), we can derive
\begin{equation}\label{eq:j1}
    \begin{aligned}
        \mathcal{J}_{D_j,\lambda}&\leq \mathcal{A}_{D_j,\lambda}\left\|(\lambda I+L_{K,D_j})^r \left(I-\eta G'_+(0) L_{K,D_j}\right)^T\right\|\left(1+\left\|(\lambda I+L_{K,D_j})^{-(r-\frac{3}{2})}(\lambda I+L_K)^{r-\frac{3}{2}}\right\|\right) \|u_\rho\|_\rho\\ 
        &\qquad+\mathcal{A}_{D_j,\lambda}^2 \left\|(\lambda I+L_{K,D_j})\left(I-\eta G'_+(0) L_{K,D_j}\right)^T \right\| \mathcal{Q}_{D_j,\lambda} \left\|(\lambda I+L_K)^{r-\frac{3}{2}}\right\| \|u_\rho\|_\rho\\
        &\leq \tilde{\mathcal{A}}_{D_j,\lambda}^{\frac{1}{2}}\left\|(\lambda I+L_{K,D_j})^r \left(I-\eta G'_+(0) L_{K,D_j}\right)^T\right\|\left(1+\tilde{\mathcal{A}}_{D_j,\lambda}^{r-\frac{3}{2}}\right) \|u_\rho\|_\rho\\ 
        &\qquad+\tilde{\mathcal{A}}_{D_j,\lambda} \left\|(\lambda I+L_{K,D_j}) \left(I-\eta G'_+(0) L_{K,D_j}\right)^T \right\| \mathcal{Q}_{D_j,\lambda} \left\|(\lambda I+L_K)^{r-\frac{3}{2}}\right\| \|u_\rho\|_\rho,
    \end{aligned}
\end{equation}
where the first inequality follows from the fact that
\[
    (\lambda I+L_{K,D_j})^{\frac{1}{2}}\left((\lambda I+L_{K})^{-1}-(\lambda I+L_{K,D_j})^{-1}\right)(\lambda I+L_K)=(\lambda I+L_{K,D_j})^{-\frac{1}{2}}(L_{K,D_j}-L_{K}).
\]
Putting (\ref{eq:j1}) into (\ref{eq:totalj}), we have
\begin{equation}\label{eq:expecr3}
    \begin{aligned}
        \bigl\|f_{T+1,D_j}^*-f_\rho\bigr\|_\rho & \leq \left(\tilde{\mathcal{A}}_{D_j,\lambda}^{\frac{1}{2}}+\tilde{\mathcal{A}}_{D_j,\lambda}^{r-1}\right)\left\|(\lambda I+L_{K,D_j})^r \left(I-\eta G'_+(0) L_{K,D_j}\right)^T\right\| \|u_\rho\|_\rho\\ 
        &\qquad+\tilde{\mathcal{A}}_{D_j,\lambda} \mathcal{Q}_{D_j,\lambda}\left\|(\lambda I+L_{K,D_j}) \left(I-\eta G'_+(0) L_{K,D_j}\right)^T \right\|\left\|(\lambda I+L_K)^{r-\frac{3}{2}}\right\| \|u_\rho\|_\rho\\
        &\qquad+ \tilde{\mathcal{A}}^{\frac{1}{2}}_{D_j,\lambda} \left\|(\lambda I+L_{K,D_j})^r \left(I-\eta G'_+(0) L_{K,D_j}\right)^T \right\| \|u_\rho\|_\rho\\
        &\qquad+\sum_{i=1}^{T} \eta \left\|(\lambda I+L_K)^{\frac{1}{2}}\left(I-\eta G'_+(0) L_{K,D_j}\right)^{T-i}\right\| \left\|\mathbb{E}^*[E_{i,D_j,\sigma}]\right\|_K.
    \end{aligned}
\end{equation}

\textbf{Case 3.} $ r>\frac{5}{2}$. 
Now we have $r-\frac{3}{2}>1$ and (\ref{eq:property1}) cannot be adopted directly. 
Lemma A.6 in \cite{blanchard2010optimal} states that for any two positive self-adjoint operators $A,B$ on a Hilbert space satisfying $\|A\| \leq c$ and $\|B\| \leq c$ for some constant $c>0$, there holds that for all $ s \geq 1$, 
\begin{equation*}
    \|A^s-B^s\|\leq s c^{s-1}\|A-B\|.
\end{equation*}
Recall that $\|L_{K}\|\leq \kappa^{2}$ and $\|L_{K,D_j}\|\leq \kappa^{2}$.
To estimate $ \mathcal{J}_{D_j,\lambda} $ in \eqref{eq:totalj}, we substitute $A=\lambda I+L_{K,D_j}$, $B=\lambda I+L_K$, and $s=r-\frac{3}{2}$ into the above inequality to yield
\begin{equation*}
    \begin{aligned}
        \mathcal{J}_{D_j,\lambda}&\leq \left(r-\frac{3}{2}\right) \kappa^{2r-5} \tilde{\mathcal{A}}_{D_j,\lambda}^{\frac{1}{2}}\left\|(\lambda I+L_{K,D_j})^{\frac{3}{2}} \left(I-\eta G'_+(0) L_{K,D_j}\right)^T \right\| \|L_K-L_{K,D_j}\| \|u_\rho\|_\rho\\ 
        & \qquad+\tilde{\mathcal{A}}_{D_j,\lambda} \left\|(\lambda I+L_{K,D_j}) \left(I-\eta G'_+(0) L_{K,D_j}\right)^T \right\| \mathcal{Q}_{D_j,\lambda} \left\|(\lambda I+L_K)^{r-\frac{3}{2}}\right\| \|u_\rho\|_\rho.
    \end{aligned}
\end{equation*}
By putting the above bound into (\ref{eq:totalj}), we have
\begin{equation}\label{eq:expecr4}
    \begin{aligned}
        \bigl\|f_{T+1,D_j}^*-f_\rho\bigr\|_\rho & \leq \left(r-\frac{3}{2}\right)  \kappa^{2r-5} \tilde{\mathcal{A}}_{D_j,\lambda}^{\frac{1}{2}} \left\|(\lambda I+L_{K,D_j})^{\frac{3}{2}} \left(I-\eta G'_+(0) L_{K,D_j}\right)^T \right\| \|L_K-L_{K,D_j}\| \|u_\rho\|_\rho\\ 
        &\qquad+\tilde{\mathcal{A}}_{D_j,\lambda} \mathcal{Q}_{D_j,\lambda}\left\|(\lambda I+L_{K,D_j}) \left(I-\eta G'_+(0) L_{K,D_j}\right)^T \right\|\left\|(\lambda I+L_K)^{r-\frac{3}{2}}\right\| \|u_\rho\|_\rho\\
        &\qquad+ \tilde{\mathcal{A}}^{\frac{1}{2}}_{D_j,\lambda} \left\|(\lambda I+L_{K,D_j})^r\left(I-\eta G'_+(0) L_{K,D_j}\right)^T \right\| \|u_\rho\|_\rho\\
        &\qquad+\sum_{i=1}^{T} \eta \left\|(\lambda I+L_K)^{\frac{1}{2}}\left(I-\eta G'_+(0) L_{K,D_j}\right)^{T-i}\right\| \left\|\mathbb{E}^*[E_{i,D_j,\sigma}]\right\|_K.
    \end{aligned}
\end{equation}
Thus, we have completed the error decomposition of $ \bigl\|f_{T+1,D_j}^*-f_\rho\bigr\|_\rho $.

\subsection{Error decomposition for DKRGD with communication}\label{sec:5.4}
In this subsection, we derive an error decomposition for DKRGD with communication based on (\ref{eq:ftl11}).
Firstly, we adopt $ f_{t+1,D} $ as an intermediate function to derive the following error decomposition 
$$
\left\|\bar{f}_{T+1,D}^l-f_\rho\right\|_\rho\leq \left\|\bar{f}_{T+1,D}^l-f_{T+1,D}\right\|_\rho + \left\|f_{T+1,D}-f_\rho\right\|_\rho.
$$
The bound on $ \left\|f_{T+1,D}-f_\rho\right\|_\rho $ has been derived in \cite{guo2018gradient} with the polynomial decay assumption on the eigenvalues of the integral operator $L_K$, in contrast to which we present a more general bound adapted to the capacity assumption as follows.
By substituting $ D_j $ with $ D $ in \eqref{eq:prop6}, we can follow a proof strategy analogous to that of Proposition \ref{prop:prop6} to derive 
\begin{equation}\label{eq:tt00}
    \begin{aligned}
        \|f_{T+1,D}-f_\rho\|_\rho &\leq (1+\eta G'_+(0) \lambda T)\mathcal{A}_{D,\lambda}^2(\mathcal{P}_{D,\lambda}+\mathcal{Q}_{D,\lambda}\|f_\rho\|_K)\\
        &\qquad+\sum_{i=1}^T \eta \left\|(\lambda I+L_{K,D})^{\frac{1}{2}}    \left(I-\eta G'_+(0) L_{K,D}\right)^{T-i}\right\| \mathcal{A}_{D,\lambda}\left\|E_{i,D,\sigma}-E_{i,\sigma}\right\|_K\\
        &\qquad+\sum_{i=1}^T \eta G'_+(0) \left\|(\lambda I+L_{K,D})\left(I-\eta G'_+(0) L_{K,D}\right)^{T-i}\right\|\mathcal{A}_{D,\lambda}^2\mathcal{Q}_{D,\lambda}\|f_i-f_\rho\|_K\\
        &\qquad+\left(\frac{r}{e\eta G'_+(0)}\right)^{r} \|u_\rho\|_\rho T^{-r}+\sum_{i=1}^{T} \eta \left\|(\lambda I+L_K)^{\frac{1}{2}}\left(I-\eta G'_+(0) L_{K}\right)^{T-i}\right\| \|E_{i,\sigma}\|_K.
    \end{aligned}
\end{equation}

Next, we will derive the error decomposition for $\left\|\bar{f}_{T+1,D}^l-f_{T+1,D}\right\|_\rho$.
Due to (\ref{eq:algorithm1}) and (\ref{eq:ftl11}), we get
\begin{equation*}
    \begin{aligned}
        f_{T+1,D}-\bar{f}_{T+1,D}^l= & -g_T(L_{K,D})\left[g_T^{-1}(L_{K,D})\bar{f}_{T+1,D}^{l-1}-\hat{f}_{K,D}\right]+\sum_{i=1}^T \eta  \left(I-\eta G'_+(0) L_{K,D}\right)^{T-i}E_{i,D,\sigma} \\
        &\qquad+\sum_{j=1}^m \frac{|D_j|}{|D|}g_T(L_{K,D_j})\left[g_T^{-1}(L_{K,D})\bar{f}_{T+1,D}^{l-1}-\hat{f}_{K,D}\right]\\
        &\qquad -\sum_{j=1}^m \frac{|D_j|}{|D|} g_T(L_{K,D_j}) g_T^{-1}(L_{K,D}) \sum_{i=1}^T \eta \left(I-\eta G'_+(0) L_{K,D}\right)^{T-i} E_{i,D,\sigma} \\
        =&\sum_{j=1}^m \frac{|D_j|}{|D|}\left(g_T(L_{K,D_j})-g_T(L_{K,D})\right)\left[g_T^{-1}(L_{K,D})\bar{f}_{T+1,D}^{l-1}-\hat{f}_{K,D}\right] \\
        &\qquad + \sum_{j=1}^m \frac{|D_j|}{|D|} \left(g_T(L_{K,D})-g_T(L_{K,D_j})\right)g_T^{-1}(L_{K,D}) \sum_{i=1}^T \eta \left(I-\eta G'_+(0) L_{K,D}\right)^{T-i}E_{i,D,\sigma}. 
    \end{aligned}
\end{equation*}
In order to simplify the above decomposition and represent $f_{T+1,D}-\bar{f}_{T+1,D}^l$ by $f_{T+1,D}-\bar{f}_{T+1,D}^{l-1}$ or $\bar{f}_{T+1,D}^{l-1}-f_{T+1,D}$, we rewrite $g_T^{-1}(L_{K,D})\bar{f}_{T+1,D}^{l-1}-\hat{f}_{K,D}$ as 
\begin{equation*}
    \begin{aligned}
        g_T^{-1}(L_{K,D})\bar{f}_{T+1,D}^{l-1}-\hat{f}_{K,D}&=g_T^{-1}(L_{K,D})\left[\bar{f}_{T+1,D}^{l-1}-g_T(L_{K,D})\hat{f}_{K,D}\right]\\
        &=g_T^{-1}(L_{K,D})\left[\bar{f}_{T+1,D}^{l-1}-f_{T+1,D}+f_{T+1,D}-g_T(L_{K,D})\hat{f}_{K,D}\right]\\
        &=g_T^{-1}(L_{K,D})\left[\bar{f}_{T+1,D}^{l-1}-f_{T+1,D}+\sum_{i=1}^T \eta  \left(I-\eta G'_+(0) L_{K,D}\right)^{T-i}E_{i,D,\sigma}\right],
    \end{aligned}
\end{equation*}
and further derive that
\begin{equation}\label{eq:tt}
    \begin{aligned}
        f_{T+1,D}-\bar{f}_{T+1,D}^l&=\sum_{j=1}^m \frac{|D_j|}{|D|}\left(g_T(L_{K,D_j})-g_T(L_{K,D})\right)g_T^{-1}(L_{K,D}) \left[\bar{f}_{T+1,D}^{l-1}-f_{T+1,D}\right] \\
        &\qquad+\sum_{j=1}^m \frac{|D_j|}{|D|}\left(g_T(L_{K,D_j})-g_T(L_{K,D})\right)g_T^{-1}(L_{K,D})\sum_{i=1}^T \eta  \left(I-\eta G'_+(0) L_{K,D}\right)^{T-i}E_{i,D,\sigma} \\
        &\qquad+ \sum_{j=1}^m \frac{|D_j|}{|D|} \left(g_T(L_{K,D})-g_T(L_{K,D_j})\right)g_T^{-1}(L_{K,D}) \sum_{i=1}^T \eta \left(I-\eta G'_+(0) L_{K,D}\right)^{T-i}E_{i,D,\sigma} \\
        &=\sum_{j=1}^m \frac{|D_j|}{|D|}\left(g_T(L_{K,D_j})-g_T(L_{K,D})\right)g_T^{-1}(L_{K,D})\left(\bar{f}_{T+1,D}^{l-1}-f_{T+1,D}\right). 
    \end{aligned}
\end{equation}
Note that it follows from the definition (\ref{eq:gt1}) of $ g_T(\cdot ) $ that
\begin{equation}\label{eq:358}
    \begin{aligned}
    &g_T(L_{K,D_j})-g_T(L_{K,D})\\
    =&\sum_{i=1}^T \eta G'_+(0)\left(I-\eta G'_+(0) L_{K,D_j}\right)^{T-i}-\sum_{i=1}^T \eta G'_+(0)\left(I-\eta G'_+(0) L_{K,D}\right)^{T-i}\\
    =&\sum_{i=1}^T \eta G'_+(0)\left(\left(I-\eta G'_+(0) L_{K,D_j}\right)^{T-i} - \left(I-\eta G'_+(0) L_{K,D}\right)^{T-i}\right)
    \end{aligned}
\end{equation}
To obtain a more refined estimate of $f_{T+1,D}-\bar{f}_{T+1,D}^l$, we need to adjust $\left(I-\eta G'_+(0) L_{K,D_j}\right)^{T-i} - \left(I-\eta G'_+(0) L_{K,D}\right)^{T-i}$ in (\ref{eq:358}).
To this end, we first introduce the following property that for two operators $X,Y$ and $n\in\mathbb{N}_+$, there holds
\begin{equation*}
    \begin{aligned}
    X^n-Y^n&=X^n-YX^{n-1}+YX^{n-1}-Y^2X^{n-2}+\cdots+Y^{n-1}X-Y^n\\
    &=(X-Y)X^{n-1}+Y(X-Y)X^{n-2}+\cdots+Y^{n-1}(X-Y)\\
    &=\cdots =\sum_{k=1}^{n} Y^{k-1} (X-Y) X^{n-k}.
    \end{aligned}
\end{equation*}
By applying the above property to $X=I-\eta G'_+(0) L_{K,D_j}$, $Y=I-\eta G'_+(0) L_{K,D}$ and $n=T-i$, we can deduce that
\begin{equation*}
    \begin{aligned}
    &\left(I-\eta G'_+(0) L_{K,D_j}\right)^{T-i} - \left(I-\eta G'_+(0) L_{K,D}\right)^{T-i}\\
    =&\sum_{k=1}^{T-i}\eta G'_+(0) \left(I-\eta G'_+(0) L_{K,D}\right)^{k-1} \left(L_{K,D}-L_{K,D_j}\right)\left(I-\eta G'_+(0) L_{K,D_j}\right)^{T-i-k} \\
    =&\sum_{k=i+1}^{T} \eta G'_+(0) \left(I-\eta G'_+(0) L_{K,D}\right)^{k-i-1} \left(L_{K,D}-L_{K,D_j}\right) \left(I-\eta G'_+(0) L_{K,D_j}\right)^{T-k}.
    \end{aligned}
\end{equation*}
Based on the results above, we can obtain an estimate of $\left\|\left(g_T(L_{K,D_j})-g_T(L_{K,D})\right)g_T^{-1}(L_{K,D})\right\|$ as follows
\begin{equation*}
    \begin{aligned}
        &\left\|\left(g_t(L_{K,D_j})-g_t(L_{K,D})\right)g_T^{-1}(L_{K,D})\right\|
        =\left\|g_T^{-1}(L_{K,D})\left(g_t(L_{K,D_j})-g_t(L_{K,D})\right)\right\| \\
        =&\left\|g_T^{-1}(L_{K,D})\sum_{i=1}^T \eta G'_+(0)\left(\left(I-\eta G'_+(0) L_{K,D_j}\right)^{T-i} - \left(I-\eta G'_+(0) L_{K,D}\right)^{T-i}\right)\right\| \\
        =&\left\|g_T^{-1}(L_{K,D})\sum_{i=1}^{T} \eta G'_+(0)  \sum_{k=i+1}^{T} \eta G'_+(0) \left(I-\eta G'_+(0) L_{K,D}\right)^{k-i-1} (L_{K,D}-L_{K,D_j}) \left(I-\eta G'_+(0) L_{K,D_j}\right)^{T-k} \right\|,
    \end{aligned}
\end{equation*}
where the first equality follows from the fact that $g_T(L_{K,D_j})-g_T(L_{K,D})$ and $g_T^{-1}(L_{K,D})$ are self-adjoint operators.
Furthermore, we can interchange the order of summation to obtain
\begin{equation*}
    \begin{aligned}
        &\left\|\left(g_t(L_{K,D_j})-g_t(L_{K,D})\right)g_T^{-1}(L_{K,D})\right\| \\
        =&\Biggl\|\sum_{k=2}^{T} \eta G'_+(0) g_T^{-1}(L_{K,D}) \sum_{i=1}^{k-1} \eta G'_+(0) \left(I-\eta G'_+(0) L_{K,D}\right)^{k-i-1} (L_{K,D}-L_{K,D_j}) \left(I-\eta G'_+(0) L_{K,D_j}\right)^{T-k} \Biggr\| \\
        \leq & \sum_{k=2}^{T} \eta G'_+(0) \Biggl\| g_T^{-1}(L_{K,D})   \sum_{i=1}^{k-1} \eta G'_+(0) \left(I-\eta G'_+(0) L_{K,D}\right)^{k-i-1}(L_{K,D}-L_{K,D_j}) \left(I-\eta G'_+(0) L_{K,D_j}\right)^{T-k} \Biggr\|\\
        \leq &\sum_{k=2}^{T} \eta G'_+(0) \left\| g_T^{-1}(L_{K,D})   \sum_{i=1}^{k-1} \eta G'_+(0) \left(I-\eta G'_+(0) L_{K,D}\right)^{k-i-1}\right\| \left\|(L_{K,D}-L_{K,D_j}) \left(I-\eta G'_+(0) L_{K,D_j}\right)^{T-k} \right\| \\
        \leq &\sum_{k=2}^{T} \eta G'_+(0)\left\|(L_{K,D}-L_{K,D_j}) \left(I-\eta G'_+(0) L_{K,D_j}\right)^{T-k} \right\|,
    \end{aligned}
\end{equation*}
where the last inequality follows from the fact that for any $ 2 \leq k \leq T $,
\[
    \left\| g_T^{-1}(L_{K,D})   \sum_{i=1}^{k-1} \eta G'_+(0) \left(I-\eta G'_+(0) L_{K,D}\right)^{k-i-1}\right\|\leq 1.
\]
Since $L_{K,D}-L_{K,D_j}$ can be decomposed as $L_{K,D}-L_K+L_K-L_{K,D_j}$, we see
\begin{align}\label{eq:gtgt}
    &\left\|\left(g_t(L_{K,D_j})-g_t(L_{K,D})\right)g_T^{-1}(L_{K,D})\right\| \notag \\ 
    \leq &\sum_{i=1}^{T} \eta G'_+(0)\left\|(L_{K,D}-L_{K})(\lambda I+L_{K})^{-\frac{1}{2}}\right\| \left\|(\lambda I+L_{K})^{\frac{1}{2}}(\lambda I+L_{K,D_j})^{-\frac{1}{2}}\right\| \left\|(\lambda I+L_{K,D_j})^{\frac{1}{2}}\left(I-\eta G'_+(0) L_{K,D_j}\right)^{T-i} \right\| \notag \\
    &+\sum_{i=1}^{T} \eta G'_+(0)\left\|(L_{K}-L_{K,D_j})(\lambda I+L_{K})^{-\frac{1}{2}}\right\| \left\|(\lambda I+L_{K})^{\frac{1}{2}}(\lambda I+L_{K,D_j})^{-\frac{1}{2}}\right\| \left\|(\lambda I+L_{K,D_j})^{\frac{1}{2}}\left(I-\eta G'_+(0) L_{K,D_j}\right)^{T-i} \right\| \notag \\
    = & (\mathcal{Q}_{D,\lambda}+ \mathcal{Q}_{D_j,\lambda}) \mathcal{A}_{D_j,\lambda} \sum_{i=1}^{T} \eta G'_+(0) \left\|(\lambda I+L_{K,D_j})^{\frac{1}{2}} \left(I-\eta G'_+(0) L_{K,D_j}\right)^{T-i} \right\|. 
\end{align}
Finally, it can be derived by combining (\ref{eq:tt}) and (\ref{eq:gtgt}) that
\begin{equation}\label{eq:communication}
    \begin{aligned}
        & \left\|f_{T+1,D}-\bar{f}_{T+1,D}^l\right\|_\rho \leq  \sum_{j=1}^m \frac{|D_j|}{|D|}\left\|\left(g_T(L_{K,D_j})-g_T(L_{K,D})\right)g_T^{-1}(L_{K,D})\right\| \left\|\bar{f}_{T+1,D}^{l-1}-f_{T+1,D}\right\|_\rho\\
        \leq &\left\{\sum_{j=1}^m \frac{|D_j|}{|D|}(\mathcal{Q}_{D,\lambda}+ \mathcal{Q}_{D_j,\lambda}) \mathcal{A}_{D_j,\lambda} \sum_{i=1}^{T} \eta G'_+(0) \left\|(\lambda I+L_{K,D_j})^{\frac{1}{2}} \left(I-\eta G'_+(0) L_{K,D_j}\right)^{T-i} \right\|\right\} \left\|\bar{f}_{T+1,D}^{l-1}-f_{T+1,D}\right\|_\rho\\
        \leq& \left\{\sum_{j=1}^m \frac{|D_j|}{|D|}(\mathcal{Q}_{D,\lambda}+ \mathcal{Q}_{D_j,\lambda}) \mathcal{A}_{D_j,\lambda} \sum_{i=1}^{T} \eta G'_+(0) \left\|(\lambda I+L_{K,D_j})^{\frac{1}{2}} \left(I-\eta G'_+(0) L_{K,D_j}\right)^{T-i} \right\|\right\}^l \left\|\bar{f}_{T+1,D}^{0}-f_{T+1,D}\right\|_\rho.
    \end{aligned}
\end{equation}
Recall that $ \bar{f}_{T+1,D}^{0}= \bar{f}_{T+1,D} $.
This decomposes the bound on $ \left\|f_{T+1,D}-\bar{f}_{T+1,D}^l\right\|_\rho $ into the product of an operator norm depending on $ l $ and the term $ \left\|\bar{f}_{T+1,D}-f_{T+1,D}\right\|_\rho $.
The latter can be further bounded via Proposition \ref{prop:prop2}.
Finally, the error decomposition for $ \left\|f_{T+1,D}-f_{\rho}\right\|_\rho $ is completed by combining (\ref{eq:tt00}) and (\ref{eq:communication}).

\section{Proofs of Main Results}\label{section:section5}
In this section, we present the proofs of the main results in Section \ref{section:section2}.
To this end, we need to introduce some preliminary estimates on the quantities defined in Section \ref{section:section4}, which are demonstrated in the following lemmas.
\subsection{Preliminary Lemmas}
\begin{lemma}\label{lemma:lemma3}
    Let $\{z_i=(x_i,y_i)\}_{i=1}^{|D|}$ be drawn independently according to $\rho$ and $0<\delta <1$.
    Under Assumption \ref{assum:1}, each of the following estimates holds with confidence at least $1-\delta$:
    \begin{align}
        &\mathcal{A}_{D,\lambda} = \left\|(\lambda I+L_K)^{\frac{1}{2}}(\lambda I+L_{K,D})^{-\frac{1}{2}}\right\| \leq \left(\frac{\mathcal{B}_{|D|,\lambda}}{\sqrt{\lambda}}+1\right)\log \frac{2}{\delta},  \label{eq:bound1} \\
        &\tilde{\mathcal{A}}_{D,\lambda}=\left\|(\lambda I+L_K)(\lambda I+L_{K,D})^{-1}\right\| \leq \left(\frac{\mathcal{B}_{|D|,\lambda}}{\sqrt{\lambda}}+1\right)^2\left(\log \frac{2}{\delta}\right)^2, \label{eq:bound110} \\
        &\mathcal{P}_{D,\lambda}=\left\|(\lambda I+L_K)^{-\frac{1}{2}}(\hat{f}_{K,D}-L_K f_\rho)\right\|_K\leq \mathcal{B}_{|D|,\lambda}\log \frac{2}{\delta}, \label{eq:bound2} \\
        &\mathcal{Q}_{D,\lambda}=\left\|(\lambda I+L_{K})^{-\frac{1}{2}}(L_K-L_{K,D})\right\|\leq \mathcal{B}_{|D|,\lambda}\log \frac{2}{\delta}, \label{eq:bound3} \\
        &\|L_K-L_{K,D}\|\leq \|L_K-L_{K,D}\|_{HS}\leq \frac{4\kappa^2}{\sqrt{|D|}}\log \frac{2}{\delta}, \label{eq:bound410}
    \end{align}
    where $\mathcal{B}_{|D|,\lambda}=\frac{2\kappa}{\sqrt{|D|}}\left\{\frac{\kappa}{\sqrt{|D|\lambda}}+\sqrt{\mathcal{N}(\lambda)}\right\}$ and $\|\cdot \|_{HS}$ denotes the norm of $ \mathrm{HS}(\mathcal{H}_{K})$, the Hilbert space of all Hilbert-Schmidt operators on $\mathcal{H}_K$.
\end{lemma}
The aforementioned bounds are well-established in the literature. 
Specifically, the proofs for (\ref{eq:bound1}), (\ref{eq:bound3}), and (\ref{eq:bound410}) follow from \cite{caponnetto2007optimal,lin2017distributed,zhou2002covering}, while (\ref{eq:bound110}) and (\ref{eq:bound2}) can be found in \cite{guoLearningTheoryDistributed2017} and \cite{caponnetto2007optimal}, respectively.
It is crucial to note that the high-probability bounds in Lemma \ref{lemma:lemma3} do not hold simultaneously.
Consequently, we must invoke the union bound in the subsequent analysis to control the overall failure probability.
This necessity is reflected in the accumulation of constants within the logarithmic terms of the final error bound.

By leveraging a new concentration inequality for self-adjoint operators in \cite{minsker2017some}, the following bound on $ \mathcal{R}_{D,\lambda} $ in Proposition \ref{prop:prop2} are proved in \cite{zhou2020distributed}, which contributes to loosening the restriction on the number of local machines $m$.
\begin{lemma}\label{lemma:lemma4}
    Let $0<\delta\leq 1$, if $0<\lambda\leq 1$ and $\mathcal{N}(\lambda)>1$, then there holds with confidence at least $1-\delta$ that
    \begin{equation*}
        \begin{aligned}
        \mathcal{R}_{D,\lambda}=&\|(\lambda I+L_{K})^{-\frac{1}{2}}(L_K-L_{K,D})(\lambda I+L_{K})^{-\frac{1}{2}}\|\\
        \leq& \max\left\{\frac{\kappa^2+1}{3},2\sqrt{\kappa^2+1}\right\}\left\{\frac{1+\log \mathcal{N}(\lambda)}{\lambda |D|}+\sqrt{\frac{1+\log \mathcal{N}(\lambda)}{\lambda |D|}}\right\}\log \frac{4}{\delta}\\
        =:&\mathcal{R}_{|D|,\lambda}\log \frac{4}{\delta}.
        \end{aligned}
    \end{equation*}        
\end{lemma}
Furthermore, leveraging Lemma \ref{lemma:lemma4}, we refine the bound on 
$ \mathcal{A}_{D,\lambda} $ in (\ref{eq:bound1}) as follows.
\begin{lemma}\label{lemma:lemmaa}
    Let $0<\delta\leq 1$, if $0<\lambda\leq 1$ and $\mathcal{N}(\lambda)>1$, then there holds with confidence at least $1-\delta$ that
    \begin{equation*}
        \mathcal{A}_{D,\lambda}\leq \left(1+\mathcal{R}_{|D|,\lambda}^2\left(\frac{\mathcal{B}_{|D|,\lambda}}{\sqrt{\lambda}}+1\right)^2  +\mathcal{R}_{|D|,\lambda}\right)^{\frac{1}{2}}\left(\log \frac{8}{\delta}\right)^2=:\mathcal{A}_{|D|,\lambda}\left(\log \frac{8}{\delta}\right)^2.
    \end{equation*}
\end{lemma}

\begin{proof}
To derive a tighter upper bound for $\mathcal{A}_{D,\lambda}$ than (\ref{eq:bound1}), consider
\begin{equation}\label{eq:abound}
    \begin{aligned}
    \mathcal{A}_{D,\lambda} &=\left\|(\lambda I+L_K)^{\frac{1}{2}}(\lambda I+L_{K,D})^{-\frac{1}{2}}\right\| =\left\|\left[(\lambda I+L_K)^{\frac{1}{2}}(\lambda I+L_{K,D})^{-1}(\lambda I+L_K)^{\frac{1}{2}}\right]^{\frac{1}{2}}\right\|\\
    &\leq \left\|(\lambda I+L_K)^{\frac{1}{2}}(\lambda I+L_{K,D})^{-1}(\lambda I+L_K)^{\frac{1}{2}}\right\|^{\frac{1}{2}},
    \end{aligned}
\end{equation}
then we only need to estimate $\left\|(\lambda I+L_K)^{\frac{1}{2}}(\lambda I+L_{K,D})^{-1}(\lambda I+L_K)^{\frac{1}{2}}\right\|$. 

It is derived by Lemma 16 in \cite{lin2017distributed} that for any two invertible operators $A$ and $B$ on a Banach space, there holds
$$
A^{-1}-B^{-1}=B^{-1}(B-A)B^{-1}(B-A)A^{-1}+B^{-1}(B-A)B^{-1}.
$$
With the above second order decomposition of inverse operator difference, we have
\begin{equation}\label{eq:decomposition}
    B^{\frac{1}{2}}A^{-1}B^{\frac{1}{2}}=B^{\frac{1}{2}}\left(A^{-1}-B^{-1}+B^{-1}\right)B^{\frac{1}{2}}=B^{-\frac{1}{2}}(B-A)B^{-1}(B-A)A^{-1}B^{\frac{1}{2}}+B^{-\frac{1}{2}}(B-A)B^{-\frac{1}{2}}+I. 
\end{equation}
Inserting $A=\lambda I+L_{K,D}$ and $B=\lambda I+L_K$ into (\ref{eq:decomposition}), we obtain
\begin{equation*}
    \begin{aligned}
    (\lambda I+L_K)^{\frac{1}{2}}(\lambda I+L_{K,D})^{-1}(\lambda I+L_K)^{\frac{1}{2}} =(\lambda I&+L_K)^{-\frac{1}{2}}(L_K-L_{K,D})(\lambda I+L_{K})^{-1}(L_K-L_{K,D})(\lambda I+L_{K,D})^{-1} \cdot \\
    &(\lambda I+L_K)^{\frac{1}{2}}+(\lambda I+L_K)^{-\frac{1}{2}}(L_K-L_{K,D})(\lambda I+L_K)^{-\frac{1}{2}}+I,
    \end{aligned}
\end{equation*}
and further derive that
\begin{equation*}
    \begin{aligned}
    &\left\|(\lambda I+L_K)^{\frac{1}{2}}(\lambda I+L_{K,D})^{-1}(\lambda I+L_K)^{\frac{1}{2}}\right\|\\
    \leq&\Bigl\|(\lambda I+L_K)^{-\frac{1}{2}}(L_K-L_{K,D})(\lambda I+L_K)^{-\frac{1}{2}}(\lambda I+L_K)^{-\frac{1}{2}}(L_K-L_{K,D})(\lambda I+L_K)^{-\frac{1}{2}}(\lambda I+L_K)^{\frac{1}{2}}(\lambda I+L_{K,D})^{-1} \cdot \\
    & \qquad (\lambda I+L_K)^{\frac{1}{2}} +(\lambda I+L_K)^{-\frac{1}{2}}(L_K-L_{K,D})(\lambda I+L_K)^{-\frac{1}{2}}+I\Bigr\|\\
    \leq&1+\mathcal{R}_{D,\lambda}^2\mathcal{A}_{D,\lambda}^2+\mathcal{R}_{D,\lambda}.
    \end{aligned}
\end{equation*}
Then for $0<\delta <1$, by Lemma \ref{lemma:lemma3} and Lemma \ref{lemma:lemma4} and scaling $ 2 \delta $ to $ \delta $, there holds with confidence at least $1−\delta $ that
\begin{equation}\label{eq:newbound}
    \begin{aligned}
    \left\|(\lambda I+L_K)^{\frac{1}{2}}(\lambda I+L_{K,D})^{-1}(\lambda I+L_K)^{\frac{1}{2}}\right\|& \leq  1+\mathcal{R}_{|D|,\lambda}^2\left(\frac{\mathcal{B}_{|D|,\lambda}}{\sqrt{\lambda}}+1\right)^2 \left(\log \frac{8}{\delta}\right)^2 \left(\log \frac{4}{\delta}\right)^2 +\mathcal{R}_{|D|,\lambda}\log \frac{8}{\delta} \\
    & \leq \left(1+\mathcal{R}_{|D|,\lambda}^2\left(\frac{\mathcal{B}_{|D|,\lambda}}{\sqrt{\lambda}}+1\right)^2  +\mathcal{R}_{|D|,\lambda}\right)\left(\log \frac{8}{\delta}\right)^4.
    \end{aligned}
\end{equation}
Finally, the proof of Lemma \ref{lemma:lemmaa} is completed by substituting (\ref{eq:newbound}) into (\ref{eq:abound}).
\end{proof}

To derive explicit learning rates based on Proposition~\ref{prop:prop4}, it is essential to first estimate the operator norms appearing in the upper bounds, specifically for $\left(I-\eta G'_+(0) L_{K}\right)^{T-i}$ and $\left(I-\eta G'_+(0) L_{K,D}\right)^{T-i}$.
These estimates are provided in the following two lemmas.
\begin{lemma}\label{lemma:lemma6}
    For any integer $ T \geq 3 $  and $0<\lambda\leq 1$, we have
    \begin{equation}\label{eq:eq82}
        \sum_{i=1}^T \eta G'_+(0) \left\|(\lambda I+L_{K,D})^{\frac{1}{2}} \left(I-\eta G'_+(0) L_{K,D}\right)^{T-i}\right\|\leq C'_\eta \left(\lambda^{\frac{1}{2}} T+ T^{\frac{1}{2}}\right),
    \end{equation}
    and
    \begin{equation}
        \sum_{i=1}^T \eta G'_+(0) \left\|(\lambda I+L_{K,D}) \left(I-\eta G'_+(0) L_{K,D}\right)^{T-i}\right\|\leq 2\eta G'_+(0)\lambda T+\left( 2\eta G'_+(0) \kappa^{2} + \frac{4}{e} \right)\log T,
    \end{equation}
    where $C'_\eta = \sqrt{2}\eta (1+\kappa) G'_+(0) + 2 \sqrt{2} \left(\frac{\eta G'_+(0)}{e}\right)^{\frac{1}{2}} $. The same bounds hold with $ L_{K,D} $ replaced by $ L_{K} $.
\end{lemma}

\begin{proof}
    By setting $\theta=0$ and $\eta_1=\eta$, Equation (21) in \cite[Prop.~4.2]{guo2018gradient} implies that for any $\tau,\lambda>0$ and $ 1\leqslant i<T $, there holds
    \begin{equation}\label{eq:21robust}
        \left\|(\lambda I+L_{K,D})^\tau \left(I-\eta G'_+(0) L_{K,D}\right)^{T-i}\right\|\leq 2^\tau\left(\lambda^\tau + \left(\frac{\tau}{eG'_+(0)}\right)^\tau \left(\eta (T-i)\right)^{-\tau} \right).
    \end{equation}
    To derive the first inequality in Lemma \ref{lemma:lemma6}, we can set $\tau=\frac{1}{2}$ in the above inequality to obtain
    \begin{equation*}
        \begin{aligned}
        &\sum_{i=1}^{T-1} \eta G'_+(0) \left\|(\lambda I+L_{K,D})^{\frac{1}{2}} \left(I-\eta G'_+(0) L_{K,D}\right)^{T-i}\right\|
        \leq \sqrt{2} \eta G'_+(0) \left(\sum_{i=1}^{T-1}\lambda^{\frac{1}{2}} + \sum_{i=1}^{T-1}\left(\frac{1}{eG'_+(0)}\right)^{\frac{1}{2}} (\eta (T-i))^{-\frac{1}{2}}\right).
        \end{aligned}
    \end{equation*}
    Since 
    \[
        \sum_{i=1}^{T-1} (T-i)^{-\frac{1}{2}}= \sum_{i=1}^{T-1} i^{-\frac{1}{2}} \leq 1+ \int_1^{T-1} x^{-\frac{1}{2}} dx \leq 2T^{\frac{1}{2}},
    \]
    we have 
    \begin{equation*}
        \begin{aligned}
        \sum_{i=1}^{T-1} \eta G'_+(0) \left\|(\lambda I+L_{K,D})^{\frac{1}{2}} \left(I-\eta G'_+(0) L_{K,D}\right)^{T-i}\right\|
        \leq \sqrt{2} \eta G'_+(0) \lambda^{\frac{1}{2}} T + 2\sqrt{2} \left(\frac{\eta G'_+(0)}{e}\right)^{\frac{1}{2}} T^{\frac{1}{2}}.
        \end{aligned}
    \end{equation*}
    This together with the fact that $\left\|(\lambda I+L_{K,D})^{\frac{1}{2}} \left(I-\eta G'_+(0) L_{K,D}\right)^{0}\right\|=\left\|(\lambda I+L_{K,D})^{\frac{1}{2}}\right\|\leq (\lambda + \kappa^{2})^{\frac{1}{2}} \leqslant \sqrt{2} \kappa$ completes the proof of the first inequality in Lemma \ref{lemma:lemma6}.

    By setting $\tau=1$ in \eqref{eq:21robust}, we can derive that
    \begin{equation*}
        \begin{aligned}
        &\sum_{i=1}^T \eta G'_+(0) \left\|(\lambda I+L_{K,D}) \left(I-\eta G'_+(0) L_{K,D}\right)^{T-i}\right\| \\ 
        & \leq \eta G'_+(0) \left\|\lambda I+L_{K,D} \right\| +2\eta G'_+(0) \left(\sum_{i=1}^{T-1}\lambda + \sum_{i=1}^{T-1}\left(\frac{1}{eG'_+(0)}\right) (\eta (T-i))^{-1}\right)\\
        & \leq  2\eta G'_+(0) \kappa^{2} + 2\eta G'_+(0)\lambda T+\frac{2}{e} \sum_{i=1}^{T-1} (T-i)^{-1} \\
        & \leq 2\eta G'_+(0) \kappa^{2} + 2\eta G'_+(0)\lambda T+\frac{4}{e}\log T \leq 2\eta G'_+(0)\lambda T+ \left( 2\eta G'_+(0) \kappa^{2} + \frac{4}{e} \right) \log T,
        \end{aligned}
    \end{equation*}
    which completes the proof of the second inequality in Lemma \ref{lemma:lemma6}.
    Finally, the same bounds apply to the population operator $ L_{K} $, as the preceding derivations rely solely on the operator being positive, compact, and satisfying $\|\cdot \|\leq \kappa^2$, common to both operators.
\end{proof}

\begin{lemma}\label{lemma:lemma9}
    For any integer $ T \geq 3 $, $0 < \lambda\leq 1$ and $r\geq \frac{1}{2}$, we have
    \begin{equation}
        \begin{aligned}
        &\sum_{i=1}^T \eta G'_+(0) \left\|(\lambda I+L_{K,D}) \left(I-\eta G'_+(0) L_{K,D}\right)^{T-i}\right\|i^{-(r-\frac{1}{2})}\\
        \leq & \left( \eta G'_+(0) (1 + \kappa^{2}) + (2\eta G'_+(0)  + \frac{2}{e}) C_{r,1} \right)  \max\left\{\lambda T^{-(r-\frac{3}{2})}\log T,\lambda,T^{-(r-\frac{1}{2})}\log T,T^{-1}\right\},
        \end{aligned}
    \end{equation}
    where $C_{r,1}$ is a constant given by \eqref{eq:Cr1}.
    The same bound holds with $ L_{K,D} $ replaced by $ L_{K} $.
\end{lemma}

\begin{proof}
    From \eqref{eq:21robust} with $ \tau = 1 $, we get for every $ 1 \le i < T $,
    \[
    \left\| (\lambda I + L_{K,D}) (I - \eta G'_+(0)L_{K,D})^{T-i} \right\| \le 2 \left( \lambda + \frac{1}{e\,G'_+(0)\,\eta\,(T-i)} \right).
    \]
    For the index $ i = T $, we have $ \eta G'_+(0) \|\lambda I + L_{K,D}\| T^{-(r-\frac{1}{2})}  \leq \eta G'_+(0) (\lambda + \kappa^{2}) T^{-(r-\frac{1}{2})} $ and thus
    \begin{equation}\label{eq:111}
        \begin{aligned}
            &\sum_{i=1}^T \eta G'_+(0) \left\| (\lambda I + L_{K,D}) (I - \eta G'_+(0)L_{K,D})^{T-i}  \right\|  \, i^{-(r-\frac{1}{2})} \\
            &\quad \le \eta G'_+(0) (\lambda + \kappa^{2}) T^{-(r-\frac{1}{2})} 
            + 2\eta G'_+(0)\lambda \sum_{i=1}^{T-1} i^{-(r-\frac{1}{2})}
            + \frac{2}{e} \sum_{i=1}^{T-1} \frac{i^{-(r-\frac{1}{2})}}{T-i}.
        \end{aligned}
    \end{equation}

    The bound for the first sum follows immediately from Lemma A.1 in \cite{hu2020distributed} that
    \[
        \sum_{i=1}^{T-1} i^{-(r-\frac{1}{2})} \leq \begin{cases}
            \frac{2T^{-(r-\frac{3}{2})}}{3-2r}, & \frac{1}{2} < r < \frac{3}{2},\\
            2\log T, & r = \frac{3}{2},\\
            \frac{2r-1}{2r-3}, & r > \frac{3}{2}.
        \end{cases}
    \]
    To bound $\sum_{i=1}^{T-1} \left(T-i\right)^{-1} i^{-(r-\frac{1}{2})}$, notice that when $1\leq i\leq\frac{T}{2}$, 
    \begin{equation*}
        \sum_{1\leq i\leq \frac{T}{2}} \left(T-i\right)^{-1} i^{-(r-\frac{1}{2})}\leq 2T^{-1} \sum_{1\leq i \leq \frac{T}{2}}i^{-(r-\frac{1}{2})}\leq \begin{cases}
            \frac{4T^{-(r-\frac{1}{2})}}{3-2r}, & \frac{1}{2} \leq r < \frac{3}{2},\\
            4 T^{-1}\log T, & r = \frac{3}{2},\\
            \frac{2(2r-1)}{2r-3} T^{-1}, & r > \frac{3}{2},
        \end{cases}
    \end{equation*}
    and for $\frac{T}{2}< i\leq T-1$,
    \begin{equation*}
        \sum_{\frac{T}{2}< i\leq T-1}  \left(T-i\right)^{-1} i^{-(r-\frac{1}{2})}\leq 2^{r-\frac{1}{2}}T^{-(r-\frac{1}{2})}\sum_{\frac{T}{2}< i\leq T-1}  \left(T-i\right)^{-1} \leq 2^{r-\frac{1}{2}}T^{-(r-\frac{1}{2})} \log T.
    \end{equation*}
    Therefore, we have
    \begin{equation*}
        \sum_{i=1}^{T-1} \left(T-i\right)^{-1} i^{-(r-\frac{1}{2})}\leq C_{r,1}\max\left\{T^{-(r-\frac{1}{2})}\log T,T^{-1}\right\},
    \end{equation*}
    where $C_{r,1}$ is a constant given by
    \begin{equation}\label{eq:Cr1}
        C_{r,1}=
            \begin{cases}
                \frac{4}{3-2r}+2^{r-\frac{1}{2}},&\frac{1}{2}\leq r<\frac{3}{2},\\
                6,&r=\frac{3}{2},\\
                \frac{2(2r-1)}{2r-3} +2^{r-\frac{1}{2}},&r>\frac{3}{2}.
            \end{cases}    
    \end{equation}
    To ensure consistency in the constant factors across our theoretical results, we alternatively derive 
    \begin{equation*}
        \sum_{i=1}^{T-1} i^{-(r-\frac{1}{2})}\leq C_{r,1} \max\left\{T^{-(r-\frac{3}{2})}\log T,1\right\}.
    \end{equation*}

    Finally, putting the bound on each sum term into (\ref{eq:111}), we have
    \begin{equation*}
        \begin{aligned}
        &\sum_{i=1}^T \eta G'_+(0) \left\|(\lambda I+L_{K,D}) \left(I-\eta G'_+(0) L_{K,D}\right)^{T-i}\right\|i^{-(r-\frac{1}{2})}\\
        \leq & \eta G'_+(0) (\lambda + \kappa^{2}) T^{-(r-\frac{1}{2})} + 2\eta G'_+(0) C_{r,1} \max\left\{ \lambda T^{-(r-\frac{3}{2})}\log T,\lambda \right\}+\frac{2}{e}C_{r,1}\max\left\{T^{-(r-\frac{1}{2})}\log T,T^{-1}\right\}\\
        \leq & \left( \eta G'_+(0) (1 + \kappa^{2}) + 2\eta G'_+(0) C_{r,1} + \frac{2}{e}C_{r,1} \right)  \max\left\{\lambda T^{-(r-\frac{3}{2})}\log T,\lambda,T^{-(r-\frac{1}{2})}\log T,T^{-1}\right\},
        \end{aligned}
    \end{equation*}
    which completes the proof of Lemma \ref{lemma:lemma9}.
\end{proof}

In the end of this subsection, we aim at estimating $E_{t,\sigma}$, $ \mathbb{E}^*[E_{t,D,\sigma}] $ and the difference between $E_{t,D,\sigma}$ and $E_{t,\sigma}$, respectively.
\begin{lemma}\label{lemma:lemma5}
    Under Assumption \ref{assum:1}, we have that for any $1\leq t \leq T$, 
    \begin{equation}\label{eq:bounde1}
        \|E_{t,\sigma}\|_K\leq \frac{\kappa c_p (M+\kappa\|f_t\|_K)^{2p+1}}{\sigma^{2p}}\leq \kappa c_p (2M)^{2p+1}\frac{t^{p+\frac{1}{2}}}{\sigma^{2p}},
    \end{equation}
    \begin{equation}\label{eq:bounde2}
        \left\|\mathbb{E}^*[E_{t,D,\sigma}]\right\|_K \leq \kappa c_p \left(2M\right)^{2p+1}\frac{t^{p+\frac{1}{2}}}{\sigma^{2p}},
    \end{equation}
    \begin{equation}\label{eq:bounde3}
        \|E_{t,D,\sigma}-E_{t,\sigma}\|_K \leq 2\kappa c_p \left(2M\right)^{2p+1}\frac{t^{p+\frac{1}{2}}}{\sigma^{2p}}. 
    \end{equation}
\end{lemma}

\begin{proof}
    Recall that $E_{t,\sigma}=\int_{\mathcal{Z}}\left(G'_+(0)-G'(\xi_{t,\sigma}(z))\right)\left(f_t(x)-y\right)K_x\mathrm{d}\rho$ and  $\xi_{t,\sigma}(z)=\frac{\left(y-f_t(x)\right)^2}{\sigma^2}$.
    Since $E_{t,\sigma} $ is the expectation of $E_{t,D,\sigma}$ with respect to the data $ D $, it follows that
    \begin{equation*}
        \begin{aligned}
            \|E_{t,\sigma}\|_K &\leq \int_{\mathcal{Z}} \left|G'_+(0)-G'(\xi_{t,\sigma}(z))\right| \left|y-f_t(x)\right| \|K_x\|_K d\rho \leq \frac{\kappa c_p (M+\kappa\|f_t\|_K)^{2p+1}}{\sigma^{2p}},
        \end{aligned}
    \end{equation*}
    where the last inequality is due to condition \eqref{eq:windowing2} of the function $ G $. Combining the above inequality with the bound on $\|f_t\|_K$ in Lemma \ref{lemma:lemma2} yields \eqref{eq:bounde1}.

    The uniform bound on $E_{t,D,\sigma}$ defined in (\ref{eq:EtD}) has been established in \cite[Proposition 5.1]{guo2018gradient}, based on which we can further set $ \eta_1=\eta $ and $ \theta=0 $ to derive that
    \[
        \|E_{t,D,\sigma}\|_K \leq \frac{\kappa c_p (M+\kappa\|f_{t,D}\|_K)^{2p+1}}{\sigma^{2p}} \leq \kappa c_p \left(2M\right)^{2p+1}\frac{t^{p+\frac{1}{2}}}{\sigma^{2p}}.
    \]
    Then (\ref{eq:bounde3}) can be proved by the triangle inequality $ \|E_{t,D,\sigma}-E_{t,\sigma}\|_K \leq \|E_{t,D,\sigma}\|_K +\|E_{t,\sigma}\|_K $.
    
    Finally, \eqref{eq:bounde2} follows from the Jensen’s inequality for conditional expectation and the bound on $\|E_{t,D,\sigma}\|_K$.
    This completes the proof of Lemma \ref{lemma:lemma5}.
\end{proof}

\subsection{Proof of Theorem \ref{thm:theorem1}}

We now proceed to prove Theorem \ref{thm:theorem1} by leveraging the error decomposition in Proposition \ref{prop:prop4}. Specifically, we bound each term on the right-hand side of (\ref{eq:prop4}) individually, and aggregate these estimates to obtain the desired learning rates.

For the first term of (\ref{eq:prop4}), combining Assumption \ref{assum:2}, \eqref{eq:bound2} and \eqref{eq:bound3} in Lemma \ref{lemma:lemma3} yields that with confidence at least $1-\delta$, there holds
\begin{equation}\label{eq:step3}
    \mathcal{P}_{D,\lambda}+\mathcal{Q}_{D,\lambda}\|f_\rho\|_K \leq \left(1+\kappa^{2r-1}\|u_\rho\|_\rho\right)\mathcal{B}_{|D|,\lambda}\log \frac{4}{\delta}.
\end{equation}

We then turn to estimate the second and third terms.
Lemma \ref{lemma:lemma5} provides a uniform bound for $\|E_{t,D,\sigma}-E_{t,\sigma}\|_K$ for any $t\in \{1,\ldots,T\}$.
Note that $T=\lceil \lambda^{-1} \rceil\in \mathbb{N}$, we have $ 1 \leq \lambda T < 2 $.
By combining Lemma \ref{lemma:lemma6} and substituting $ D $ with $ D_j $ in Lemma \ref{lemma:lemma5}, we have
\begin{equation}\label{eq:step1}
    \begin{aligned}
        &\sum_{i=1}^T \eta\left\|(\lambda I+L_K)^{\frac{1}{2}} \left(I-\eta G'_+(0) L_{K}\right)^{T-i}\right\| \sum_{j=1}^{m}\frac{|D_j|}{|D|}\left\|E_{i,D_j,\sigma}-E_{i,\sigma}\right\|_K\\
        \leq&\max_{1\leq j\leq m}\sum_{i=1}^T \eta\left\|(\lambda I+L_K)^{\frac{1}{2}} \left(I-\eta G'_+(0) L_{K}\right)^{T-i}\right\|\left\|E_{i,D_j,\sigma}-E_{i,\sigma}\right\|_K\\
        \leq & \frac{C'_\eta}{G'_+(0)} \left(\lambda^{\frac{1}{2}} T+ T^{\frac{1}{2}}\right)2\kappa c_p \left(2M\right)^{2p+1}\frac{T^{p+\frac{1}{2}}}{\sigma^{2p}}
        = C_{r,2}\frac{T^{p+1}}{\sigma^{2p}},
    \end{aligned}
\end{equation}
where $C_{r,2}:=2^{\frac{5}{2}}\frac{C'_\eta}{G'_+(0)} \kappa c_p \left(2M\right)^{2p+1}$ is a constant independent of $T$, $ \lambda $ and $\sigma$.
Moreover, by \eqref{eq:21robust} with $\tau=1$, we have
\begin{equation*}
    \begin{aligned}
        &\sum_{i=1}^T\eta G'_+(0) \left\|(\lambda I+L_K)\left(I-\eta G'_+(0) L_{K}\right)^{T-i}\right\| \|f_i-f_\rho\|_K\\
        \leq & 2 \eta G'_+(0) \left( \lambda + \frac{1}{e\,G'_+(0)\,\eta\,(T-1)} \right) \kappa^{2r-1} \left\| u_{\rho} \right\|_{\rho}  + \sum_{i=2}^T\eta G'_+(0) \left\|(\lambda I+L_K)\left(I-\eta G'_+(0) L_{K}\right)^{T-i}\right\| \|f_i-f_\rho\|_K.
    \end{aligned}
\end{equation*}
From \eqref{eq:prop3}, Lemma \ref{lemma:lemma6} and Lemma \ref{lemma:lemma9}, there holds
\begin{equation*}
    \begin{aligned}
        &\sum_{i=2}^T\eta G'_+(0) \left\|(\lambda I+L_K)\left(I-\eta G'_+(0) L_{K}\right)^{T-i}\right\| \|f_i-f_\rho\|_K\\
        \leq& \sum_{i=2}^T\eta G'_+(0) \left\|(\lambda I+L_K) \left(I-\eta G'_+(0) L_{K}\right)^{T-i}\right\|\left[\left(\frac{r-\frac{1}{2}}{e\eta G'_+(0)}\right)^{r-\frac{1}{2}} \|u_\rho\|_\rho (i-1)^{-(r-\frac{1}{2})}+\sum_{l=1}^{i-1} \eta \|E_{l,\sigma}\|_K\right]\\
        \leq & \left(\frac{r-\frac{1}{2}}{e\eta G'_+(0)}\right)^{r-\frac{1}{2}} \|u_\rho\|_\rho \left( \eta G'_+(0) (1 + \kappa^{2}) + (2\eta G'_+(0)  + \frac{2}{e}) C_{r,1} \right)  \max\left\{\lambda T^{-(r-\frac{3}{2})}\log T,\lambda,T^{-(r-\frac{1}{2})}\log T,T^{-1}\right\} \\
        &+ \left( 2\eta G'_+(0)\lambda T+\left( 2\eta G'_+(0) \kappa^{2} + \frac{4}{e} \right)\log T \right) \eta \kappa c_p \left(2M\right)^{2p+1}\frac{T^{p+\frac{3}{2}}}{\sigma^{2p}}.
    \end{aligned}
\end{equation*}
Therefore, we obtain that
\begin{equation}\label{eq:step2}
    \begin{aligned}
        &\sum_{i=1}^T\eta G'_+(0) \left\|(\lambda I+L_K)\left(I-\eta G'_+(0) L_{K}\right)^{T-i}\right\| \|f_i-f_\rho\|_K \leq C_{r,3}\max\left\{ T^{-(r-\frac{1}{2})}\log T,T^{-1},\frac{T^{p+\frac{3}{2}}\log T}{\sigma^{2p}}\right\}.
    \end{aligned}
\end{equation}
where $C_{r,3}:= 2 \left(\frac{r-\frac{1}{2}}{e\eta G'_+(0)}\right)^{r-\frac{1}{2}} \|u_\rho\|_\rho \left( \eta G'_+(0) (1 + \kappa^{2} +2 C_{r,1} ) +  \frac{2 C_{r,1} }{e} \right) + \left( 2\eta G'_+(0) (1 + \kappa^{2}) + \frac{4}{e} \right) \eta \kappa c_p \left(2M\right)^{2p+1} $.

For any fixed $j \in \{1, \ldots, m\}$, by applying Lemma \ref{lemma:lemma3}-\ref{lemma:lemmaa} with $ D $ replaced by $ D_j $, we can derive that with confidence at least $1−\frac{\delta}{m} $, there holds
\begin{equation*}
    \begin{aligned}
        \mathcal{R}_{D_j,\lambda}\mathcal{A}_{D_j,\lambda}^2 \left(\mathcal{P}_{D_j,\lambda}+\mathcal{Q}_{D_j,\lambda}\|f_\rho\|_K\right)
        \leq \left(1+\kappa^{2r-1}\|u_\rho\|_\rho\right)\mathcal{R}_{|D_j|,\lambda}\mathcal{B}_{|D_j|,\lambda}\mathcal{A}_{|D_j|,\lambda}^2\left(\log \frac{16m}{\delta}\right)^6.
    \end{aligned}
\end{equation*}
Since Lemma \ref{lemma:lemma6} and Lemma \ref{lemma:lemma9} also hold with $ L_{K} $ replaced by $ L_{K,D} $, we have the same bound in (\ref{eq:step1}) and (\ref{eq:step2}) when $ L_K $ is replaced by $ L_{K,D} $.
Therefore, combining (\ref{eq:step1}), (\ref{eq:step2}) and the above result, we can derive that with confidence at least $1−\frac{\delta}{m} $, there holds
\begin{equation*}
    \begin{aligned}
    \mathcal{K}&_{D_j,\lambda}\leq \left( 4\eta G'_+(0)+ 2\eta G'_+(0) \kappa^{2} + 4 / e  \right) \log T \cdot \left(1+ 2\eta G'_+(0) \right) \left(1+\kappa^{2r-1}\|u_\rho\|_\rho\right) \mathcal{R}_{|D_j|,\lambda}\mathcal{B}_{|D_j|,\lambda}\mathcal{A}_{|D_j|,\lambda}^2\left(\log \frac{16m}{\delta}\right)^6 \\
    & +\left( 4\eta G'_+(0)+ 2\eta G'_+(0) \kappa^{2} + 4 / e  \right) \log T \cdot C_{r,2} \frac{T^{p+1}}{\sigma^{2p}}\mathcal{R}_{|D_j|,\lambda}\mathcal{A}_{|D_j|,\lambda}\left(\log \frac{16m}{\delta}\right)^3 \\
    & + \left( 4\eta G'_+(0)+ 2\eta G'_+(0) \kappa^{2} + 4 / e  \right) \log T \cdot \mathcal{R}_{|D_j|,\lambda}\mathcal{B}_{|D_j|,\lambda}\mathcal{A}_{|D_j|,\lambda}^2 \left(\log \frac{16m}{\delta}\right)^6 C_{r,3}\max\left\{ T^{-(r-\frac{1}{2})}\log T,T^{-1},\frac{T^{p+\frac{3}{2}}\log T}{\sigma^{2p}}\right\} \\
    &\leq  C'_{r,2}\mathcal{R}_{|D_j|,\lambda}\mathcal{A}_{|D_j|,\lambda}\frac{T^{p+1}}{\sigma^{2p}} \log T \left(\log \frac{16m}{\delta}\right)^3 \\
    & \qquad+ C'_{r,3} \mathcal{R}_{|D_j|,\lambda}\mathcal{B}_{|D_j|,\lambda}\mathcal{A}_{|D_j|,\lambda}^2\max\left\{\log T, T^{-(r-\frac{1}{2})}(\log T)^{2},\frac{T^{p+\frac{3}{2}}(\log T)^2}{\sigma^{2p}}\right\}\left(\log \frac{16m}{\delta}\right)^6.
    \end{aligned}
\end{equation*}
where $C'_{r,2}:=\left( 4\eta G'_+(0)+ 2\eta G'_+(0) \kappa^{2} + 4 / e  \right)C_{r,2}$ and $C'_{r,3}:= \left( 4\eta G'_+(0)+ 2\eta G'_+(0) \kappa^{2} + 4 / e  \right) \left(1+ 2\eta G'_+(0) \right) (1+\kappa^{2r-1}  \cdot $ $\|u_\rho\|_\rho) + \left( 4\eta G'_+(0)+ 2\eta G'_+(0) \kappa^{2} + 4 / e  \right) C_{r,3} $.

Plugging the above estimates into \eqref{eq:prop4}, with confidence at least $1−\delta $, we have
\begin{equation}\label{eq:probbound}
    \begin{aligned}
    \left\|\bar{f}_{T+1,D}-f_\rho\right\|_\rho&\leq (1+2\eta G'_+(0)) (1+\kappa^{2r-1}\|u_\rho\|_\rho)\mathcal{B}_{|D|,\lambda}\log \frac{12}{\delta}+\left(\frac{r}{e\eta G'_+(0)}\right)^{r} \|u_\rho\|_\rho T^{-r}\\
    &+\frac{3C_{r,2}}{2} \frac{T^{p+1}}{\sigma^{2p}}+C_{r,3}\mathcal{B}_{|D|,\lambda}\max\left\{T^{-(r-\frac{1}{2})}\log T,T^{-1},\frac{T^{p+\frac{3}{2}}\log T}{\sigma^{2p}}\right\}\log \frac{6}{\delta}\\
    &+C'_{r,2} \max_{1\leq j\leq m} \left\{ \mathcal{R}_{|D_j|,\lambda}\mathcal{A}_{|D_j|,\lambda} \right\} \frac{T^{p+1}}{\sigma^{2p}} \log T \left(\log \frac{48m}{\delta}\right)^3 \\
    &+C'_{r,3} \max_{1\leq j\leq m} \left\{ \mathcal{R}_{|D_j|,\lambda}\mathcal{B}_{|D_j|,\lambda}\mathcal{A}_{|D_j|,\lambda}^2 \right\}  \max\left\{\log T, T^{-(r-\frac{1}{2})}(\log T)^{2},\frac{T^{p+\frac{3}{2}}(\log T)^2}{\sigma^{2p}}\right\}\left(\log \frac{48m}{\delta}\right)^6.
    \end{aligned}
\end{equation}
Note that the condition on $ m $ in \eqref{eq:m1} implies
\begin{equation}
    \label{eq:condition_m}
    m|D|^{-\frac{4r+2s-1}{4r+2s}}\leq \sqrt{m}|D|^{-\frac{r}{2r+s}} \qquad \text{ and } \qquad m|D|^{-\frac{2r+s-1}{2r+s}}\leq \sqrt{m}|D|^{-\frac{2r+s-1}{4r+2s}}.
\end{equation}
Let $\lambda=|D|^{-\frac{1}{2r+s}}$, then we can adopt Assumption \ref{assume:3} and $|D_1|=\cdots=|D_m|$ to obtain that
\[
    \mathcal{B}_{|D|,\lambda} \leq 2 \kappa \left( \kappa |D|^{-\frac{4r+2s-1}{4r+2s}} + \sqrt{C_0 } |D|^{-\frac{r}{2r+s}} \right)  \leq C_{0,1} |D|^{-\frac{r}{2r+s}},
\]
and for any $j\in \{1,\ldots,m\}$,
\[
    \mathcal{B}_{|D_j|,\lambda} \leq 2 \kappa \left( \kappa m |D|^{-\frac{4r+2s-1}{4r+2s}} + \sqrt{C_0 m} |D|^{-\frac{r}{2r+s}} \right)  \leq C_{0,1} \sqrt{m}|D|^{-\frac{r}{2r+s}},
\]
\[
    \begin{aligned}
        \mathcal{R}_{|D_j|,\lambda} & \leq \max\left\{\frac{\kappa^2+1}{3},2\sqrt{\kappa^2+1}\right\} \left\{ \frac{2s}{2r+s} m |D|^{-\frac{2r+s-1}{2r+s}} \log |D| + \sqrt{\frac{2s}{2r+s} m |D|^{-\frac{2r+s-1}{2r+s}} \log |D|} \right\} \\
        & \leq C_{0,2}\sqrt{m}|D|^{-\frac{2r+s-1}{4r+2s}}\log |D|,
    \end{aligned}
\]
where $C_{0,1}:=2\kappa(\kappa+\sqrt{C_0})$, and $C_{0,2}=2\max\left\{\frac{\kappa^2+1}{3},2\sqrt{\kappa^2+1}\right\}\sqrt{\frac{2s}{2r+s}}$.

Additionally, by leveraging the above results, we can derive the following bound for $\mathcal{A}_{|D|,\lambda}$:
\begin{equation*}
    \begin{aligned}
        \mathcal{A}^2_{|D_j|,\lambda}&=1+\frac{\mathcal{R}_{|D_j|,\lambda}^2\mathcal{B}_{|D_j|,\lambda}^2}{\lambda}+2\frac{\mathcal{R}_{|D_j|,\lambda}^2\mathcal{B}_{|D_j|,\lambda}}{\sqrt{\lambda}}+\mathcal{R}_{|D_j|,\lambda}^2+\mathcal{R}_{|D_j|,\lambda}\\
        &\leq 1+C_{0,1}^2C_{0,2}^2m^2|D|^{-\frac{4r+s-2}{2r+s}}\left(\log |D|\right)^2+2C_{0,1}C_{0,2}^2m^{\frac{3}{2}}|D|^{-\frac{3r+s-\frac{3}{2}}{2r+s}}\left(\log |D|\right)^2\\
        & \qquad+C_{0,2}^2m|D|^{-\frac{2r+s-1}{2r+s}}\left(\log |D|\right)^2+C_{0,2}\sqrt{m}|D|^{-\frac{2r+s-1}{4r+2s}}\log |D|\\
        & \leq C_{0,3}+1. 
    \end{aligned}  
\end{equation*}
where  $C_{0,3}=C_{0,1}^2C_{0,2}^2+2C_{0,1}C_{0,2}^2+C_{0,2}^2+C_{0,2}$ and the last inequality follows from the condition on $ m $ in \eqref{eq:m1}.

Finally, putting the aforementioned results, $\lambda=|D|^{-\frac{1}{2r+s}}$ and $T=\left\lceil |D|^{\frac{1}{2r+s}} \right\rceil$ into (\ref{eq:probbound}), we can derive that, with confidence at least $1−\delta $, there holds
\begin{equation*}
    \begin{aligned}
    &\left\|\bar{f}_{T+1,D}-f_\rho\right\|_\rho\leq \left[ C_{0,1} (1+2\eta G'_+(0)) (1+\kappa^{2r-1}\|u_\rho\|_\rho)+\left(\frac{r}{e\eta G'_+(0)}\right)^{r} \|u_\rho\|_\rho\right]|D|^{-\frac{r}{2r+s}}\log \frac{12}{\delta}\\
    &\quad+3\cdot  2^{p} C_{r,2} \frac{|D|^{\frac{p+1}{2r+s}}}{\sigma^{2p}}+ 2^{p+\frac{3}{2}} C_{r,3} C_{0,1} |D|^{-\frac{r}{2r+s}}\max\left\{|D|^{-\frac{r-\frac{1}{2}}{2r+s}}\log |D|,|D|^{-\frac{1}{2r+s}},\frac{|D|^{\frac{p+\frac{3}{2}}{2r+s}}\log |D|}{\sigma^{2p}}\right\}\log \frac{6}{\delta}\\
    &\quad+ 2^{p+1} C'_{r,2}(C_{0,3}+1)^{\frac{1}{2}}C_{0,2}\sqrt{m}|D|^{-\frac{2r+s-1}{4r+2s}}(\log |D|)^{2} \frac{|D|^{\frac{p+1}{2r+s}}}{\sigma^{2p}}\left(\log \frac{48m}{\delta}\right)^3 \\
    &\quad + 2^{p+\frac{3}{2}} C'_{r,3} C_{0,1}C_{0,2}(C_{0,3}+1)m|D|^{-\frac{4r+s-1}{4r+2s}} \max\left\{(\log |D|)^{2}, |D|^{-\frac{r-\frac{1}{2}}{2r+s}}(\log |D|)^{3}, \frac{|D|^{\frac{p+\frac{3}{2}}{2r+s}}\left(\log |D|\right)^{3}}{\sigma^{2p}}\right\}\left(\log \frac{48m}{\delta}\right)^6.
    \end{aligned}
\end{equation*}
Note that 
\[
    \log^{6} \frac{48m}{\delta} \leq 2^6 \left( \log^6 \frac{48}{\delta}+ \log^6 m \right) \leq  2^6 \log^6 \frac{48}{\delta}\left( \log^6 |D| +1 \right),
\]
and $ D^{-\alpha}\log |D|=\alpha^{-1}D^{-\alpha}\log |D|^{\alpha} \leq \alpha^{-1} $ with $ \alpha= \frac{r-\frac{1}{2}}{2r+s}$.
Since $ r>\frac{1}{2} $, by combining the condition on $ m $ in \eqref{eq:m1}, we have with confidence at least $1−\delta $,
\[
\left\|\bar{f}_{T+1,D}-f_\rho\right\|_\rho\leq \tilde{C}_1\max\left\{|D|^{-\frac{r}{2r+s}},\frac{|D|^{\frac{p+1}{2r+s}}}{{\sigma^{2p}}}\right\}\left(\log \frac{48}{\delta}\right)^6,
\]
where $\tilde{C}_1:= C_{0,1} (1+2\eta G'_+(0)) (1+\kappa^{2r-1}\|u_\rho\|_\rho)+ \left(\frac{r}{e\eta G'_+(0)}\right)^{r} \|u_\rho\|_\rho + 3\cdot  2^{p} C_{r,2} + 2^{p+\frac{3}{2}} C_{r,3} C_{0,1} \frac{4r+2s}{2r-1} + 2^{p+5} C'_{r,2}(C_{0,3}+1)^{\frac{1}{2}}C_{0,2} + 2^{p+\frac{15}{2}} C'_{r,3} C_{0,1}C_{0,2}(C_{0,3}+1) \frac{4r+2s}{2r-1}$.
This completes the proof of Theorem \ref{thm:theorem1}. \qed

\subsection{Proof of Theorem \ref{thm:theorem2}}
In this subsection, we aim to derive the explicit learning rates in expectation for the DKRGD based on the error decomposition in Subsection \ref{sec:5.3}.
To achieve this, we first establish the error bounds in probability for $ \left\|f_{T+1,D_j}-f_\rho\right\|_\rho $ and $ \left\|f_{T+1,D_j}^*-f_\rho\right\|_\rho $ for each $j\in \{1,\ldots,m\}$, and then obtain their expected error bounds by leveraging the formula
\begin{equation}
    \label{eq:expectation}
    \mathbb{E}[\xi]=\int_{0}^{\infty} \mathrm{Prob}(\xi>t) \mathrm{d}t,
\end{equation}
for nonnegative random variables $\xi$.

For the nonnegative random variable $\xi=\left\|\bar{f}_{T+1,D}-f_\rho\right\|_\rho^2$, (\ref{eq:prob1}) in Theorem \ref{thm:theorem1} immediately yields that, for any $ u $ satisfying $u\geq \tilde{C}_1^2\left(\log 48\right)^{12}\max\left\{|D|^{-\frac{r}{2r+s}},\frac{|D|^{\frac{p+1}{2r+s}}}{{\sigma^{2p}}}\right\}^2$, there holds 
$$
\mathrm{Prob}(\xi>u)=\mathrm{Prob}(\sqrt{\xi}>\sqrt{u})\leq 48\exp \left\{-\left(\tilde{C}_1\right)^{-\frac{1}{6}} \max\left\{|D|^{-\frac{r}{2r+s}},\frac{|D|^{\frac{p+1}{2r+s}}}{{\sigma^{2p}}}\right\}^{-\frac{1}{6}} u^{\frac{1}{12}}\right\}.
$$
Then applying the formula $\mathbb{E}[\xi]=\int_{0}^{\infty} \mathrm{Prob}(\xi>t) \mathrm{d}t$ to  $\xi$, we have
\begin{equation}\label{eq:eq107}
    \begin{aligned}
    &\mathbb{E}[\left\|\bar{f}_{T+1,D}-f_\rho\right\|_\rho^2] \\
    \leq& \tilde{C}_1^2\left(\log 48\right)^{12}\max\left\{|D|^{-\frac{r}{2r+s}},\frac{|D|^{\frac{p+1}{2r+s}}}{{\sigma^{2p}}}\right\}^2+48\int_{0}^{\infty} \exp \left\{-(\tilde{C}_1)^{-\frac{1}{6}} \max\left\{|D|^{-\frac{r}{2r+s}},\frac{|D|^{\frac{p+1}{2r+s}}}{{\sigma^{2p}}}\right\}^{-\frac{1}{6}} u^{\frac{1}{12}}\right\} \mathrm{d}u\\
    =&\tilde{C}_1^2\left(\log 48\right)^{12}\max\left\{|D|^{-\frac{r}{2r+s}},\frac{|D|^{\frac{p+1}{2r+s}}}{{\sigma^{2p}}}\right\}^2+576 \tilde{C}_1^2\max\left\{|D|^{-\frac{r}{2r+s}},\frac{|D|^{\frac{p+1}{2r+s}}}{{\sigma^{2p}}}\right\}^2\int_{0}^{\infty} t^{11} e^{-t}\mathrm{d}t\\
    =& \tilde{C}_1^2\left[\left(\log 48\right)^{12}+576\Gamma(12) \right]\max\left\{|D|^{-\frac{2r}{2r+s}},\frac{|D|^{\frac{2p+2}{2r+s}}}{{\sigma^{4p}}}\right\}.
    \end{aligned}
\end{equation}

We now turn to estimate the term $ \left\|f_{T+1,D_j}-f_\rho\right\|_\rho $.
By substituting the bounds in Lemma \ref{lemma:lemma3} and Lemma \ref{lemma:lemmaa} into Proposition \ref{prop:prop6} and employing the intermediate results in proving Theorem \ref{thm:theorem1}, we can derive that with confidence at least $1−3 \delta $, there holds
\begin{equation}\label{eq:ftd-ft}
    \begin{aligned}
    \left\|f_{T+1,D_j}-f_\rho\right\|_\rho &\leq (1+\eta G'_+(0) \lambda T)(1+\kappa^{2r-1}\|u_\rho\|_\rho)\mathcal{A}_{|D_j|,\lambda}^2\mathcal{B}_{|D_j|,\lambda}\left(\log \frac{8}{\delta}\right)^5\\
    &+ C_{r,2} \left(\mathcal{A}_{|D_j|,\lambda}+\frac{1}{2}\right)\frac{T^{p+1}}{\sigma^{2p}}\left(\log \frac{8}{\delta}\right)^2+\left(\frac{r}{e\eta G'_+(0)}\right)^{r} \|u_\rho\|_\rho T^{-r}\\
    &+C_{r,3} \mathcal{A}_{|D_j|,\lambda}^2\mathcal{B}_{|D_j|,\lambda}\max\left\{ T^{-(r-\frac{1}{2})}\log T,T^{-1},\frac{T^{p+\frac{3}{2}}\log T}{\sigma^{2p}}\right\} \left(\log \frac{8}{\delta}\right)^5.
    \end{aligned}
\end{equation}
Since the condition on $ m $ in \eqref{eq:m2} also implies \eqref{eq:condition_m}, we can adopt
the bounds on $\mathcal{B}_{|D_j|,\lambda}$, $\mathcal{R}_{|D_j|,\lambda}$ and $\mathcal{A}_{|D_j|,\lambda}$ mentioned in the proof of Theorem \ref{thm:theorem1} to further derive explicit rate in probability for $ \left\|f_{T+1,D_j}-f_\rho\right\|_\rho $.
Since $\lambda=|D|^{-\frac{1}{2r+s}}$ and $T=\left\lceil |D|^{\frac{1}{2r+s}} \right\rceil$, scaling $ 3\delta $ to $ \delta $ yields that with confidence at least $1−\delta $, there holds
\begin{equation*}
    \begin{aligned}
    \left\|f_{T+1,D_j}-f_\rho\right\|_\rho &\leq \tilde{C}_{r,2}\max\left\{\mathcal{B}_{|D_j|,\lambda},\frac{T^{p+1}}{\sigma^{2p}}, \lambda^{r},  \mathcal{B}_{|D_j|,\lambda} \frac{T^{p+\frac{3}{2}}\log T}{\sigma^{2p}}  \right\}\left(\log \frac{24}{\delta}\right)^5,
    \end{aligned}
\end{equation*}
where $\tilde{C}_{r,2}:=\left((1+ 2 \eta G'_+(0))(1+\kappa^{2r-1}\|u_\rho\|_\rho)+\frac{3}{2}C_{r,2}+ \frac{2 C_{r,3}}{2r-1} \right)(C_{0,3}+1) + \left(\frac{r}{e\eta G'_+(0)}\right)^{r} \|u_\rho\|_\rho $.
Then, we can easily apply \eqref{eq:expectation} to obtain that for any $ j \in \{ 1, \ldots ,m \} $,
\begin{equation}
    \label{eq:expectation1}
    \begin{aligned}
        \mathbb{E}\left[\left\|f_{T+1,D_j}-f_\rho\right\|_\rho^2\right] &\leq \tilde{C}_{r,2}^2 \max\left\{\mathcal{B}_{|D_j|,\lambda},\frac{T^{p+1}}{\sigma^{2p}}, \lambda^{r},  \mathcal{B}_{|D_j|,\lambda} \frac{T^{p+\frac{3}{2}}\log T}{\sigma^{2p}}  \right\}^2\left(\log 24\right)^{10}\\
        &+24\int_{0}^{\infty} \exp \left\{-(\tilde{C}_{r,2})^{-\frac{1}{5}}  \max\left\{\mathcal{B}_{|D_j|,\lambda},\frac{T^{p+1}}{\sigma^{2p}}, \lambda^{r},  \mathcal{B}_{|D_j|,\lambda} \frac{T^{p+\frac{3}{2}}\log T}{\sigma^{2p}} \right\}^{-\frac{1}{5}} u^{\frac{1}{10}}\right\} \mathrm{d}u\\
        &=\tilde{C}_{r,2}^2\left[\left(\log 24\right)^{10}+240\Gamma(10)\right] \max\left\{\mathcal{B}_{|D_j|,\lambda},\frac{T^{p+1}}{\sigma^{2p}}, \lambda^{r},  \mathcal{B}_{|D_j|,\lambda} \frac{T^{p+\frac{3}{2}}\log T}{\sigma^{2p}} \right\}^2.
    \end{aligned}
\end{equation}

In the following, we will estimate the term $ \left\|f_{T+1,D_j}^*-f_\rho\right\|_\rho $ for each $j\in \{1,\ldots,m\}$ by applying the derived decomposition of $f_{T+1,D_j}^*-f_\rho$ for different cases of $r$ in Subsection \ref{sec:5.3}.
By invoking Equation (22) in \cite[Prop.~4.2]{guo2018gradient} under the constant step size $ \eta $, we obtain that for any $\tau,\lambda>0$, there holds
$$
\left\|(\lambda I+L_{K,D})^\tau \left(I-\eta G'_+(0) L_{K,D}\right)^{T}\right\|=C_\tau\left(\lambda^\tau + T^{-\tau} \right).
$$
where $C_\tau:= 2^\tau\left[1 + \left(\frac{\tau}{eG'_+(0)}\right)^\tau \eta^{-\tau} \right]$.
For brevity, we denote $ C_{r} := C_{\tau}|_{\tau=r} $.

For the regime $\frac{1}{2} < r\leq 1$, by substituting the corresponding estimates in Lemma \ref{lemma:lemmaa}, Lemma \ref{lemma:lemma6} and Lemma \ref{lemma:lemma5} into (\ref{eq:expecr1}), we obtain that with confidence at least $1-\delta$,
\begin{equation*}
    \begin{aligned}
        \left\|f_{T+1,D_j}^*-f_\rho\right\|_\rho&\leq C_{r}\left(\lambda^r+T^{-r}\right) \mathcal{A}_{|D_j|,\lambda}^{2r}\left(\log \frac{8}{\delta}\right)^{4r} \left\| u_{\rho} \right\|_{\rho}  +\frac{C'_\eta}{G'_+(0)} \left(\lambda^{\frac{1}{2}} T + T^{\frac{1}{2}}\right) \kappa c_p \left(2M\right)^{2p+1}\frac{T^{p+\frac{1}{2}}}{\sigma^{2p}}.
    \end{aligned}
\end{equation*}
For $\lambda=|D|^{-\frac{1}{2r+s}}$ and $T=\left\lceil |D|^{\frac{1}{2r+s}} \right\rceil$, combining the partition size $m$ specified in (\ref{eq:m2}) and the bound on $ \mathcal{A}_{|D_j|,\lambda} $, the preceding inequality simplifies to the following high-probability guarantee:
\begin{equation*}
    \begin{aligned}
        \left\|f_{T+1,D_j}^*-f_\rho\right\|_\rho&\leq C_{r,4}\max\left\{\lambda^r,\frac{T^{p+1}}{\sigma^{2p}}\right\}\left(\log \frac{8}{\delta}\right)^{4r},
    \end{aligned}
\end{equation*}
where $C_{r,4}:=2C_{r}(C_{0,3}+1)^r \left\| u_{\rho} \right\|_{\rho} + \frac{C'_\eta}{G'_+(0)} (\sqrt{2}+1) \kappa c_p \left(2M\right)^{2p+1}$.
Finally, applying the tail-expectation identity \eqref{eq:expectation} to the random variable $ \xi= \left\|f_{T+1,D_j}^*-f_\rho\right\|_\rho $ yields the expected error bound
\begin{equation*}
    \mathbb{E} \left[  \left\|f_{T+1,D_j}^*-f_\rho\right\|_{\rho} \right]  \leq C_{r,4}\left[(\log 8)^{4r} + 32r\Gamma(4r)\right]\max\left\{\lambda^r,\frac{T^{p+1}}{\sigma^{2p}}\right\}.
\end{equation*}

For $1 < r\leq \frac{3}{2}$, combining (\ref{eq:expecr2}), Lemma \ref{lemma:lemma3}, Lemma \ref{lemma:lemma6} and Lemma \ref{lemma:lemma5}, it follows that with confidence at least $1-\delta$,
\begin{equation*}
    \begin{aligned}
        \left\|f_{T+1,D_j}^*-f_\rho\right\|_\rho&\leq C_{r}\left(\lambda^r+T^{-r}\right)\left(\frac{\mathcal{B}_{|D_j|,\lambda}}{\sqrt{\lambda}}+1\right)^{2r}\left(\log \frac{2}{\delta}\right)^{2r} \left\| u_{\rho} \right\|_{\rho} + \frac{C'_\eta}{G'_+(0)} \left(\lambda^{\frac{1}{2}} T+ T^{\frac{1}{2}}\right) \kappa c_p \left(2M\right)^{2p+1}\frac{T^{p+\frac{1}{2}}}{\sigma^{2p}}.
    \end{aligned}
\end{equation*}
Note that $\lambda=|D|^{-\frac{1}{2r+s}}$ and $T=\left\lceil |D|^{\frac{1}{2r+s}} \right\rceil$.
When $r>1$ and $m$ satisfies (\ref{eq:m2}), we have
\[
    \frac{\mathcal{B}_{|D_j|,\lambda}}{\sqrt{\lambda}}+1\leq C_{0,1}\sqrt{m}|D|^{-\frac{r-\frac{1}{2}}{2r+s}} +1 \leq C_{0,1}+1.
\]
Then, we have that with confidence at least $1−\delta $, there holds
\begin{equation*}
    \begin{aligned}
        \left\|f_{T+1,D_j}^*-f_\rho\right\|_\rho&\leq C_{r,5}\max\left\{\lambda^r,\frac{T^{p+1}}{\sigma^{2p}}\right\}\left(\log \frac{2}{\delta}\right)^{2r}.
    \end{aligned}
\end{equation*}
where $C_{r,5}:=2 C_{r} (C_{0,1}+1)^{2r} \left\| u_{\rho} \right\|_{\rho} + \frac{C'_\eta}{G'_+(0)} (\sqrt{2}+1) \kappa c_p \left(2M\right)^{2p+1}$.
Further, we can also derive the expected error bound
\begin{equation*}
    \mathbb{E}\left[\left\|f_{T+1,D_j}^*-f_\rho\right\|_{\rho} \right]\leq C_{r,5}\left[(\log 2)^{2r} + 4r\Gamma(2r)\right]\max\left\{\lambda^r,\frac{T^{p+1}}{\sigma^{2p}}\right\}.
\end{equation*}

For the case of $\frac{3}{2}<r\leq \frac{5}{2}$, putting the estimates in  Lemma \ref{lemma:lemma3}, Lemma \ref{lemma:lemma6} and Lemma \ref{lemma:lemma5} into (\ref{eq:expecr3}) yields that with confidence at least $1-\delta$,
\begin{equation*}
    \begin{aligned}
        \bigl\|f_{T+1,D_j}^*-f_\rho\bigr\|_\rho& \leq C_{r}\left(\lambda^r+T^{-r}\right)\|u_\rho\|_\rho\left(\frac{\mathcal{B}_{|D_j|,\lambda}}{\sqrt{\lambda}}+1\right)\left(1+\left(\frac{\mathcal{B}_{|D_j|,\lambda}}{\sqrt{\lambda}}+1\right)^{2r-3}\right)\left(\log \frac{4}{\delta}\right)^{3}\\
        & \qquad +C_{1} \left(\lambda+T^{-1}\right) \|u_\rho\|_\rho(1 + \kappa^{2}) \left(\frac{\mathcal{B}_{|D_j|,\lambda}}{\sqrt{\lambda}}+1\right)^{2} \mathcal{B}_{|D_j|,\lambda}\left(\log \frac{4}{\delta}\right)^{3}\\
        & \qquad +C_{r} \left(\lambda^r+T^{-r}\right) \left(\frac{\mathcal{B}_{|D_j|,\lambda}}{\sqrt{\lambda}}+1\right) \|u_\rho\|_\rho \log \frac{4}{\delta} + \frac{C'_\eta}{G'_+(0)} \left(\lambda^{\frac{1}{2}} T+ T^{\frac{1}{2}}\right) \kappa c_p \left(2M\right)^{2p+1}\frac{T^{p+\frac{1}{2}}}{\sigma^{2p}}.
    \end{aligned}
\end{equation*}
Analogously, for $\lambda=|D|^{-\frac{1}{2r+s}}$ and $T=\left\lceil |D|^{\frac{1}{2r+s}} \right\rceil$, we have that with confidence at least $1−\delta $,
\begin{equation*}
    \begin{aligned}
        \bigl\|f_{T+1,D_j}^*-f_\rho\bigr\|_\rho &\leq C_{r,6}\max\left\{\lambda^r,\lambda \mathcal{B}_{|D_j|,\lambda},\frac{T^{p+1}}{\sigma^{2p}}\right\}\left(\log \frac{4}{\delta}\right)^{3}.
    \end{aligned}
\end{equation*}
where 
$C_{r,6}:=6 C_{r} \left\| u_{\rho} \right\|_{\rho} (C_{0,1}+1)^{2r-2} + 2 C_{1}  \left\| u_{\rho} \right\|_{\rho} (1+\kappa^{2}) (C_{0,1}+1)^{2} + \frac{C'_\eta}{G'_+(0)} (\sqrt{2}+1) \kappa c_p \left(2M\right)^{2p+1} $.
Then it can be easily derived that
\begin{equation*}
    \mathbb{E}\left[\bigl\|f_{T+1,D_j}^*-f_\rho\bigr\|_\rho\right]\leq C_{r,6}\left[(\log 4)^{3} + 12\Gamma(3)\right]\max\left\{\lambda^r,\lambda \mathcal{B}_{|D_j|,\lambda},\frac{T^{p+1}}{\sigma^{2p}}\right\}.
\end{equation*}

When $r>\frac{5}{2}$, by (\ref{eq:expecr4}), Lemma \ref{lemma:lemma3}, Lemma \ref{lemma:lemma6} and Lemma \ref{lemma:lemma5}, with confidence at least $1-\delta$, there holds
\begin{equation*}
    \begin{aligned}
        \bigl\|f_{T+1,D_j}^*-f_\rho\bigr\|_\rho & \leq \left(r-\frac{3}{2}\right) \kappa^{2r-5} C_{\frac{3}{2}} \left(\lambda^{\frac{3}{2}}+T^{-\frac{3}{2}}\right) \frac{4\kappa^{2}}{\sqrt{|D_j|}}  \|u_\rho\|_\rho \left(\frac{\mathcal{B}_{|D_j|,\lambda}}{\sqrt{\lambda}}+1\right) \left( \log \frac{6}{\delta} \right)^{2} \\
        & \qquad + C_{1} (1 + \kappa^{2})^{r-\frac{3}{2}} \|u_\rho\|_\rho \left(\lambda+T^{-1}\right)\left(\frac{\mathcal{B}_{|D_j|,\lambda}}{\sqrt{\lambda}}+1\right)^2  \mathcal{B}_{|D_j|,\lambda}\left(\log \frac{6}{\delta}\right)^{3}\\
        & \qquad + C_{r} \left(\lambda^r+T^{-r}\right) \left(\frac{\mathcal{B}_{|D_j|,\lambda}}{\sqrt{\lambda}}+1\right) \|u_\rho\|_\rho \log \frac{6}{\delta} + \frac{C'_\eta}{G'_+(0)} \left(\lambda^{\frac{1}{2}} T+ T^{\frac{1}{2}}\right) \kappa c_p \left(2M\right)^{2p+1}\frac{T^{p+\frac{1}{2}}}{\sigma^{2p}},
    \end{aligned}
\end{equation*}
Note that with the partition size $m$ specified in (\ref{eq:m2}), we also have $ \lambda^{-\frac{1}{2}} \mathcal{B}_{|D_j|,\lambda} \leq C_{0,1} +1 $.
Hence, for $\lambda=|D|^{-\frac{1}{2r+s}}$ and $T=\left\lceil |D|^{\frac{1}{2r+s}} \right\rceil$, there holds with confidence at least $1−\delta $ that
\begin{equation*}
    \begin{aligned}
        \bigl\|f_{T+1,D_j}^*-f_\rho\bigr\|_\rho &\leq C_{r,7}\max\left\{\lambda^r,\frac{\lambda^{\frac{3}{2}}}{\sqrt{|D_j|}},\lambda \mathcal{B}_{|D_j|,\lambda},\frac{T^{p+1}}{\sigma^{2p}}\right\}\left(\log \frac{6}{\delta}\right)^{3}.
    \end{aligned}
\end{equation*}
where $C_{r,7}:= \|u_\rho\|_\rho (C_{0,1}+1)^{2} \left[ (8r-12) \kappa^{2r-3} C_{\frac{3}{2}} + 2C_{1} (1+\kappa^{2})^{r-\frac{3}{2}} +2 C_{r} \right] + \frac{C'_\eta}{G'_+(0)} (\sqrt{2}+1) \kappa c_p \left(2M\right)^{2p+1} $.
Further, we can apply \eqref{eq:expectation} to obtain
\begin{equation*}
    \mathbb{E}\left[\bigl\|f_{T+1,D_j}^*-f_\rho\bigr\|_\rho\right]\leq C_{r,7}\left[(\log 6)^{3} + 18\Gamma(3)\right]\max\left\{\lambda^r,\frac{\lambda^{\frac{3}{2}}}{\sqrt{|D_j|}},\lambda \mathcal{B}_{|D_j|,\lambda},\frac{T^{p+1}}{\sigma^{2p}}\right\}.
\end{equation*}

Before concluding the proof, we remark that for $ r > \frac{5}{2} $, the constraint on partition size $ m $ derived from the above bound is more stringent than condition \eqref{eq:m2}.
Specifically, ensuring $ \frac{\lambda^{\frac{3}{2}}}{\sqrt{|D_j|}} = \mathcal{O}(|D|^{-\frac{2r}{2r+s}})$ requires $ m \leq |D|^{\frac{3+s}{2r+s}} $, which imposes a tighter restriction on $ m $ for lager $r$.
To overcome this bottleneck, we employ (\ref{eq:eq107}) to derive the explicit learning rates in Theorem \ref{thm:theorem2} for $r>\frac{5}{2}$.
Conversely, for $ \frac{5-s}{2} < r \leq \frac{5}{2} $, condition \eqref{eq:m1} is less restrictive than \eqref{eq:m2}; thus, we also apply \eqref{eq:eq107} to establish the rates.
This distinction clarifies why the regularity regimes in \eqref{eq:m2} differ from those implied by the intermediate technical conditions.

Finally, to derive the explicit learning rates in Theorem \ref{thm:theorem2}, we need the following results
$$
\lambda^2 \mathcal{B}_{|D_j|,\lambda}^2=C_{0,1}^2 m |D|^{-\frac{2r+2}{2r+s}}\leq C_{0,1}^2 |D|^{-\frac{2r}{2r+s}},
$$
and
$$
\sum_{j=1}^{m} \frac{|D_j|^2}{|D|^2} \mathcal{B}_{|D_j|,\lambda}^2=\sum_{j=1}^{m} \frac{|D_j|^2}{|D|^2} C_{0,1}^{2} m |D|^{-\frac{2r}{2r+s}} \leq C_{0,1}^2 |D|^{-\frac{2r}{2r+s}},
$$
which can be verified by the partition size $m$ specified in (\ref{eq:m2}).
Since $\left\|\mathbb{E}\left[f_{T+1,D_j}\right]-f_\rho\right\|_\rho\leq \mathbb{E} \left[ \left\|f_{T+1,D_j}^*-f_\rho\right\|_{\rho} \right] $, it follows from (\ref{eq:expectationeq}) that
\begin{equation*}
    \begin{aligned}
    \mathbb{E}\left[\left\|\bar{f}_{T+1,D}-f_\rho\right\|_\rho^2\right] \leq  \sum_{j=1}^{m} \frac{|D_j|^2}{|D|^2} \mathbb{E}\left[\left\|f_{T+1,D_j}-f_\rho\right\|_\rho^2\right]+\sum_{j=1}^{m} \frac{|D_j|}{|D|} \left\{\mathbb{E}\left[\left\|f_{T+1,D_j}^*-f_\rho\right\|_\rho\right]\right\}^2.
    \end{aligned}
\end{equation*}
For $ \frac{1}{2} < r \leq \frac{5-s}{2} $, combining \eqref{eq:expectation1} and the above bounds on $ \mathbb{E} \left[ \left\|f_{T+1,D_j}^*-f_\rho\right\|_{\rho} \right] $ yields that
\begin{equation*}
    \begin{aligned}
    & \mathbb{E}\Bigl[\bigl\|\bar{f}_{T+1,D}-f_\rho\bigr\|_\rho^2\Bigr] \\
    \leq{} & A_r \sum_{j=1}^{m} \frac{|D_j|^2}{|D|^2} 
    \max\left\{\mathcal{B}_{|D_j|,\lambda},\frac{T^{p+1}}{\sigma^{2p}}, \lambda^{r},  
    \mathcal{B}_{|D_j|,\lambda} \frac{T^{p+\frac{3}{2}}\log T}{\sigma^{2p}} \right\}^2 + B_r \sum_{j=1}^{m} \frac{|D_j|}{|D|} 
    \max \left\{ \lambda^r,\frac{T^{p+1}}{\sigma^{2p}}, \lambda \mathcal{B}_{|D_j|,\lambda} \right\}^{2} \\
    \leq{} &  2^{2p+2} C_{0,1}^{2} \left( A_r  \frac{8(2r+s)^{2}}{(2r-1)^{2}} + B_r   \right)
    \max\left\{ |D|^{-\frac{2r}{2r+s}}, \frac{T^{2p+2}}{\sigma^{4p}} \right\},
    \end{aligned}
\end{equation*}
where $A_r:=\tilde{C}_{r,2}^2\bigl[\left(\log 24\right)^{10}+240\Gamma(10)\bigr]$ and $B_r:= C_{r,4}^{2} \left[(\log 8)^{4r} + 32r\Gamma(4r)\right]^{2} + C_{r,5}^2\bigl[\left(\log 2\right)^{3} + 6\Gamma(3)\bigr]^2+C_{r,6}^2\bigl[\left(\log 6\right)^{3} + 18\Gamma(3)\bigr]^2$.

The proof is completed by combining \eqref{eq:eq107} for $r>\frac{5-s}{2}$ and the above bound for $ \frac{1}{2} < r \leq \frac{5-s}{2} $, where
\begin{align*}
    \tilde{C}_2&:=\tilde{C}_1^2\bigl[ \left(\log 48\right)^{12}+576\Gamma(12) \bigr] + 2^{2p+2} C_{0,1}^{2} \left( A_r  \frac{8(2r+s)^{2}}{(2r-1)^{2}} + B_r   \right).
\end{align*}

\subsection{Proof of Theorem \ref{thm:theorem3}}

To prove Theorem \ref{thm:theorem3}, we adopt the error decomposition in Subsection \ref{sec:5.4} and establish the high-probability bounds for $ \left\|\bar{f}_{T+1,D}^l-f_{T+1,D}\right\|_\rho $ and $ \left\|f_{T+1,D}-f_\rho\right\|_\rho $, respectively.

First, we can directly apply the derived high-probability bounds for $ \left\|f_{T+1,D_j}-f_\rho\right\|_\rho $ to estimate $ \left\|f_{T+1,D}-f_\rho\right\|_\rho $ by substituting $ D_j $ with $ D $ in \eqref{eq:ftd-ft}.
Specifically, there holds with confidence at least $1-\delta$ that
\begin{equation*}
    \begin{aligned}
    \left\|f_{T+1,D}-f_\rho\right\|_\rho &\leq (1+\eta G'_+(0) \lambda T)(1+\kappa^{2r-1}\|u_\rho\|_\rho)\mathcal{A}_{|D|,\lambda}^2\mathcal{B}_{|D|,\lambda}\left(\log \frac{24}{\delta}\right)^5\\
    &+ C_{r,2} \left(\mathcal{A}_{|D|,\lambda}+\frac{1}{2}\right)\frac{T^{p+1}}{\sigma^{2p}}\left(\log \frac{24}{\delta}\right)^2+\left(\frac{r}{e\eta G'_+(0)}\right)^{r} \|u_\rho\|_\rho T^{-r}\\
    &+C_{r,3} \mathcal{A}_{|D|,\lambda}^2\mathcal{B}_{|D|,\lambda}\max\left\{ T^{-(r-\frac{1}{2})}\log T,T^{-1},\frac{T^{p+\frac{3}{2}}\log T}{\sigma^{2p}}\right\} \left(\log \frac{24}{\delta}\right)^5.
    \end{aligned}
\end{equation*}
Since $ \lambda= |D|^{-\frac{1}{2r+s}} $ and $ T=\left\lceil |D|^{\frac{1}{2r+s}} \right\rceil $, it can be directly derived that $ \mathcal{B}_{|D|,\lambda} \leq C_{0,1} |D|^{-\frac{r}{2r+s}} $ and $ \mathcal{A}_{|D|,\lambda}^2 \leq C_{0,3} +1 $ by taking $ m=1 $ on the bounds on $\mathcal{B}_{|D_j|,\lambda}$, $\mathcal{R}_{|D_j|,\lambda}$ and $\mathcal{A}_{|D_j|,\lambda}$ in the proof of Theorem \ref{thm:theorem1}.
Hence, we have that with confidence at least $1-\delta$,
\begin{equation}
    \label{eq:tt002}
    \left\|f_{T+1,D}-f_\rho\right\|_\rho \leq \tilde{C}_{r,3}\max\left\{|D|^{-\frac{r}{2r+s}},\frac{|D|^{\frac{p+1}{2r+s}}}{\sigma^{2p}}\right\}\left(\log \frac{24}{\delta}\right)^5,
\end{equation}
where $ \tilde{C}_{r,3} := \left( C_{0,1} (1+ 2 \eta G'_+(0))(1+\kappa^{2r-1}\|u_\rho\|_\rho)+ 3\cdot 2^{p} C_{r,2}+ 2^{p+\frac{3}{2}} C_{0,1} C_{r,3} \frac{ 4r+2s}{2r-1} \right)(C_{0,3}+1) + \left(\frac{r}{e\eta G'_+(0)}\right)^{r} \|u_\rho\|_\rho$.
It is evident that this result differs from that in \cite{guo2018gradient} only by a constant factor.

Now we turn to estimate $ \left\|\bar{f}^{l}_{T+1,D}-f_{T+1,D}\right\|_\rho $.
By (\ref{eq:communication}), we first need to estimate the bound of $\left\|\bar{f}_{T+1,D}^{0}-f_{T+1,D}\right\|_\rho$, which equals to $ \left\|\bar{f}_{T+1,D}-f_{T+1,D}\right\|_\rho $.
For convenience, the decomposition $ \bar{f}_{T+1,D}-f_{T+1,D}=\bar{f}_{T+1,D}-f_{T+1}+f_{T+1}-f_{T+1,D}$ is adopted, which allows us to separately estimate the two terms $ \left\|\bar{f}_{T+1,D}-f_{T+1}\right\|_\rho $ and $ \left\|f_{T+1,D}-f_{T+1}\right\|_\rho $.

To estimate $\bar{f}_{T+1,D}-f_{T+1}$, by Proposition \ref{prop:prop2} and the proof of Theorem \ref{thm:theorem1}, we have that with confidence at least $1−\delta $, there holds
\begin{equation*}
    \begin{aligned}
        \left\|\bar{f}_{T+1,D}-f_{T+1}\right\|_\rho&\leq (1+ 2\eta G'_+(0) ) (1+\kappa^{2r-1}\|u_\rho\|_\rho)\mathcal{B}_{|D|,\lambda}\log \frac{12}{\delta}  \\
        &+C_{r,2} \frac{T^{p+1}}{\sigma^{2p}}+C_{r,3}\mathcal{B}_{|D|,\lambda}\max\left\{T^{-(r-\frac{1}{2})}\log T,T^{-1},\frac{T^{p+\frac{3}{2}}\log T}{\sigma^{2p}}\right\}\log \frac{6}{\delta} \\
        &+C'_{r,2} \max_{1\leq j\leq m} \left\{ \mathcal{R}_{|D_j|,\lambda}\mathcal{A}_{|D_j|,\lambda} \right\} \frac{T^{p+1}}{\sigma^{2p}} \log T \left(\log \frac{48m}{\delta}\right)^3   \\
        &+C'_{r,3} \max_{1\leq j\leq m} \left\{ \mathcal{R}_{|D_j|,\lambda}\mathcal{B}_{|D_j|,\lambda}\mathcal{A}_{|D_j|,\lambda}^2 \right\}  \max\left\{\log T, T^{-(r-\frac{1}{2})}(\log T)^{2},\frac{T^{p+\frac{3}{2}}(\log T)^2}{\sigma^{2p}}\right\}\left(\log \frac{48m}{\delta}\right)^6.  
    \end{aligned}
\end{equation*}
Since the condition \eqref{eq:m3} still implies \eqref{eq:condition_m}, we can apply the bounds on $\mathcal{B}_{|D_j|,\lambda}$, $\mathcal{R}_{|D_j|,\lambda}$ and $\mathcal{A}_{|D_j|,\lambda}$ in the proof of Theorem \ref{thm:theorem1} to further simplify the above bound as
\begin{equation}\label{eq:communication11}
    \begin{aligned}
    &\left\|\bar{f}_{T+1,D}-f_{T+1}\right\|_\rho \\
    \leq& \tilde{C}_{3,1} \max\left\{\mathcal{B}_{|D|,\lambda},\frac{T^{p+1}}{\sigma^{2p}}\right\}\left(\log \frac{48m}{\delta}\right)^6 + \tilde{C}_{3,2} \max_{1\leq j\leq m} \left\{ \mathcal{R}_{|D_j|,\lambda}\mathcal{B}_{|D_j|,\lambda} \right\} \log T\left(\log \frac{48m}{\delta}\right)^6,
    \end{aligned}
\end{equation}
where $ \tilde{C}_{3,1} := (1+ 2 \eta G'_+(0) ) (1+\kappa^{2r-1}\|u_\rho\|_\rho) +C_{r,2} + \sqrt{2} C_{r,3} C_{0,1} \frac{4r+2s}{2r-1} + C'_{r,2}C_{0,2}(C_{0,3}+1)^{\frac{1}{2}} + \sqrt{2} C'_{r,3}C_{0,1}C_{0,2}(C_{0,3}+1)  $ and $\tilde{C}_{3,2} := \frac{2}{2r-1} C'_{r,3}(C_{0,3}+1) $.

Since for each $ f \in \mathcal{H}_K$, there holds $\|f\|_\rho=\left\|L_K^{\frac{1}{2}}f\right\|_K$.
Considering that $f_{T+1,D}, f_{T+1}\in \mathcal{H}_K$, we have
$$
\|f_{T+1,D}-f_{T+1}\|_\rho=\left\|L_K^{\frac{1}{2}}(f_{T+1,D}-f_{T+1})\right\|_K\leq \left\|(\lambda I+L_{K})^{\frac{1}{2}}(f_{T+1,D}-f_{T+1})\right\|_K.
$$
By Proposition \ref{prop:prop1}, the proof of Theorem \ref{thm:theorem1} and the partition size $m$ specified in (\ref{eq:m3}), we can obtain that with confidence at least $1−\delta $, there holds
\begin{equation}\label{eq:communication12}
    \begin{aligned}
    \left\|f_{T+1,D}-f_{T+1}\right\|_\rho &\leq (1 + 2 \eta G'_+(0) )\mathcal{A}_{D,\lambda}^2(\mathcal{P}_{D,\lambda}+\mathcal{Q}_{D,\lambda}\|f_\rho\|_K)\\
    &+C_{r,2} \mathcal{A}_{D,\lambda} \frac{T^{p+1}}{\sigma^{2p}} + C_{r,3} \mathcal{A}_{D,\lambda}^2\mathcal{Q}_{D,\lambda} \max\left\{ T^{-(r-\frac{1}{2})}\log T,T^{-1},\frac{T^{p+\frac{3}{2}}\log T}{\sigma^{2p}}\right\}  \\
    &\leq \tilde{C}_{3,3} \max\left\{\mathcal{B}_{|D|,\lambda},\frac{T^{p+1}}{\sigma^{2p}}\right\} \left(\log \frac{24}{\delta}\right)^5,
    \end{aligned}
\end{equation}
where $ \tilde{C}_{3,3} := (1+ 2 \eta G'_+(0) ) (1+\kappa^{2r-1}\|u_\rho\|_\rho)(C_{0,3}+1) + C_{r,2}(C_{0,3}+1)^{\frac{1}{2}} + \sqrt{2} C_{r,3} C_{0,1} \frac{4r+2s}{2r-1} (C_{0,3}+1) $.

With the triangle inequality, we combine (\ref{eq:communication11}) and (\ref{eq:communication12})to derive that with confidence at least $1−2\delta $, there holds
\begin{equation*}
    \begin{aligned}
        \left\|\bar{f}_{T+1,D}^{0}-f_{T+1,D}\right\|_\rho&=\left\|\bar{f}_{T+1,D}-f_{T+1,D}\right\|_\rho \leq \left\|\bar{f}_{T+1,D}-f_{T+1}\right\|_\rho+\left\|f_{T+1,D}-f_{T+1}\right\|_\rho\\
        &\leq ( \tilde{C}_{3,1} + \tilde{C}_{3,3}) \max\left\{\mathcal{B}_{|D|,\lambda},\frac{T^{p+1}}{\sigma^{2p}}\right\}\left(\log \frac{48m}{\delta}\right)^6\\
        & \qquad + \tilde{C}_{3,2} \max_{1\leq j\leq m} \left\{ \mathcal{R}_{|D_j|,\lambda}\mathcal{B}_{|D_j|,\lambda} \right\} \log T \left(\log \frac{48m}{\delta}\right)^6.
    \end{aligned}
\end{equation*}
Then we can adopt (\ref{eq:communication}) to bound $ \left\|\bar{f}^{l}_{T+1,D}-f_{T+1,D}\right\|_\rho $.
Note that bounding the sum term in (\ref{eq:communication}) relies on the high-probability bounds of $\mathcal{Q}_{|D|,\lambda}$, $\mathcal{Q}_{|D_j|,\lambda}$ and $\mathcal{A}_{|D_j|,\lambda}$, which are also analyzed in deriving the bound of $\left\|\bar{f}_{T+1,D}^{0}-f_{T+1,D}\right\|_\rho$.
Hence, by Lemma \ref{lemma:lemma3} and Lemma \ref{lemma:lemmaa}, we have that with confidence at least $1−2\delta $, there holds
\begin{equation}\label{eq:communication13}
    \begin{aligned}
        &\left\|f_{T+1,D}-\bar{f}_{T+1,D}^l\right\|_\rho \leq  \max_{1\leq j\leq m} \left( (\mathcal{Q}_{D,\lambda} + \mathcal{Q}_{D_j,\lambda}) \mathcal{A}_{D_j,\lambda} C'_\eta \left(\lambda^{\frac{1}{2}} T+ T^{\frac{1}{2}}\right)\right)^l \left\|\bar{f}_{T+1,D}^{0}-f_{T+1,D}\right\|_\rho   \\
        \leq & (4+2\sqrt{2})^{l} (C'_\eta)^l  (C_{0,3}+1)^{\frac{l}{2}} \lambda^{-\frac{l}{2}}  \max_{1\leq j\leq m}  \mathcal{B}_{|D_j|,\lambda}^{l} \cdot  ( \tilde{C}_{3,1} + \tilde{C}_{3,3}) \max\left\{\mathcal{B}_{|D|,\lambda},\frac{T^{p+1}}{\sigma^{2p}}\right\}\left(\log \frac{48m}{\delta}\right)^{6+3l}  \\
        & \quad + (4+2\sqrt{2})^{l} (C'_\eta)^l  (C_{0,3}+1)^{\frac{l}{2}} \lambda^{-\frac{l}{2}}  \max_{1\leq j\leq m}  \mathcal{B}_{|D_j|,\lambda}^{l} \cdot \tilde{C}_{3,2} \max_{1\leq j\leq m} \left\{ \mathcal{R}_{|D_j|,\lambda}\mathcal{B}_{|D_j|,\lambda} \right\} \log T \left(\log \frac{48m}{\delta}\right)^{6+3l}. \\
        & \leq \tilde{C}_{3,4}\max\left\{|D|^{-\frac{r}{2r+s}},\frac{|D|^{\frac{p+1}{2r+s}}}{\sigma^{2p}}\right\} \left(\log \frac{48}{\delta}\right)^{6+3l}   .
    \end{aligned}
\end{equation}
where $ \tilde{C}_{3,4}:= (4+2\sqrt{2})^{l} 2^{6+3l} (C'_\eta)^l C_{0,1}^{l} (C_{0,3}+1)^{\frac{l}{2}} \left[ (C_{0,1} + 2^{p+1}) ( \tilde{C}_{3,1} +  \tilde{C}_{3,3}) +  C_{0,1} C_{0,2}  \tilde{C}_{3,2} \right] $.

It is observed that the derivation of both high-probability upper bounds for $ \left\|f_{T+1,D}-f_{T+1}\right\|_\rho $ and $ \left\|f_{T+1,D}-f_{\rho}\right\|_\rho $ relies on the estimation of the bounds of $\mathcal{P}_{|D|,\lambda}$, $\mathcal{Q}_{|D|,\lambda}$ and $\mathcal{A}_{|D|,\lambda}$.
Consequently, these two bounds hold simultaneously.
Applying the triangle inequality and scaling $ 2 \delta $ to $ \delta $, we can combine \eqref{eq:tt002} and \eqref{eq:communication13} to derive that with confidence at least $1-\delta$, there holds
\begin{equation*}
    \left\|\bar{f}_{T+1,D}^l-f_\rho\right\|_\rho\leq \tilde{C}_3\max\left\{|D|^{-\frac{r}{2r+s}},\frac{|D|^{\frac{p+1}{2r+s}}}{\sigma^{2p}}\right\} \left(\log \frac{96}{\delta}\right)^{6+3l}.
\end{equation*}
where $\tilde{C}_3= \tilde{C}_{3,4} + \tilde{C}_{r,3} $.
This completes the proof of Theorem \ref{thm:theorem3}.

\section*{Appendix}
In this appendix, we give the proof of  Lemma \ref{lemma:lemma2} and Lemma \ref{lemma:lemma1}.

\noindent\textbf{Proof of Lemma \ref{lemma:lemma2}.}
We prove \eqref{eq:lemma2} by induction, $f_1=0$ is obviously satisfies the inequality (\ref{eq:lemma2}). 
Then for $t=2$,  (\ref{eq:ft1}) indicates that
\begin{equation*}
    \|f_2\|_K=\left\|-\eta \int_{\mathcal{Z}} G'(\xi_{t,\sigma}(z)) (-y)K_x\mathrm{d}\rho\right\|_K\leq \eta \kappa M C_G \leq M \sqrt{\eta C_G} \leq \frac{M}{\kappa},
\end{equation*}
which means that the inequality (\ref{eq:lemma2}) also holds for $t=2$.
When $t>2$, we denote $H_t=\int_{\mathcal{Z}} G'(\xi_{t,\sigma}(z)) (f_t(x)-y)K_x\mathrm{d}\rho$ and have that
\begin{align}\label{eq:lemma12}
    &\|f_{t+1}\|_K^2= \|f_{t}-\eta H_t\|_K^2 =\|f_{t}\|_K^2-2\eta \langle f_t,H_t \rangle_K + \eta^2\|H_t\|_K^2 \notag\\
    \leq &\|f_{t}\|_K^2+\int_\mathcal{Z}\left(-2\eta  G'(\xi_{t,\sigma}(z)) (f_t(x)-y)f_t(x) + \eta^2 \kappa^2  G'(\xi_{t,\sigma}(z))^2 (f_t(x)-y)^2\right)\mathrm{d}\rho \notag \\
    =&\|f_{t}\|_K^2+\int_\mathcal{Z}\eta \left(\eta \kappa^2 G'(\xi_{t,\sigma}(z)) (f_t(x)-y)^2 - 2(f_t(x)-y)f_t(x) \right)G'(\xi_{t,\sigma}(z))\mathrm{d}\rho\\
    =&\|f_{t}\|_K^2+\int_\mathcal{Z}\eta \left\{\left(\eta \kappa^2 G'(\xi_{t,\sigma}(z))-2\right) \left((f_t(x)-y)- \frac{y}{\eta \kappa^2 G'(\xi_{t,\sigma}(z))-2}\right)^2 +\frac{y^2}{2-\eta \kappa^2 G'(\xi_{t,\sigma}(z))} \right\}G'(\xi_{t,\sigma}(z))\mathrm{d}\rho, \notag
\end{align}
where the first inequality is due to the Jensen's inequality and the fact that $K(x,x')\leq \kappa^{2}$ for any $x,x'\in \mathcal{X}$, and the last equality is due to the completion of squares.

Since $G'(\xi_{t,\sigma}(z))\leq C_G$, we have that $\eta \kappa^2 G'(\xi_{t,\sigma}(z))\leq 1$, and further $\eta \kappa^2 G'(\xi_{t,\sigma}(z))-2\leq 0$ and $2-\eta \kappa^2 G'(\xi_{t,\sigma}(z))\geq 1$ for any $ z \in \mathcal{Z} $. 
Putting the above bound with $|y|\leq M$  and the induction assumption $‖f_t‖_K^2 \leq \eta C_G M^2(t-1)$ into (\ref{eq:lemma12}) yields that 
$$
\|f_{t+1}\|_K^2 \leq \|f_{t}\|_K^2+\eta C_G M^2 \leq \eta C_G M^2 t,
$$
and thus
$$
\|f_{t+1}\|_K\leq M\sqrt{\eta C_G t}\leq \frac{M}{\kappa}t^\frac{1}{2}.
$$
This completes the proof of Lemma \ref{lemma:lemma2}. \qed

\noindent\textbf{Proof of Lemma \ref{lemma:lemma1}.}
With the choice of $\eta$, we have 
\begin{equation*}
    \left\|\left(I-\eta G'_+(0) L_{K}\right)^{T-i}\right\|\leq 1, \quad \forall i=0,1,\cdots,T.
\end{equation*}
Based on the proof of Theorem 1 (i) in \cite[Page 15]{guo2018gradient}, it can be derived with $\theta=0$ and $\eta_1$ replaced by $\eta$ that 
\[
    \left\| \sum_{i=1}^T \eta G'_+(0) (\lambda I+L_{K,D}) \left(I-\eta G'_+(0) L_{K,D}\right)^{T-i}\right\| \leq \eta G'_+(0)\lambda T+1.
\]
Analogously, we can replace $ L_{K,D} $ with $ L_{K} $ to obtain the same bound of the operator norm concerning $L_{K}$.
Therefore, the proof is completed.

\bibliographystyle{gmcm}
\bibliography{reference}

\appendix % Start the appendix section

\end{document}